\documentclass[journal]{IEEEtran}

\usepackage{comment}
\usepackage{cite}
\usepackage{amsmath,amssymb,amsfonts,mathrsfs}
\usepackage{algorithm,algorithmic,setspace}
\usepackage{graphicx}
\usepackage{textcomp}
\usepackage{xcolor}
\usepackage{mathtools}
\usepackage{dsfont}
\usepackage{subcaption}
\usepackage{makecell}
\usepackage{epstopdf} 
\usepackage{graphics} 
\usepackage{url}
\usepackage{color,soul}
\usepackage{amsthm}
\usepackage[colorlinks=true,linkcolor=blue, citecolor=green, urlcolor=magenta]{hyperref}

\theoremstyle{definition}

\newtheorem{theorem}{Theorem}
\newtheorem{assumption}{Assumption}
\newtheorem{definition}{Definition}

\newtheorem{lemma}{Lemma}
\newtheorem{remark}{Remark}
\newtheorem{problem}{Problem}
\newtheorem{fact}{Fact}

\newcommand{\Lim}{\displaystyle\lim}
\allowdisplaybreaks[4]

\AtEndEnvironment{remark}{\hfill\ensuremath{\square}}
\AtEndEnvironment{definition}{\hfill\ensuremath{\square}}

\makeatletter
\@ifundefined{proof}{
}{%
  \let\oldproof\proof
  \renewcommand{\proof}[1]{\oldproof{#1\hfill\ensuremath{\blacksquare}}}
}
\makeatother

\title{
    From Kinematic Motion Planners to Dynamic Autonomous Navigation with Obstacle Avoidance (Extended version)
} 
\author{Mayur Sawant and Abdelhamid Tayebi  
	\thanks{This work was supported by the Natural Sciences and Engineering Research Council of Canada (NSERC), under the grant RGPIN-2026-06037}
	\thanks{M. Sawant and A. Tayebi are with the Department of Electrical and Computer Engineering, Lakehead University, Thunder Bay, ON P7B 5E1, Canada. (e-mail: {\tt\small msawant, atayebi@lakeheadu.ca}).}%
}%

\begin{document}
\maketitle


\begin{abstract}
Kinematic motion planners are among the most widely used control approaches in robotic applications.
By modeling the robot as a first-order system, they generate a feedback-based desired velocity field that guides the robot toward a target while avoiding obstacles. 
However, extending such feedback planners to systems with higher-order dynamics, while preserving safety and stability properties, is not a straightforward task.
In the present work, we propose an approach that adapts existing feedback-based kinematic motion planners to second-order autonomous systems while retaining their safety and almost global asymptotic stability guarantees.  
We consider two general classes of kinematic motion planners: those derived from navigation functions and those defined directly through desired velocity fields without relying on an underlying navigation function.
To validate the proposed methodology, two feedback-based kinematic motion planners are adapted to second-order systems and evaluated both in simulations and experimentally.
\end{abstract}


\section{Introduction}
\IEEEPARstart{A}{utonomous} robot navigation involves steering a robot to a target location while avoiding obstacles. 
A widely explored class of methods for this task is based on artificial potential fields (APFs) \cite{khatib1986real}, in which an attractive vector field guides the robot toward the target, while a repulsive vector field facilitates obstacle avoidance. 
However, in certain obstacle configurations, these methods can generate undesired local minima that hinder convergence to the target.
The navigation function-based approach \cite{koditschek1990robot} addresses this problem in a sphere world\footnote{A Euclidean sphere world of dimension $n$ is formed by removing from the interior of a large $n$-dimensional ball a finite number of non-overlapping smaller balls.} by restricting the influence of the repulsive vector field within a neighborhood of the obstacle by means of a properly tuned parameter, thereby guaranteeing almost\footnote{Almost global convergence here refers to the convergence from all initial conditions except a set of zero Lebesgue measure.} global convergence to the target.
The navigation function-based approach has subsequently been extended to more general settings, including convex and star-shaped obstacles using diffeomorphic transformations \cite{koditschek1992exact, li2018navigation} and sufficiently curved obstacles \cite{filippidis2012navigation, paternain2017navigation}.

The authors of \cite{arslan2019sensor} proposed a separating-hyperplane-based reactive feedback controller that guarantees almost global asymptotic convergence to the target in unknown environments with convex obstacles, each satisfying a sufficient curvature condition. 
In \cite{kumar2022navigation}, the authors relaxed the curvature restrictions imposed in \cite{paternain2017navigation} and \cite{arslan2019sensor} for environments containing ellipsoidal obstacles by constructing a repulsive vector field that pushes the robot away from the center of the obstacle rather than from the closest point on its boundary.

Despite their effectiveness in navigating environments with complex obstacle geometries, most of these navigation strategies assume that the robot is governed by a velocity-controlled first-order model.  However, many practical robotic systems are more accurately described by second-order dynamics.
Several works have proposed navigation strategies directly applicable to second-order systems \cite{wang2017safety, verginis2021adaptive, stavridis2017dynamical, smaili2024dissipative}, but these approaches are often limited in the range of obstacle geometries they can accommodate.

\begin{figure}
    \centering
    \includegraphics[width=1\linewidth]{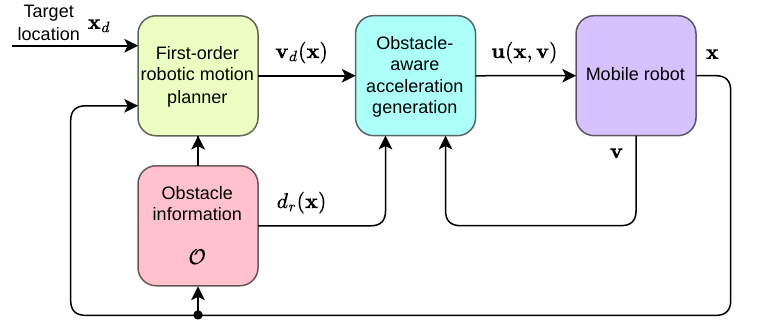}
    \caption{Block diagram illustrating the adaptation of the feedback-based kinematic motion planners to second-order dynamics via obstacle-aware damping.}
    \label{block_diagram:general}
\end{figure}

In \cite{wang2017safety}, a control-barrier-function-based method is introduced for multi-robot navigation in two-dimensional environments with circular obstacles, focusing on robots governed by second-order dynamics. 
In \cite{verginis2021adaptive}, the navigation function framework is extended to robots with uncertain second-order dynamics operating in known sphere worlds.
Similar to the APF in \cite{khatib1986real}, the proposed navigation function grows unbounded as the robot approaches the obstacle boundary.
Adaptive control techniques are used to design a gradient-based controller that compensates for dynamic model uncertainties, while guaranteeing obstacle avoidance and almost global asymptotic convergence to the desired location. 

In \cite{stavridis2017dynamical}, a nominal second-order system is considered whose target is a globally asymptotically stable equilibrium in a three-dimensional obstacle-free environment.
In the presence of known obstacles, each enclosed by a generalized ellipsoid, the nominal dynamics are augmented by an additive obstacle avoidance control input to ensure safety.
However, similar to \cite{khatib1986real}, this modification can lead to undesired local minima in certain obstacle configurations.
Recently, in \cite{smaili2024dissipative}, the authors proposed a feedback control design for autonomous navigation of a robot governed by second-order dynamics in environments with sufficiently separated obstacles that satisfy some specific curvature conditions. 

Although these approaches directly apply to second-order systems, their applicability is often limited by the types of obstacle geometries they can handle. In contrast, first-order motion planners, such as those in \cite{filippidis2012navigation, kumar2022navigation}, can ensure safe navigation in environments with obstacles of complex geometries. 
This highlights the need for a control strategy that adapts feedback-based kinematic motion planners to robots governed by second-order dynamics.

In \cite{arslan2017smooth}, a reference governor architecture was proposed to extend feedback-based kinematic motion planners to robots modeled by second-order dynamics. The resulting closed-loop system comprises three states: the robot's position $\mathbf{x}$, its velocity $\mathbf{v}$, and a virtual governor state $\mathbf{x}_g$.
Similar to $\mathbf{x}$, the virtual governor state $\mathbf{x}_g$ belongs to the free workspace.
The state $\mathbf{x}_g$ evolves according to a reference governor dynamics constructed using the desired velocity prescribed by the kinematic motion planner at $\mathbf{x}_g$, the distance from $\mathbf{x}_g$ to nearby obstacles, and the separation between $\mathbf{x}_g$ and $\mathbf{x}$.
A tracking controller is then employed to drive the second-order robot toward $\mathbf{x}_g$, which serves as a collision-free reference that guides the robot to the target.

Recently, \cite{icsleyen2022low} generalized this reference governor framework to extend the feedback-based kinematic motion planners to robotic systems governed by $n$-th order dynamics.
In both \cite{arslan2017smooth} and \cite{icsleyen2022low}, however, the navigation scheme requires evaluating the desired velocity prescribed by the kinematic motion planner at the virtual governor state $\mathbf{x}_g$ rather than at the robot's position $\mathbf{x}$.
Since $\mathbf{x}_g$ is spatially separated from $\mathbf{x}$, this requirement generally necessitates either access to global workspace information
or additional computation to transform the robot-centered sensor measurements into geometric information relative to the virtual governor state.

We consider two widely used classes of feedback-based kinematic motion planners: planners relying on a navigation function and planners specified directly by a desired velocity field without an underlying navigation function.
For the former class, we design a control strategy that relies on a damping term with a gain depending on the robot's proximity to the obstacles. We show that the resulting second-order closed-loop system inherits the safety and convergence properties of the kinematic motion planner.
For the latter class, we design a control strategy that ensures a monotonic decrease in the norm of the error between the robot's velocity and the desired velocity prescribed by the kinematic motion planner, and prove that the safety and convergence properties of the kinematic motion planner are preserved for the resulting second-order closed-loop system.
A general block diagram of the proposed control schemes is shown in Fig. \ref{block_diagram:general}.

The main contributions of this paper are as follows:
\begin{enumerate}
\item The proposed control schemes extend the safety and convergence guarantees of feedback-based kinematic motion planners, such as \cite{koditschek1990robot}, \cite{paternain2017navigation}, \cite{kumar2022navigation}, to second-order systems.
This enables provably correct navigation in environments with complex obstacle geometries (\textit{e.g.,} ellipsoidal obstacles that are either very close to the target location or have an extremely flattened shape), for which existing second-order motion planners \cite{khatib1986real, smaili2024dissipative} do not guarantee almost global asymptotic stability of the target state.

\item Unlike reference governor-based approaches \cite{arslan2017smooth}, \cite{icsleyen2022low}, which require evaluating the desired velocity prescribed by the kinematic motion planner at the virtual governor state that is spatially separated from the actual robot state, the proposed control schemes evaluate the desired velocity directly at the robot's center.
Consequently, any kinematic motion planner that relies solely on information available at the robot's current position, including a sensor-based planner designed for an \textit{a priori} unknown environment, can be adapted to second-order dynamics without requiring workspace information beyond that already used by the planner or transforming robot-centered sensor measurements into geometric information relative to a virtual governor state.


\item For the existing gradient-based kinematic motion planners that rely on a navigation function that becomes unbounded near obstacle boundaries, a fixed-gain damping term can be added to handle second-order dynamics. This adaptation is not possible for kinematic motion planners that rely on bounded navigation functions \cite{koditschek1990robot}, \cite{paternain2017navigation}. The proposed control scheme removes this restriction by introducing a damping term whose gain varies with the robot's proximity to the obstacles. Consequently, the proposed control strategy can accommodate gradient-based kinematic motion planners derived from bounded navigation functions.

\end{enumerate}

The rest of the paper is organized as follows. 
Section \ref{section:notations} introduces the notations and mathematical preliminaries used throughout the paper, and Section \ref{section:problem_statements} formulates the problem.
In Section \ref{section:feedback_control_design}, we present two feedback controller designs that adapt the first-order motion planners to the second-order dynamical systems.
Section \ref{section:examples-of-motion-planners} presents two examples of first-order motion planners, which are subsequently adapted to the second-order system in
Section \ref{section:simulations} to demonstrate the effectiveness of the proposed control schemes through non-trivial simulation studies. 
Section \ref{section:experimental_validation} presents the experimental validation of the proposed control strategies using the QDrone2 quadrotor.
Finally, concluding remarks are provided in Section \ref{section:conclusion}.

\section{Notations and preliminaries}\label{section:notations}
The sets of real numbers, non-negative real numbers, positive real numbers, and natural numbers are denoted by $\mathbb{R}$, $\mathbb{R}_{\geq 0}$, $\mathbb{R}_{>0}$, and $\mathbb{N}$, respectively. 
Bold lowercase letters denote vector quantities. 
Given a vector $\mathbf{a}\in\mathbb{R}^n$, the closed Euclidean ball of radius $r\geq0$ with its center at $\mathbf{a}$ is given by $\mathcal{B}_{r}(\mathbf{a}) = \{\mathbf{b}\in\mathbb{R}^n|\|\mathbf{b} - \mathbf{a}\|\leq r\}$, where $\|\cdot\|$ represents the Euclidean norm.
The set of $n$-dimensional unit vectors is defined as $\mathbb{S}^{n-1} = \{\mathbf{a}\in\mathbb{R}^n|\|\mathbf{a}\| = 1\}$.
The identity matrix and the zero matrix of dimension $n\in\mathbb{N}$ are denoted by $\mathbf{I}_n$ and $\mathbf{0}_{n}$, respectively. 
The Frobenius norm of a matrix $\mathbf{A}\in\mathbb{R}^{n\times n}$ is defined as $\|\mathbf{A}\|_{\mathbf{F}} = \sqrt{\underset{i}{\sum}\underset{k}{\sum}a_{ik}^2}$, where $a_{ik}$ denotes the element in the $i$-th row and $k$-th column of $\mathbf{A}$.
For a twice continuously differentiable scalar mapping $f:\mathbb{R}\to\mathbb{R}$, we denote its first and second derivatives by $f'(x) = \frac{df(x)}{dx}$ and $f''(x) = \frac{d^2f(x)}{dx^2}$, respectively.
Given two functions $f$ and $g$ for which the composition is defined, their composition is denoted by $f\circ g$, with $(f\circ g)(x)=f(g(x))$.

For sets $\mathcal{A}, \mathcal{B}\subset\mathbb{R}^n$, the relative complement of $\mathcal{B}$ with respect to $\mathcal{A}$ is given by $\mathcal{A}\setminus\mathcal{B} = \{\mathbf{a}\in\mathcal{A}|\mathbf{a}\notin\mathcal{B}\}$.
Given a set $\mathcal{A}\subset\mathbb{R}^n$, the symbols $\bar{\mathcal{A}}, \mathcal{A}^{\circ}, \mathcal{A}^c, \text{ and } \partial\mathcal{A}$ represent the closure, interior, complement and the boundary of $\mathcal{A}$, where $\partial\mathcal{A} = \bar{\mathcal{A}}\setminus\mathcal{A}^{\circ}$. 
Given $\mathcal{A}\subset\mathbb{R}^n$, the cardinality of $\mathcal{A}$ is denoted by $\mathrm{card}(\mathcal{A})$.
The Minkowski sum of two sets $\mathcal{A}, \mathcal{B}\subset\mathbb{R}^n$ is defined as $\mathcal{A}\oplus\mathcal{B} = \{\mathbf{a} + \mathbf{b}|\mathbf{a}\in\mathcal{A}, \mathbf{b}\in\mathcal{B}\}$.
The dilation of a set $\mathcal{A}\subset\mathbb{R}^n$ by $r>0$ is given as $\mathcal{D}_r(\mathcal{A}) := \mathcal{A}\oplus\mathcal{B}_{r}(\mathbf{0})$. 
Given $r>0$ and $\mathcal{A}\subset\mathbb{R}^n$, the $r$-neighborhood of $\mathcal{A}$ is denoted by $\mathcal{N}_{r}(\mathcal{A})$, and is given by $\mathcal{N}_{r}(\mathcal{A}) = \mathcal{D}_{r}(\mathcal{A})\setminus\bar{\mathcal{A}}$.

\subsection{Distance to a set}\label{definition_distance_to_a_set}
The distance between a point $\mathbf{a}\in\mathbb{R}^n$ and a closed set $\mathcal{A}\subset\mathbb{R}^n$ is denoted by $d(\mathbf{a}, \mathcal{A})$ and is evaluated as
\begin{equation*}
    d(\mathbf{a}, \mathcal{A}) := \underset{\mathbf{b}\in\mathcal{A}}{\min}\|\mathbf{a} - \mathbf{b}\|.
\end{equation*}
The set $\mathcal{P}(\mathbf{a}, \mathcal{A})$, which is defined as
\begin{equation*}
    \mathcal{P}(\mathbf{a}, \mathcal{A}) := \{\mathbf{b}\in\mathcal{A}\mid  \|\mathbf{a} - \mathbf{b}\|=d(\mathbf{a}, \mathcal{A})\},
\end{equation*}
contains all points in the set $\mathcal{A}$ that are at a distance $d(\mathbf{a}, \mathcal{A})$ from $\mathbf{a}$. 
When $\mathrm{card}(\mathcal{P}(\mathbf{a}, \mathcal{A})) = 1$, the element in the set $\mathcal{P}(\mathbf{a} ,\mathcal{A})$, which represents the unique closest point in $\mathcal{A}$ to $\mathbf{a}$, is denoted by $\Pi(\mathbf{a}, \mathcal{A})$.

\subsection{Geometric sets}

\subsubsection{Hyperplane}\label{definition_hyperplane} The hyperplane passing through $\mathbf{x}\in\mathbb{R}^n$ and perpendicular to $\mathbf{p}\in\mathbb{R}^n\setminus\{\mathbf{0}\}$ is defined as
\begin{equation}
    \mathcal{H}(\mathbf{x}, \mathbf{p}) := \{\mathbf{q}\in\mathbb{R}^n\mid(\mathbf{q} - \mathbf{x})^\top\mathbf{p} = 0\}.\label{equation_hyperplane}
\end{equation}
The closed positive half-space and the closed negative half-space, denoted by $\mathcal{H}_{\geq}(\mathbf{x}, \mathbf{p})$ and $\mathcal{H}_{\leq}(\mathbf{x}, \mathbf{p})$, respectively, are obtained by replacing `$=$' on the right in \eqref{equation_hyperplane} with `$\geq$' and `$\leq$', respectively. 
We also use the notations $\mathcal{H}_{>}(\mathbf{p}, \mathbf{q})$ and $\mathcal{H}_{<}(\mathbf{p} ,\mathbf{q})$ to denote the open positive and the open negative half-spaces such that $\mathcal{H}_{>}(\mathbf{p}, \mathbf{q}) = \mathcal{H}_{\geq}(\mathbf{p}, \mathbf{q})\backslash\mathcal{H}(\mathbf{p}, \mathbf{q})$ and $\mathcal{H}_{<}(\mathbf{p} ,\mathbf{q})= \mathcal{H}_{\leq}(\mathbf{p}, \mathbf{q})\backslash\mathcal{H}(\mathbf{p}, \mathbf{q})$.

\section{Problem statement}\label{section:problem_statements}
We are interested in adapting feedback-based kinematic motion planners to second-order systems while retaining their safety and stability properties. In other words, we aim to design a feedback control law $\mathbf{u}$ for the second-order system
\begin{equation}\label{equation:second_order_system}
    \begin{aligned}
        \dot{\mathbf{x}} &= \mathbf{v},\\
        \dot{\mathbf{v}} &= \mathbf{u},
    \end{aligned}
\end{equation}
guaranteeing safety and asymptotic stability of the equilibrium $(\mathbf{x}=\mathbf{x}_d, \mathbf{v}=\mathbf{0})$, knowing that the first-order system $\dot{\mathbf{x}}=\mathbf{v}_d$ guarantees safety and asymptotic stability of the target location $\mathbf{x}=\mathbf{x}_d$, where $\mathbf{v}_d$ is referred to as the first-order motion planner. The vectors $\mathbf{x} \in \mathbb{R}^n$, $\mathbf{v}\in\mathbb{R}^n$ and $\mathbf{u}\in\mathbb{R}^n$ denote the vehicle's position, velocity and control input, respectively.

We assume that the workspace $\mathcal{W}$ is a pathwise connected subset of $\mathbb{R}^n$ with compact boundary\footnote{Compactness of $\partial\mathcal{W}$ does not require $\mathcal{W}$ to be bounded. In particular, $\mathcal{W}=\mathbb{R}^n$ is allowed since its boundary is the empty set, which is compact.} and contains $n$-dimensional compact obstacles $\mathcal{O}_i$, where $i\in\mathbb{I}:=\{1, \ldots, m\}$ and $m\in\mathbb{N}$. 
The obstacles are pairwise disjoint, \textit{i.e.,} $\mathcal{O}_i\cap\mathcal{O}_j = \emptyset$ for all $i, j\in\mathbb{I}, i\ne j$.
Collectively, $\mathcal{O} := \bigcup_{i\in\mathbb{I}\cup\{0\}}\mathcal{O}_i$ represents the obstacle set, where $\mathcal{O}_0 = (\mathcal{W}^{\circ})^c$ is the region outside of $\mathcal{W}$ with its boundary $\partial\mathcal{W}$.

The robot's body is contained within an $n$-dimensional sphere of radius $r \geq 0$, where $r$ is the sum of the robot's radius and a safety distance.\footnote{The case $r=0$ is allowed to accommodate first-order motion planners $\mathbf{v}_d$ developed for point robot navigation settings, such as \cite{koditschek1990robot, kumar2022navigation}.} For collision-free navigation, the robot's center $\mathbf{x}$ must always belong to the interior of the free space $\mathcal{X}_r$, where given $p \geq 0$, the set $\mathcal{X}_p$ is defined as
\begin{equation}
    \mathcal{X}_p = \{\mathbf{x}\in\mathcal{W}\mid\mathcal{B}_{p}(\mathbf{x})\cap\mathcal{O}^{\circ}=\emptyset\}. \label{equation_free_space}
\end{equation}

To ensure safe navigation to any desired target location $ \mathbf{x}_d$, the free space $\mathcal{X}_r$ must be pathwise connected. Additionally, as mentioned next in Assumption \ref{assumption:twice_differentiable_boundary}, we impose certain conditions on the unsafe region $\mathcal{X}_r^c$ that are necessary to ensure that the gradient vector and the Hessian matrix of the distance function $d_r(\mathbf{x})$, defined as
\begin{equation}\label{distance_function_d_x}
d_r(\mathbf{x})= d_s(\mathbf{x}, \mathcal{O}) - r,
\end{equation}
are well-defined when the robot is close to the obstacles, \textit{i.e.,} when $d_r$ is small, where $d_s(\mathbf{x}, \mathcal{O}) = d(\mathbf{x}, \mathcal{O}) - d(\mathbf{x}, \mathcal{O}^c)$ is the oriented distance function\footnote{For $\mathbf{x}\in\mathbb{R}^n\setminus\mathcal{O}$, one has $d_s(\mathbf{x}, \mathcal{O}) = d(\mathbf{x}, \mathcal{O})$. 
The scalar function $d_s$ is useful because, under additional smoothness assumptions on $\partial\mathcal{O}$, it is differentiable on $\partial\mathcal{O}$, unlike the Euclidean distance function $d$.}\cite[Chapter 7, Definition 2.1]{delfour2011shapes}.

\begin{assumption}\label{assumption:twice_differentiable_boundary}
    The obstacle set $\mathcal{O}$ satisfies the following requirements:
    \begin{enumerate}
        \item \label{condition:C2}The boundary $\partial\mathcal{O}$ is an embedded submanifold of $\mathbb{R}^n$ of class $\mathcal{C}^2$, where a differentiable manifold of class $\mathcal{C}^2$ is defined using \cite[Definition 1.33]{lee2009manifolds}.
        \item \label{condition:uniqueness}There exists a scalar $\delta_u >0$ such that, for every $\mathbf{x}\in\mathcal{N}_{r + \delta_u}(\mathcal{O})$, there exists a unique closest point on $\mathcal{O}$, \textit{i.e.}, $\mathrm{card}(\mathcal{P}(\mathbf{x}, \mathcal{O})) = 1$.
    \end{enumerate}
\end{assumption}

\begin{remark}\label{remark:twice_differentiability}
Loosely speaking, Condition \ref{condition:C2} of Assumption \ref{assumption:twice_differentiable_boundary} requires obstacle boundaries to be sufficiently smooth \textit{i.e.}, they have no corners, yielding a well-defined and continuous curvature.
Furthermore, Condition \ref{condition:uniqueness} of Assumption \ref{assumption:twice_differentiable_boundary} implies that the pairwise Euclidean distance between obstacles $\mathcal{O}_i$ and $\mathcal{O}_j$ is at least $2(r + \delta_u)$, where $i, j\in\mathbb{I}\cup\{0\}, i\ne j$.
In other words, 
\[
d(\mathcal{O}_i, \mathcal{O}_j) = \underset{\mathbf{a}\in\mathcal{O}_i, \mathbf{b}\in\mathcal{O}_j}{\min}\|\mathbf{a} - \mathbf{b}\| \geq 2(r + \delta_u),
\]
for all $i,j\in\mathbb{I}\cup\{0\}, i\ne j$.

Assumption \ref{assumption:twice_differentiable_boundary} further ensures that $d_r$ defined in \eqref{distance_function_d_x} is twice continuously differentiable over $\overline{\mathcal{N}_{r+\delta_u}(\mathcal{O})}$.
Let $\boldsymbol{\eta}(\mathbf{x})$ be the gradient of $d_r$ with respect to $\mathbf{x}$ over $\overline{\mathcal{N}_{r+\delta_u}(\mathcal{O})}$, \textit{i.e.},
\begin{equation}\label{Eq:eta-is-gradient-of-dr}\boldsymbol{\eta}(\mathbf{x})=\nabla_{\mathbf{x}}d_r(\mathbf{x}), \forall \mathbf{x}\in\overline{\mathcal{N}_{r+\delta_u}(\mathcal{O})}.\end{equation}
At points $\mathbf{x}$ on $\partial\mathcal{O}$, $\boldsymbol{\eta}(\mathbf{x})$ is equal to the unit exterior normal vector to $\mathcal{O}$ at $\mathbf{x}$ \cite[Chapter 7, Theorem 8.2]{delfour2011shapes}.
When $\mathbf{x}\in\mathcal{N}_{r+\delta_u}(\mathcal{O})$, one has $\mathbf{x}\notin\mathcal{O}$ and it follows that $d_s(\mathbf{x}, \mathcal{O}) = d(\mathbf{x}, \mathcal{O})$.
Therefore, by \cite[Lemma 4.2]{rataj2019curvature}, the gradient of $d_r$ at $\mathbf{x}\in\mathcal{N}_{r+\delta_u}(\mathcal{O})$ is given as
    \begin{equation}
        \label{definition:eta}\nabla_{\mathbf{x}}d_r(\mathbf{x}) = \frac{\mathbf{x} - \Pi(\mathbf{x}, \mathcal{O})}{\|\mathbf{x} - \Pi(\mathbf{x}, \mathcal{O})\|},
\end{equation}
where $\Pi(\mathbf{x}, \mathcal{O})$ is the unique closest point to $\mathbf{x}$ in $\mathcal{O}$, as defined in Section \ref{definition_distance_to_a_set}.
\end{remark}


Let $\mathbf{v}_d:\mathcal{X}_r^{\circ}\to\mathbb{R}^n$ be the first-order motion planner such that for the closed-loop system 
\begin{equation}\label{kinematic_closed_loop_system}\dot{\mathbf{x}} = \mathbf{v}_d(\mathbf{x})
\end{equation}
the following properties hold:
\begin{enumerate}
\renewcommand{\labelenumi}{P\arabic{enumi}}
\renewcommand{\theenumi}{P\arabic{enumi}}
    \item \label{property:move_away}There exists $\mu>0$ and $\delta_d\in (0, \delta_u]$ such that the following inequality  \[\mathbf{v}_d(\mathbf{x})^\top\boldsymbol{\eta}(\mathbf{x})\geq \mu\] 
    holds for all $\mathbf{x}\in\mathcal{N}_{\delta_d}(\mathcal{X}_r^c)$, where the existence of $\delta_u > 0$ is assumed in Condition \ref{condition:uniqueness} of Assumption \ref{assumption:twice_differentiable_boundary} and $\boldsymbol{\eta}(\mathbf{x})$ is defined in Remark \ref{remark:twice_differentiability}.

    \item\label{property:equilibrium_set}
    The equilibrium set is given by $\mathcal{E}\cup\{\mathbf{x}_d\}$, where $\mathcal{E}$ is a finite set of undesired isolated equilibrium points given by
    \[
    \mathcal{E} = \{\mathbf{x}\in\mathcal{X}_r^{\circ}\mid\mathbf{v}_d(\mathbf{x}) = \mathbf{0}, \mathbf{x}\ne\mathbf{x}_d\},
    \]
    with $\mathbf{x}_d\in\mathcal{X}_r^{\circ}$ being the desired equilibrium point satisfying $\mathbf{v}_d(\mathbf{x}_d) = \mathbf{0}$.
    \item\label{property:vd_continuously_differentiable} The first-order motion planner $\mathbf{v}_d$ is continuously differentiable on $\mathcal{X}_r^{\circ}$.
    
    

    \item\label{property:eigenvalue_analysis}The Jacobian $\mathbf{M}_d(\mathbf{x})$ satisfies the following properties:
    \begin{enumerate}
        \item \label{property:stable_equilibria}Every eigenvalue of $\mathbf{M}_d(\mathbf{x}_d)$ has negative real part.
        \item \label{property:unstable_equilibria}For every $\mathbf{x}^*\in\mathcal{E}$, $\mathbf{M}_d(\mathbf{x}^*)$ has at least one eigenvalue with positive real part.
    \end{enumerate}

    \item\label{property:distance_differentiable_at_E}
    Every point $\mathbf{x}$ in an open neighborhood of each element of $\mathcal{E}\cup\{\mathbf{x}_d\}$ has a unique closest point on $\mathcal{O}$, \textit{i.e.}, $\mathrm{card}(\mathcal{P}(\mathbf{x}, \mathcal{O})) = 1$.
\end{enumerate}

\begin{remark}\label{remark:explanation-of-properties}
    Property \ref{property:move_away} ensures that whenever $d_r(\mathbf{x})\leq\delta_d$, the vector field $\mathbf{v}_d$ steers $\mathbf{x}$ away from obstacles in the direction of increasing $d_r$.
    Consequently, $\mathcal{X}_r^{\circ}$ is forward invariant for the closed-loop system \eqref{kinematic_closed_loop_system}.
    In particular, the positive lower bound $\mu$ in this property ensures that, for every initial condition $\mathbf{x}(0)\in\mathcal{X}_r^{\circ}$, the corresponding state trajectory $\mathbf{x}(t)$ remains in a closed set
    \begin{equation*}
        \mathcal{K}(\mathbf{x}(0)) := \{\mathbf{x}\in\mathcal{X}_r^{\circ}\mid d_r(\mathbf{x}) \geq \min\{d_r(\mathbf{x}(0)), \delta_d\}\}\subset\mathcal{X}_r^{\circ}
    \end{equation*}
    for all $t\geq 0$.

    Property \ref{property:equilibrium_set} states that the equilibrium points of \eqref{kinematic_closed_loop_system} are isolated. 
    Since $\mathbf{v}_d$ is continuously differentiable on $\mathcal{X}_r^{\circ}$ by Property \ref{property:vd_continuously_differentiable}, the Jacobian matrix at each equilibrium point is well defined, and its eigenvalues can be used to characterize local stability, as described in Property \ref{property:eigenvalue_analysis}.
    Property \ref{property:stable_equilibria} ensures that the desired equilibrium point $\mathbf{x}_d$ is asymptotically stable.
    Moreover, Property \ref{property:unstable_equilibria}, together with the center manifold theorem \cite[Section 2.7, p.~116]{perko2013differential}, implies that every equilibrium point in $\mathcal{E}$ is unstable and has a basin of attraction of Lebesgue measure zero in $\mathbb{R}^n$.

    Property \ref{property:distance_differentiable_at_E} requires that every equilibrium point in $\mathcal{E}\cup\{\mathbf{x}_d\}$ have an open neighborhood in which each point has a unique closest point on $\mathcal{O}$ in the sense of the Euclidean norm.
    This ensures that $d_r$ is differentiable at every equilibrium point in $\mathcal{E}\cup\{\mathbf{x}_d\}$.
\end{remark}

Two distinct problems are considered, characterized by additional properties of the first-order motion planner $\mathbf{v}_d$, alongside Properties \ref{property:move_away}-\ref{property:distance_differentiable_at_E}.

\begin{problem}\label{problem_with_existence_function}
     We assume that $\mathbf{v}_d(\mathbf{x})=-k_1\nabla_{\mathbf{x}}\varphi(\mathbf{x})$ satisfies Properties \ref{property:move_away}-\ref{property:distance_differentiable_at_E}, where $k_1>0$ and $\varphi:\mathcal{X}_r^{\circ}\to\mathbb{R}_{\geq 0}$ is a twice continuously differentiable navigation function that is positive definite with respect to $\mathbf{x}_d$\footnote{If $\varphi:\mathcal{X}_r^{\circ}\to\mathbb{R}_{\geq 0}$ is positive definite function with respect to $\mathbf{x}_d$, then $\varphi(\mathbf{x}_d) = 0$ and $\varphi(\mathbf{x}) > 0$ for all $\mathbf{x}\in\mathcal{X}_r^{\circ}\setminus\{\mathbf{x}_d\}$.}. 
     We additionally require that $\varphi$ have bounded sublevel sets, \textit{i.e.,} for every $c \geq 0$, the set $\{\mathbf{x}\in\mathcal{X}_r^{\circ}\mid\varphi(\mathbf{x}) \leq c\}$ is bounded.
\end{problem}

\begin{problem}\label{problem_with_VF}
We assume knowledge of a first-order motion planner $\mathbf{v}_d:\mathcal{X}_r^{\circ}\to\mathbb{R}^n$ that satisfies Properties \ref{property:move_away}-\ref{property:distance_differentiable_at_E}.
In addition, we assume that $\mathbf{v}_d$ is bounded on $\mathcal{X}_r^{\circ}$, \textit{i.e.}, there exists $\bar{D}_d > 0$ such that $\|\mathbf{v}_d(\mathbf{x})\| \leq \bar{D}_d$ for all $\mathbf{x}\in\mathcal{X}_r^{\circ}$.
Moreover, if $\mathcal{X}_r^{\circ}$ is unbounded, we assume that there exists $R > 0$ such that $(\mathbf{x} - \mathbf{x}_d)^\top\mathbf{v}_d(\mathbf{x}) \leq 0$ for every $\mathbf{x}\in\mathcal{X}_r^{\circ}$ satisfying $\|\mathbf{x} - \mathbf{x}_d\| \geq R$.
Finally, we assume that the chain recurrent set associated with $\dot{\mathbf{x}} = \mathbf{v}_d(\mathbf{x})$, in the sense of \cite[Definition 4]{welde2023compositional}, is a subset of $\mathcal{E}\cup\{\mathbf{x}_d\}$.
\end{problem}

\begin{remark}
    In Problem \ref{problem_with_VF}, we require that the chain recurrent set of \eqref{kinematic_closed_loop_system} is a subset of the finite set of isolated equilibrium points.
    Loosely speaking, a point $\mathbf{x}_0\in\mathcal{X}_r^{\circ}$ is chain recurrent if, one can start at $\mathbf{x}_0$ and return to it by following a finite sequence of trajectory segments, each governed by \eqref{kinematic_closed_loop_system} and lasting at least an arbitrarily prescribed time, such that the endpoint of each segment is arbitrarily close to the starting point of the next segment.\\
    We refer readers to Appendix \ref{appendix:chain_recurrent_set_concept} for the formal definition and related terminology.
    Together with Property \ref{property:equilibrium_set}, the closed-loop system \eqref{kinematic_closed_loop_system} satisfying this requirement excludes complex limiting behavior such as nonisolated equilibrium sets, periodic orbits, homoclinic orbits, and heteroclinic cycles.
\end{remark}

\begin{remark}
    Properties \ref{property:move_away}-\ref{property:distance_differentiable_at_E} alone do not guarantee global attractivity of $\mathcal{E}\cup\{\mathbf{x}_d\}$ for the closed-loop system \eqref{kinematic_closed_loop_system}.
    The additional properties mentioned in Problems \ref{problem_with_existence_function} and \ref{problem_with_VF} provide this guarantee through different arguments.

    In Problem \ref{problem_with_existence_function}, for \eqref{kinematic_closed_loop_system}, the gradient structure yields
    \[
    \dot{\varphi}(\mathbf{x}) = -k_1\|\nabla_{\mathbf{x}}\varphi(\mathbf{x})\|^2,\, \forall \mathbf{x}\in\mathcal{X}_r^{\circ},
    \]
    which vanishes precisely at the equilibrium points.
    Therefore, every solution remains in a sublevel set of $\varphi$, which is bounded by properties of $\varphi$ mentioned in Problem \ref{problem_with_existence_function}.
    Together with the positive clearance from obstacles boundary guaranteed by Property \ref{property:move_away}, as stated in Remark \ref{remark:explanation-of-properties}, this establishes precompactness\footnote{A state trajectory is called precompact in its state space if its image has compact closure in that space. Equivalently, its image is contained in a compact subset of the state space.} of every solution of \eqref{kinematic_closed_loop_system} in $\mathcal{X}_r^{\circ}$.
    LaSalle's invariance principle then implies global attractivity of $\mathcal{E}\cup\{\mathbf{x}_d\}$ relative to $\mathcal{X}_r^{\circ}$.

    In Problem \ref{problem_with_VF}, boundedness of
    $\mathbf{x}(t)$ is immediate when $\mathcal{X}_r^{\circ}$ is bounded. 
    When $\mathcal{X}_r^{\circ}$ is unbounded, the satisfaction of \[(\mathbf{x} - \mathbf{x}_d)^\top\mathbf{v}_d(\mathbf{x})\leq 0,\] for every $\mathbf{x}\in\mathcal{X}_r^{\circ}$ satisfying $\|\mathbf{x} - \mathbf{x}_d\| \geq R$,  for some $R > 0$, prevents $\|\mathbf{x} - \mathbf{x}_d\|$ from increasing outside the Euclidean ball of radius $R$ centered at $\mathbf{x}_d$, for the closed-loop system \eqref{kinematic_closed_loop_system}. 
    Thus, $\mathbf{x}(t)$ remains bounded $\forall t\geq 0$. 
    As in Problem \ref{problem_with_existence_function}, this boundedness, together with Property \ref{property:move_away}, guarantees that every solution of \eqref{kinematic_closed_loop_system} is precompact in $\mathcal{X}_r^{\circ}$.
    This precompactness and the characterization of the chain recurrent set in Problem \ref{problem_with_VF} allows us to apply \cite[Proposition 1.2]{mischaikow1995asymptotically} to conclude global attractivity of $\mathcal{E}\cup\{\mathbf{x}_d\}$ for the closed-loop system \eqref{kinematic_closed_loop_system} over $\mathcal{X}_r^{\circ}$.

    In either case, global attractivity of $\mathcal{E}\cup\{\mathbf{x}_d\}$, together with continuity of $\mathbf{x}(t), \forall t\geq 0$, and finiteness of $\mathcal{E}$, ensures that, for every initial condition $\mathbf{x}(0)\in\mathcal{X}_r^{\circ}$, there exists a unique equilibrium $\mathbf{x}^*\in\mathcal{E}\cup\{\mathbf{x}_d\}$ such that $\Lim_{t\to\infty}\mathbf{x}(t) = \mathbf{x}^*$.
    Property \ref{property:eigenvalue_analysis} then establishes almost global asymptotic stability of $\mathbf{x}_d$ for the closed-loop \eqref{kinematic_closed_loop_system} over $\mathcal{X}_r^{\circ}$.
\end{remark}


Given a first-order motion planner $\mathbf{v}_d:\mathcal{X}_r^{\circ}\to\mathbb{R}^n$ that satisfies Properties \ref{property:move_away}-\ref{property:distance_differentiable_at_E} and additional requirements as mentioned earlier, the objective is to design a feedback control $\mathbf{u}$ in \eqref{equation:second_order_system} such that for the resulting closed-loop system the following statements hold:
\begin{enumerate}
    \item The set $\mathcal{X}_r^{\circ}\times\mathbb{R}^n$ is forward invariant.
    \item The equilibrium point $(\mathbf{x}_d, \mathbf{0})$ is almost globally asymptotically stable over $\mathcal{X}_r^{\circ}\times\mathbb{R}^n$.
\end{enumerate}

\section{Feedback control design}\label{section:feedback_control_design}
If the first-order planner $\mathbf{u} = \mathbf{v}_d(\mathbf{x})$ is applied to the second-order system \eqref{equation:second_order_system} without incorporating damping, the $\mathbf{x}$-trajectory of the resulting closed-loop system may exhibit an overshoot. This overshoot may lead the system into unsafe regions, potentially resulting in collisions with obstacles.  
To avoid such overshoots, a damping vector of the form $-k \mathbf{v}$, with a sufficiently high gain $k > 0$, can be introduced. 
However, excessively high damping significantly reduces the robot's velocity, leading to slow convergence to the desired target state.  

To address this trade-off, we propose a dynamic damping gain that adapts based on the robot's proximity to obstacles. 
Specifically, the damping gain remains low when the robot is far from obstacles, allowing for fast motion, and increases as the robot approaches obstacle boundaries.
This mechanism ensures safety by reducing the robot's velocity near obstacles while maintaining efficient progress toward the target in obstacle-free regions.

\subsection{Feedback control law for Problem \ref{problem_with_existence_function}}\label{section:dynamic_damping_feedback_control_design}
The proposed dynamic damping feedback (DDF) control  $\mathbf{u} = \mathbf{u}_d(\boldsymbol{\xi})$ is given by
\begin{equation}\label{controller:DDF}
    \mathbf{u}_d(\boldsymbol{\xi}) = -k_1\nabla_{\mathbf{x}}\varphi(\mathbf{x}) - k_d\beta(d_r(\mathbf{x}))\mathbf{v},
\end{equation}
where $k_1>0$, $k_d > 0$ and $d_r(\mathbf{x}) = d_s(\mathbf{x}, \mathcal{O}) - r$ is defined in \eqref{distance_function_d_x}.
The composite state vector $\boldsymbol{\xi}$ is given by $\boldsymbol{\xi}:=(\mathbf{x}, \mathbf{v})\in\mathcal{X}_r^{\circ}\times\mathbb{R}^n$.
The navigation function $\varphi$ satisfies the properties mentioned in Problem \ref{problem_with_existence_function}.
Given $p\in\mathbb{R}_{>0}$, the scalar function $\beta:\mathbb{R}_{>0}\to\mathbb{R}_{>0}$ is defined as
\begin{equation}\label{beta_function_definition}
    \beta(p) = \begin{cases}
        1, & p \geq \epsilon_2,\\
        \phi(p), & \epsilon_1\leq p\leq \epsilon_2, \\
        p^{-1}, & 0<p \leq \epsilon_1,
    \end{cases}
\end{equation}
where $\epsilon_1 > 0$, $\epsilon_2 > \epsilon_1$ and $\phi(p):[\epsilon_1, \epsilon_2]\to\mathbb{R}_{>0}$ is a continuously differentiable function such that $\phi(\epsilon_1) = \epsilon_1^{-1}$, $\phi(\epsilon_2) = 1$, $\phi'(\epsilon_1) = -\epsilon_1^{-2}$, $\phi'(\epsilon_2) = 0$ and $\phi(p) > 0, \;\forall p\in[\epsilon_1, \epsilon_2]$.\footnote{An example of such a function is $\phi(p) = (1 - b(p))\frac{1}{p} + b(p)$, where the mapping $b:[\epsilon_1, \epsilon_2]\to[0, 1]$ is given by $b(p) = 3s(p)^2 - 2s(p)^3$ and $s(p) = \frac{p - \epsilon_1}{\epsilon_2 - \epsilon_1}$.} 
The block diagram representation of the DDF controller is provided in Fig. \ref{block_diagram:DDF}.


According to \eqref{controller:DDF} and \eqref{beta_function_definition}, the damping gain remains constant when the robot is at least $\epsilon_2$ units away from obstacles. However, when the robot is within $\epsilon_1$ units of the obstacles and moves toward them, the damping gain increases.
This increase in damping reduces the robot’s velocity, eventually causing it to move in the direction of $ -k_1 \nabla_{\mathbf{x}} \varphi(\mathbf{x}) $.
According to Property \ref{property:move_away}, when $d_r(\mathbf{x})$ is small, the vector $-k_1\nabla_{\mathbf{x}}\varphi(\mathbf{x})$ always belongs to the open half-space $\mathcal{H}_{>}(\mathbf{0}, \boldsymbol{\eta}(\mathbf{x}))$ and points away from nearby obstacles, ensuring that the robot eventually moves away from them. 
These observations are summarized in the following lemma:

\begin{lemma}\label{lemma:DDF}
Consider the closed-loop system \eqref{equation:second_order_system}-\eqref{controller:DDF}, under Assumption \ref{assumption:twice_differentiable_boundary}. 
If $d_r(\mathbf{x}(0)) > 0$, then the following statements hold:
\begin{enumerate}
    \item\label{lemma:NF_claim1} $d_r(\mathbf{x}(t)) > 0$ for all $t\geq 0$.
    \item\label{lemma:NF_claim2} There exists $\sigma(\boldsymbol{\xi}(0)) > 0$ and $t_{\sigma}(\boldsymbol{\xi}(0)) \geq 0$ such that
    \[
    d_r(\mathbf{x}(t))\geq \sigma(\boldsymbol{\xi}(0)), \forall t\geq t_{\sigma}(\boldsymbol{\xi}(0)).
    \]
    \item\label{lemma:NF_claim3} There exists $B_d(\boldsymbol{\xi}(0)) > 0$ such that the control input satisfies 
    \[\|\mathbf{u}_d(\boldsymbol{\xi}(t))\|\leq B_d(\boldsymbol{\xi}(0)), \forall t\geq 0.\]
\end{enumerate}
\end{lemma}
\proof{
    See Appendix \ref{proof:lemma:DDF}.
    }

\begin{figure}
    \centering
    \includegraphics[width=0.933\linewidth]{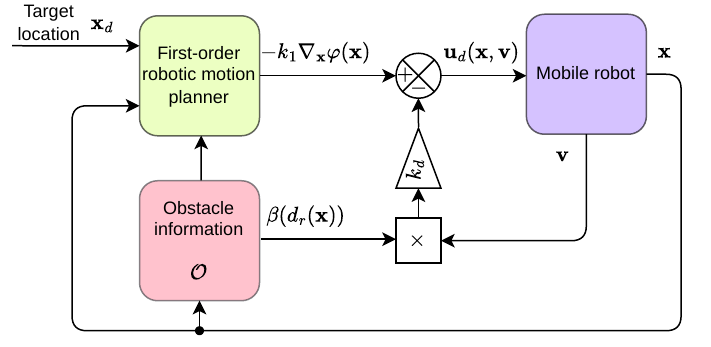}
    \caption{Block diagram of the dynamic damping feedback controller \eqref{controller:DDF}.}
    \label{block_diagram:DDF}
\end{figure}

Now, we show that the DDF input \eqref{controller:DDF} ensures almost global asymptotic stability of the target state $(\mathbf{x}_d, \mathbf{0})$ for the closed-loop system \eqref{equation:second_order_system}-\eqref{controller:DDF} over $\mathcal{X}_r^{\circ}\times\mathbb{R}^n$.

\begin{theorem}\label{theorem:DDF}
    For the closed-loop system \eqref{equation:second_order_system}-\eqref{controller:DDF}, under Assumption \ref{assumption:twice_differentiable_boundary}, the following statements hold:
    \begin{enumerate}
    \item The set $\mathcal{X}_r^{\circ}\times\mathbb{R}^n$ is forward invariant\label{theorem:NF_claim1}.
    \item \label{theorem:NF_claim2} The set of equilibrium points is given by $\mathcal{S}\cup\{(\mathbf{x}_d, \mathbf{0})\}$, where the set $\mathcal{S}$ is defined as \begin{equation}\label{definition:undesired_target_set}
            \mathcal{S}:=\{(\mathbf{x}, \mathbf{0})\in\mathcal{X}_r^{\circ}\times\{\mathbf{0}\}\mid\mathbf{x}\in\mathcal{E}\}.
        \end{equation}

    \item {The equilibrium point $(\mathbf{x}_d, \mathbf{0})$ is almost globally asymptotically stable over $\mathcal{X}_r^{\circ}\times\mathbb{R}^n$.}\label{theorem:NF_claim3}
    \end{enumerate}
\end{theorem}
\proof{
    See Appendix \ref{proof:theorem:DDF}.}


The DDF control input $\mathbf{u}_d(\mathbf{x}, \mathbf{v})$, as defined in \eqref{controller:DDF}, requires knowledge of a navigation function $\varphi(\mathbf{x})$ such that the first-order motion planner of the form $-k_1\nabla_{\mathbf{x}}\varphi(\mathbf{x})$ satisfies Properties \ref{property:move_away}-\ref{property:distance_differentiable_at_E}, along with other conditions stated in Problem \ref{problem_with_existence_function}. 
However, ensuring the existence and knowledge of such functions whose negative gradient with respect to $\mathbf{x}$ would align with a given first-order motion planner $\mathbf{v}_d(\mathbf{x})$ is a challenging task. 
Consequently, in the next section, we propose an alternative feedback control input design for Problem \ref{problem_with_VF}, where the knowledge of $\varphi(\mathbf{x})$ is not required.

\subsection{Feedback control law for Problem \ref{problem_with_VF}}
The proposed velocity tracking feedback (VTF) control $\mathbf{u} = \mathbf{u}_v(\boldsymbol{\xi})$ is given by
\begin{equation}\label{controller:VTF}
    \mathbf{u}_v(\boldsymbol{\xi}) =  -  k_d\beta(d_r(\mathbf{x}))(\mathbf{v} - \mathbf{v}_d(\mathbf{x})) + \mathbf{M}_d(\mathbf{x})\mathbf{v},
\end{equation}
where $k_d > 0$.
The first-order motion planner $\mathbf{v}_d:\mathcal{X}_r^{\circ}\to\mathbb{R}^n$ satisfies Properties \ref{property:move_away}-\ref{property:distance_differentiable_at_E}, and additional requirements mentioned in Problem \ref{problem_with_VF}.
The scalar function $\beta(\cdot)$ is defined in \eqref{beta_function_definition}, and $\mathbf{M}_d(\mathbf{x}) = \frac{\partial\mathbf{v}_d(\mathbf{x})}{\partial\mathbf{x}}$.
The block diagram representation of the VTF controller is provided in Fig. \ref{block_diagram:VTF}.

The control law $\mathbf{u}_v$ is designed so that the scalar function 
\begin{equation}
    V_v(\boldsymbol{\xi}) = \frac{1}{2}\|\mathbf{v} - \mathbf{v}_d(\mathbf{x})\|^2
\end{equation}
is non-increasing along the trajectories of the closed-loop system \eqref{equation:second_order_system}-\eqref{controller:VTF} when $\mathbf{x}\in\mathcal{X}_r^{\circ}$.
Intuitively, the VTF control input drives $\mathbf{v}$ toward the desired velocity $\mathbf{v}_d(\mathbf{x})$, thereby encouraging the second-order system to follow the motion generated by the first-order planner $\mathbf{v}_d$.

According to \eqref{beta_function_definition} and \eqref{controller:VTF}, similar to \eqref{controller:DDF}, the proposed feedback control input in \eqref{controller:VTF} is undefined when $d_r(\mathbf{x}) = 0$. However, if $\mathbf{x}$ is initialized in the interior of the free space $ \mathcal{X}_r$ \textit{i.e.,} if $\mathbf{x}(0)\in\mathcal{X}_r^{\circ}$, then $\mathbf{x}(t)\in\mathcal{X}_r^{\circ}$ for all $t\geq 0$ and it does not approach the boundary of the unsafe region $\partial \mathcal{X}_r^c$ as $t\to\infty$, as stated in the next lemma.

\begin{lemma}\label{lemma:VF}
Consider the closed-loop system \eqref{equation:second_order_system}-\eqref{controller:VTF}, under Assumption \ref{assumption:twice_differentiable_boundary}. 
If $d_r(\mathbf{x}(0)) > 0$, then the following statements hold:
\begin{enumerate}
    \item\label{lemma:VF_claim1} $d_r(\mathbf{x}(t)) > 0$ for all $t\geq 0$.
    \item\label{lemma:VF_claim2} There exists a time $t_v(\boldsymbol{\xi}(0)) \geq 0$ such that
    \[
    d_r(\mathbf{x}(t))\geq \delta_d, \forall t\geq t_v(\boldsymbol{\xi}(0)),
    \]
    where the existence of $\delta_d > 0$ is guaranteed by Property \ref{property:move_away}.
    \item\label{lemma:VF_claim3} There exists $B_{v}(\boldsymbol{\xi}(0)) > 0$ such that the control input satisfies 
    \[\|\mathbf{u}_v(\boldsymbol{\xi}(t))\|\leq B_{v}(\boldsymbol{\xi}(0)), \forall t\geq 0.\]
\end{enumerate}\end{lemma}
\proof{See Appendix \ref{proof:lemma:VF}.}

\begin{figure}
    \centering
    \includegraphics[width=1\linewidth]{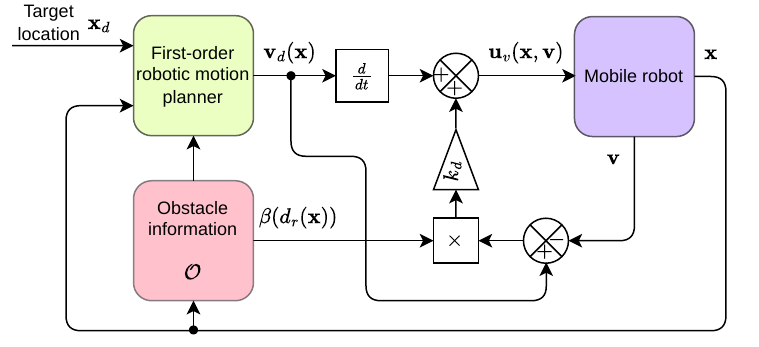}
    \caption{Block diagram of the velocity tracking feedback controller \eqref{controller:VTF}.}
    \label{block_diagram:VTF}
\end{figure}

The VTF control input \eqref{controller:VTF} ensures a monotonic decrease in the norm of the error between the robot's velocity $\mathbf{v}(t)$ and the desired velocity prescribed by the first-order motion planner $\mathbf{v}_d$ for all $t\geq 0$.
This enables us to establish the almost global asymptotic stability of the target state $(\mathbf{x}_d, \mathbf{0})$ for the closed-loop system \eqref{equation:second_order_system}-\eqref{controller:VTF} over $\mathcal{X}_r^{\circ} \times \mathbb{R}^n$, as stated in the next theorem.

\begin{theorem}\label{theorem:VF}
    For the closed-loop system \eqref{equation:second_order_system}-\eqref{controller:VTF}, under Assumption \ref{assumption:twice_differentiable_boundary}, the following statements hold:
    \begin{enumerate}
        \item\label{theorem:VF_claim1} The set $\mathcal{X}_r^{\circ}\times\mathbb{R}^n$ is forward invariant.
        \item\label{theorem:VF_decrease} For all $t\geq 0$, $\|\mathbf{v}(t) - \mathbf{v}_d(\mathbf{x}(t))\|$ is nonincreasing and converges exponentially to zero.
        \item \label{theorem:VF_claim2}The set of equilibrium points is given by $\mathcal{S}\cup\{(\mathbf{x}_d, \mathbf{0})\}$, where the set $\mathcal{S}$ is defined in \eqref{definition:undesired_target_set}. 
        \item\label{theorem:VF_claim3} The equilibrium point $(\mathbf{x}_d, \mathbf{0})$ is almost globally asymptotically stable over $\mathcal{X}_r^{\circ}\times\mathbb{R}^n$.
    \end{enumerate}
\end{theorem}
\proof{
    See Appendix \ref{proof:theorem_first_order_convergence}.
}

\section{Examples of first-order motion planners}\label{section:examples-of-motion-planners}

Implementation of the proposed control laws \eqref{controller:DDF} and \eqref{controller:VTF} depends on first-order motion planners $\mathbf{v}_d$ satisfying Properties \ref{property:move_away}-\ref{property:distance_differentiable_at_E}, as well as the additional requirements specified in Problems \ref{problem_with_existence_function} and \ref{problem_with_VF}.
In this section, we present two first-order motion planners $\mathbf{v}_d$ that satisfy these properties under different geometric assumptions.
The first planner is a negative gradient of a navigation function proposed in \cite{paternain2017navigation} and is applicable in environments with convex obstacles, provided they satisfy specific curvature conditions.
Later, in Section \ref{section:simulations}, we adapt this planner to second-order systems using both the DDF \eqref{controller:DDF} and the VTF \eqref{controller:VTF} control schemes.
The second planner is inspired by \cite{kumar2022navigation} and is not constructed from a navigation function.
Instead, it uses the center of each ellipsoidal obstacle to generate a repulsive vector field.
Since a navigation function is not available for this planner, we adapt it to second-order systems only through the VTF \eqref{controller:VTF} control scheme, as shown later in Section \ref{section:simulations}.

For both designs, the workspace $\mathcal{W}$ is a compact, convex subset of $\mathbb{R}^n$ and contains a finite collection of compact, convex obstacles $\mathcal{O}_i, i \in \mathbb{I}$ with non-empty interior that satisfy Assumption \ref{assumption:twice_differentiable_boundary}. 
Consequently, the boundary $\partial\mathcal{X}_0$ of the free space is of class $C^2$, and there exists $\delta_u > 0$ such that for all $i, j \in \mathbb{I}\cup\{0\}, i \neq j$,
\begin{equation}
    d(\mathcal{O}_i, \mathcal{O}_j) \geq 2(r + \delta_u),
\end{equation}
where $d(\mathcal{O}_i, \mathcal{O}_j) := \underset{\mathbf{a}\in\mathcal{O}_i, \mathbf{b}\in\mathcal{O}_j}{\min } \|\mathbf{a} - \mathbf{b}\|$.

\subsection{Navigation function based first-order motion planner}\label{section:navigation-function-based-first-order_controller}
In this subsection, we present an example of a first-order motion planner $\mathbf{v}_d$ based on a navigation function. This motion planner is then extended via the DDF control strategy \eqref{controller:DDF} to the second-order system \eqref{equation:second_order_system}, and simulation results are presented in Section \ref{section:simulations}.\\
We consider the navigation function-based first-order motion planner proposed in \cite{paternain2017navigation}.
Each obstacle $\mathcal{O}_i, i\in\mathbb{I},$ is represented as the zero sublevel set of a twice continuously differentiable, convex function $h_i$:
\begin{equation}\label{individual_obstacle_definition_using_hi}
    \mathcal{O}_i = \{\mathbf{x}\in\mathcal{W}\mid h_i(\mathbf{x}) \leq 0\}.
\end{equation}
The navigation function $\varphi:\mathcal{X}_r\to[0, 1]$ is defined by
\begin{equation}\label{paternain_navigation_function_definition}
    \varphi(\mathbf{x}) = \frac{f(\mathbf{x})}{(f(\mathbf{x})^{\kappa} + h(\mathbf{x}))^{\frac{1}{\kappa}}},
\end{equation}
where $\kappa > 0$ is a tuning parameter.
The attractive potential $f$, whose global minimum is at the target position $\mathbf{x}_d$, is given by
\begin{equation}\label{attractive_potential_function}
    f(\mathbf{x}) = \delta_1\|\mathbf{x} - \mathbf{x}_d\|^2,
\end{equation}
where $\delta_1 > 0$.
The obstacle proximity function $h$ is defined as the product
\begin{equation}
    h(\mathbf{x}) = \prod_{i\in\mathbb{I}\cup\{0\}}h_i(\mathbf{x}),
\end{equation}
where $h_0$ is a twice continuously differentiable, concave function representing the workspace $\mathcal{W}$ as
\begin{equation}
    \mathcal{W} = \{\mathbf{x}\in\mathbb{R}^n\mid h_0(\mathbf{x}) \geq 0\}.
\end{equation}

The first-order motion planner $\mathbf{v}_d$ is the scaled negative gradient of the navigation function $\varphi$ defined in \eqref{paternain_navigation_function_definition}:
\begin{equation}\label{example:NF_based_vector_field}
    \mathbf{v}_d(\mathbf{x}) = \frac{k_1\left(\frac{f(\mathbf{x})}{\kappa}\nabla_{\mathbf{x}}h(\mathbf{x})-h(\mathbf{x})\nabla_{\mathbf{x}}f(\mathbf{x}) \right)}{(f(\mathbf{x})^{\kappa} + h(\mathbf{x}))^{\frac{1}{\kappa} + 1}},
\end{equation}
where $k_1 > 0$.

Since $\frac{\partial^2f(\mathbf{x})}{\partial\mathbf{x}^2} = 2\delta_1\mathbf{I}_n$, the ratio $\frac{\lambda_{\max}}{\lambda_{\min}}$ appearing in \cite[Theorem 2]{paternain2017navigation} is equal to one, and the sufficient curvature condition therein reduces to
\begin{equation}\label{curvature_condition_to_check}
    \frac{\nabla_{\mathbf{x}}h_i(\mathbf{x})^\top(\mathbf{x} - \mathbf{x}_d)}{\|\mathbf{x} - \mathbf{x}_d\|^2} < \mu_{\min}^i,
\end{equation}
for every $i\in\mathbb{I}$ and every $\mathbf{x}\in\partial\mathcal{O}_i$, where $\mu_{\min}^i>0$ is the smallest eigenvalue of the Hessian of $h_i$ over $\mathcal{X}_r$, as mentioned in \cite[Assumption 2]{paternain2017navigation}.

The first-order motion planner \eqref{example:NF_based_vector_field} applies to point-robot navigation \textit{i.e.}, $r = 0$.
If \eqref{curvature_condition_to_check} holds for every $i\in\mathbb{I}$ and each $\mathbf{x}\in\partial\mathcal{O}_i$, then according to \cite[Theorem 2]{paternain2017navigation}, there exists a constant $K > 0$ such that if $\kappa >  K$, then the target $\mathbf{x}_d$ is almost globally asymptotically stable for the closed-loop system \eqref{kinematic_closed_loop_system}-\eqref{example:NF_based_vector_field} over $\mathcal{X}_r$. 
One can verify that the closed-loop system \eqref{kinematic_closed_loop_system}-\eqref{example:NF_based_vector_field} satisfies Properties \ref{property:move_away}-\ref{property:distance_differentiable_at_E} along with the additional requirements mentioned in Problems \ref{problem_with_existence_function} and \ref{problem_with_VF}.
Therefore, later in Section \ref{section:simulations}, we use the DDF \eqref{controller:DDF} and the VTF \eqref{controller:VTF} control strategies to extend the applicability of $\mathbf{v}_d$ in \eqref{example:NF_based_vector_field} to the second-order system \eqref{equation:second_order_system}.

\subsubsection{Application to star worlds}\label{section:star-world-theory}
The navigation function-based first-order motion planner in \eqref{example:NF_based_vector_field} applies directly to workspace $\mathcal{W}$ containing smooth convex obstacles that satisfy the sufficient curvature condition \eqref{curvature_condition_to_check}.
Its applicability can be extended to star worlds through a suitable change of coordinates \cite{koditschek1992exact}.
Specifically, consider a workspace $\mathcal{W}$ containing star-shaped\footnote{A star-shaped set $\mathcal{A}\subset\mathbb{R}^n$ is characterized by the presence of a point $\mathbf{p}$ in $\mathcal{A}^{\circ}$ such that for every $\mathbf{x}\in\partial\mathcal{A}$, the line segment $\mathcal{L}_s(\mathbf{x}, \mathbf{p})$ belongs to $\mathcal{A}$.} obstacles $\mathcal{O}_i$, $i\in\mathbb{I}$, and let $\mathcal{X}_r$ denote its free space, defined using \eqref{equation_free_space}.
Under the conditions established in
\cite{koditschek1992exact}, one constructs a diffeomorphism
\[
    H:\mathcal{X}_r\to\mathcal{X}_r^s,
\]
where $\mathcal{X}_r^s$ is the free space of a corresponding sphere world.
The mapping $H$ transforms the outer boundary of the star world into the outer spherical boundary and maps each obstacle boundary $\partial\mathcal{O}_i$ to the boundary of a spherical obstacle $\mathcal{O}_i^s$.
Consequently, a navigation function defined on the sphere world can be transferred to the star world through $H$. 
The construction of $H$ is beyond the scope of this paper, and the interested reader is referred to \cite{koditschek1992exact}.


As shown in \cite[Section III.A.3]{koditschek1992exact}, let $\varphi:\mathcal{X}_r^s\to[0, 1]$ be the navigation function defined in \eqref{paternain_navigation_function_definition} for the sphere world $\mathcal{X}_r^s$, and let $H:\mathcal{X}_r\to\mathcal{X}_r^s$ be a diffeomorphism from the star world to the sphere world.
A navigation function on $\mathcal{X}_r$ is then obtained by composition: $\varphi^{\star}(\mathbf{x}) = \varphi(H(\mathbf{x}))$.
Consequently, the corresponding gradient-based first-order motion planner $\mathbf{v}_d$ is
\begin{equation}\label{star_world_motion_planner}
    \mathbf{v}_d(\mathbf{x}) = -k_1\frac{\partial H(\mathbf{x})}{\partial\mathbf{x}}^\top\nabla_{H(\mathbf{x})}\varphi(H(\mathbf{x})),
\end{equation}
where $k_1 > 0$.

Since $\varphi^{\star}=\varphi\circ H$ is a navigation function on
$\mathcal{X}_r$, it can be verified that, provided proper parameter tuning, the closed-loop system
\eqref{kinematic_closed_loop_system} with
$\mathbf{v}_d$ defined in \eqref{star_world_motion_planner} satisfies Properties
\ref{property:move_away}-\ref{property:distance_differentiable_at_E}.
In Section \ref{section:simulations}, we use the DDF control law \eqref{controller:DDF} with $\mathbf{v}_d$ \eqref{star_world_motion_planner} to control the second-order system \eqref{equation:second_order_system}.

\begin{figure*}[ht]
    \centering
    \begin{minipage}[t]{0.48\textwidth}
        \vspace{0pt}

        \begin{subfigure}{\linewidth}
            \centering
            \includegraphics[width=\linewidth]{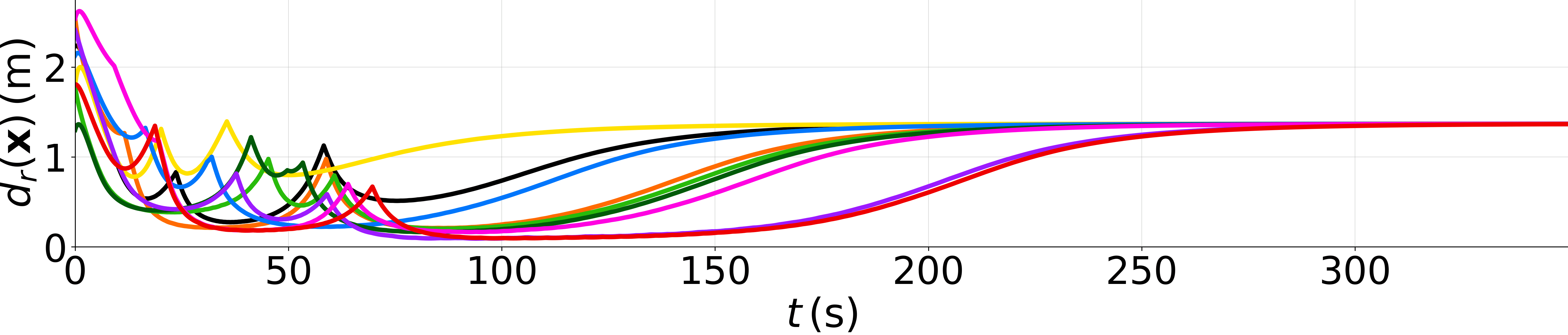}
            \caption{}
            \label{fig:planner1:DDF:distance-to-obstacles}
        \end{subfigure}

        \begin{subfigure}{\linewidth}
            \centering
            \includegraphics[width=\linewidth]{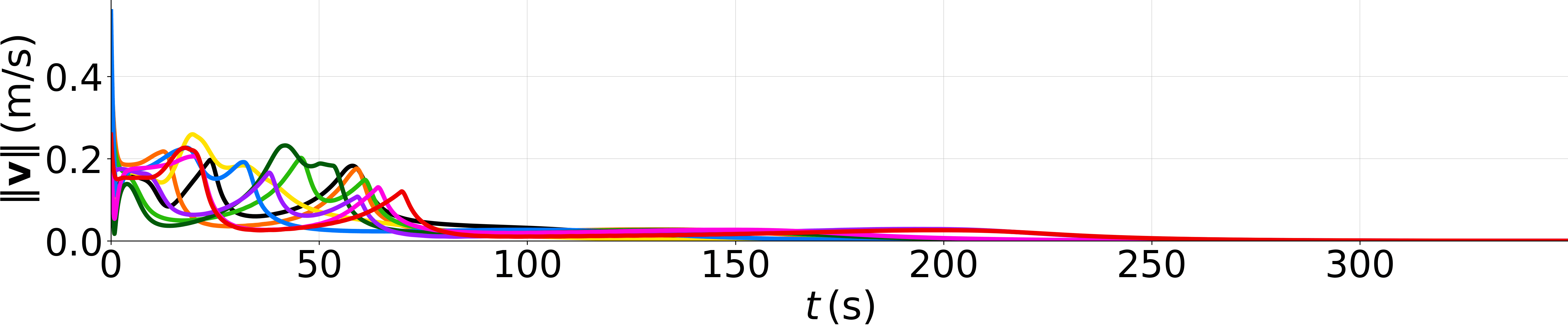}
            \caption{}
            \label{fig:planner1:DDF:velocity-norm}
        \end{subfigure}

        \begin{subfigure}{\linewidth}
            \centering
            \includegraphics[width=\linewidth]{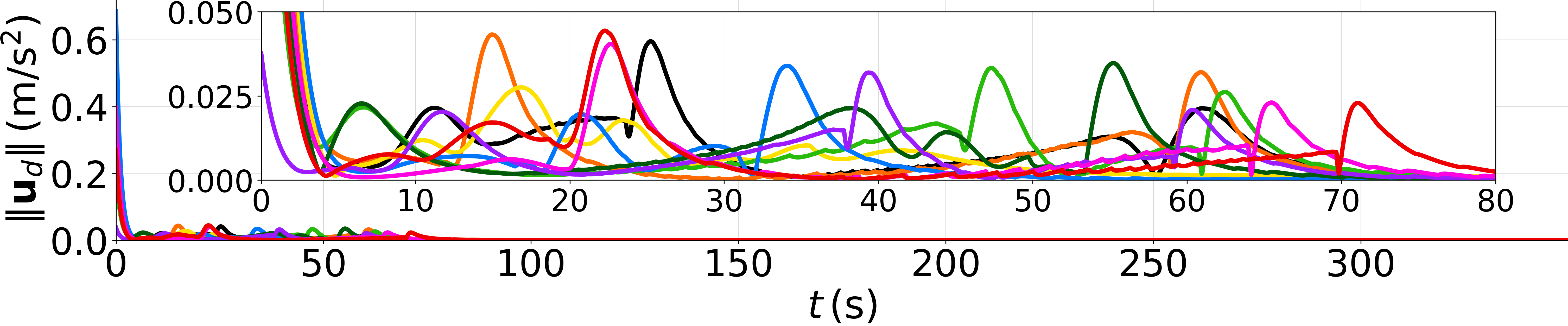}
            \caption{}
            \label{fig:planner1:DDF:acceleration-norm}
        \end{subfigure}
    \end{minipage}\hfill
    \begin{minipage}[t]{0.48\textwidth}
        \vspace{0pt}

        \begin{subfigure}{\linewidth}
            \centering
            \includegraphics[width=\linewidth]{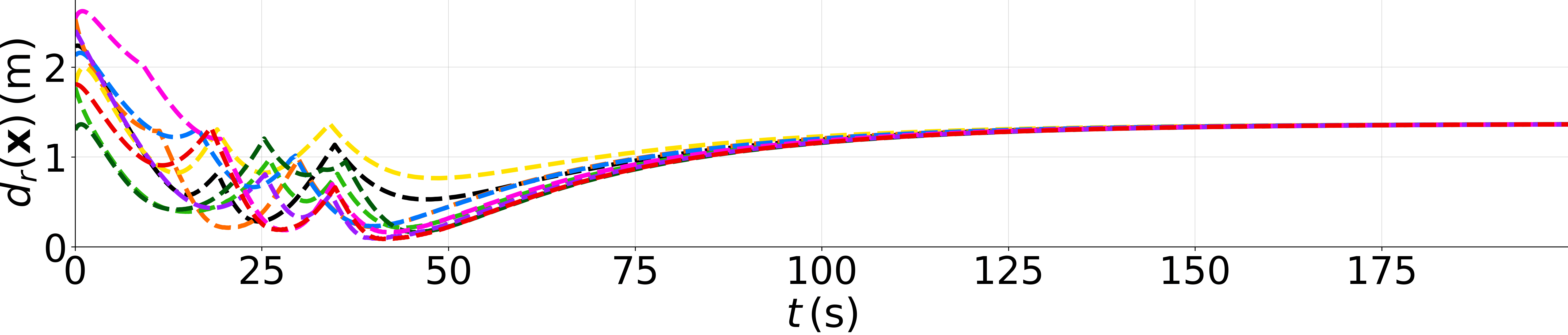}
            \caption{}
            \label{fig:planner1:VTF:distance-to-obstacles}
        \end{subfigure}

        \begin{subfigure}{\linewidth}
            \centering
            \includegraphics[width=\linewidth]{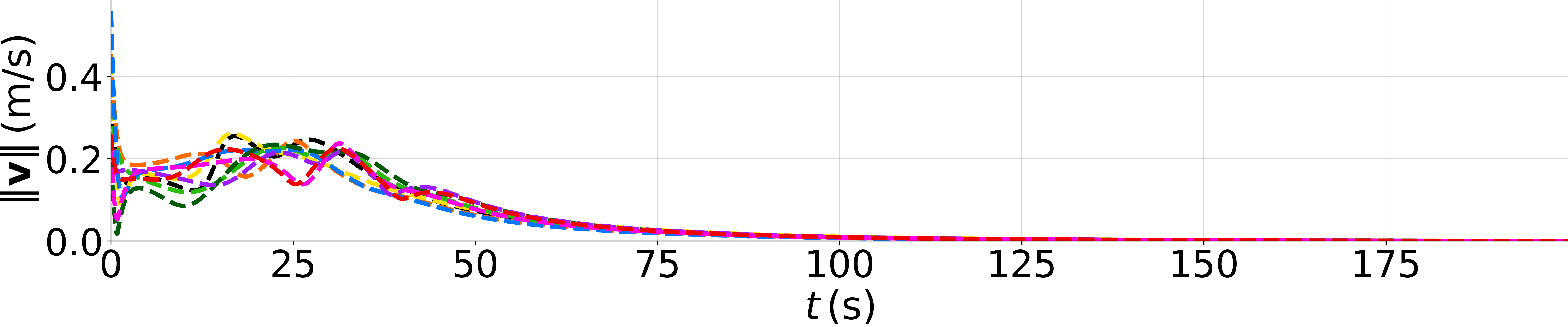}
            \caption{}
            \label{fig:planner1:VTF:velocity-norm}
        \end{subfigure}

        \begin{subfigure}{\linewidth}
            \centering
            \includegraphics[width=\linewidth]{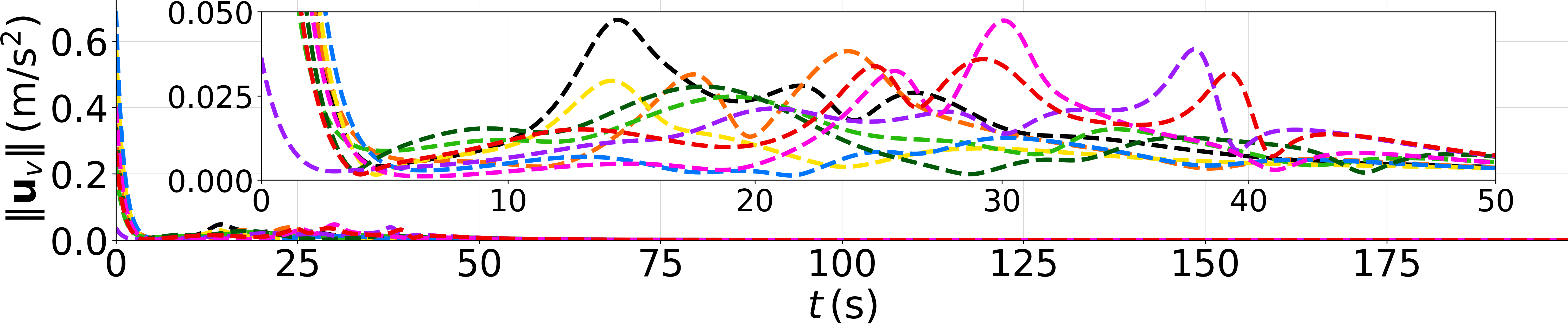}
            \caption{}
            \label{fig:planner1:VTF:acceleration-norm}
        \end{subfigure}
    \end{minipage}

    \caption{
    Plots obtained using the DDF and VTF control laws with $\mathbf{v}_d$ defined in \eqref{example:NF_based_vector_field}. (a)-(c) Plots of $d_r(\mathbf{x})$, $\|\mathbf{v}\|$ and $\|\mathbf{u}_d\|$ versus time obtained using the DDF control law.
    (d)-(f) Plots of $d_r(\mathbf{x})$, $\|\mathbf{v}\|$ and $\|\mathbf{u}_v\|$ versus time obtained using the VTF control law.
    }
    \label{fig:planner1:results-1}
\end{figure*}

\subsection{First-order motion planner without navigation function}\label{section:second_planner_setup}
In this subsection, we present an example of a first-order motion planner $\mathbf{v}_d$ that is not based on a navigation function. This motion planner is then extended via the VTF control strategy \eqref{controller:VTF} to the second-order system \eqref{equation:second_order_system}, and simulation results are presented in Section \ref{section:simulations}.\\
This example consists of a modified version of the first-order motion planner proposed in \cite{kumar2022navigation}, with modifications discussed later in Remark \ref{remark:modifications}.
The workspace $\mathcal{W}$ is a compact, convex subset of $\mathbb{R}^n$, and contains a finite collection of disjoint, compact ellipsoidal obstacles $\mathcal{O}_i, i\in\mathbb{I}$.
As in \cite[Assumption 1]{kumar2022navigation}, each ellipsoidal obstacle $\mathcal{O}_i$ is represented as the zero sublevel set of a convex quadratic function $e_i$:
\begin{equation}\label{ellipsoid_function_definition}
    e_i(\mathbf{x}) = (\mathbf{x} - \mathbf{x}_i)^\top\mathbf{S}_i(\mathbf{x} - \mathbf{x}_i) - 1,
\end{equation}
where $\mathbf{x}_i$ denotes the center of $\mathcal{O}_i$ and $\mathbf{S}_i\in\mathbb{R}^{n\times n}$ is a positive definite matrix that determines its shape and orientation such that
\begin{equation}\label{planner2:individual-obstacle-represenation}
    \mathcal{O}_i = \{\mathbf{x}\in\mathcal{W}\mid e_i(\mathbf{x})\leq 0\}, \forall i\in\mathbb{I}.
\end{equation}
The modified version of the first-order motion planner $\mathbf{v}_d$ defined in \cite[(15)]{kumar2022navigation} is given by
\begin{equation}\label{second_first_order_planner_example}
    \mathbf{v}_d(\mathbf{x}) = k_1\left(\frac{f(\mathbf{x})}{\kappa}\sum_{i\in\mathbb{I}}\bar{\alpha}_i(\mathbf{x})(\mathbf{x} - \mathbf{x}_i)-\alpha(\mathbf{x})(\mathbf{x} - \mathbf{x}_d)\right),
\end{equation}
where $k_1 > 0, \kappa > 0$, and the attractive potential function $f$ with global minimum at $\mathbf{x}_d\in\mathcal{X}_r^{\circ}$ is defined in \eqref{attractive_potential_function}.
The scalar functions $\alpha$ and $\bar{\alpha}_i,i\in\mathbb{I}$ are given by
\begin{equation}\label{alpha_product_definition}
    \alpha(\mathbf{x}) = \prod_{i\in\mathbb{I}}\alpha_i(\mathbf{x}) \text{ and } \bar{\alpha}_i(\mathbf{x}) = (1-\alpha_i(\mathbf{x}))\prod_{j\in\mathbb{I}\setminus\{i\}}\alpha_j(\mathbf{x}),
\end{equation}
where given $\nu > 0$, for each $i\in\mathbb{I}$, the scalar function $\alpha_i$ is defined as
\begin{equation}\label{individual_alpha_definition}
    \alpha_i(\mathbf{x}) = \begin{cases}
        \phi_1(e_i(\mathbf{x})), & 0\leq e_i(\mathbf{x}) \leq \nu,\\
        1, & e_i(\mathbf{x})\geq \nu.
    \end{cases}
\end{equation}
The scalar mapping $\phi_1:[0, \nu]\to[0, 1]$ is real analytic
and satisfies $\phi_1(0) = 0, \phi_1(\nu) = 1$, $\phi_1'(0) = 1$, $\phi_1'(\nu) = 0$, and $\phi'(p) > 0, \forall p\in [0, \nu)$.\footnote{An example of such a function is $\phi_1(p) = (1 - b_1(p))p + b_1(p)$, where $b_1(p) = 3s_1(p)^2 - 2s_1(p)^3$ and $s_1(p)=\frac{p}{\nu}$. This $\phi_1$ is strictly increasing on $[0, \nu]$ provided that $\nu \leq 1$.}

\begin{figure}[ht]
    \centering
    \includegraphics[width=0.8\linewidth]{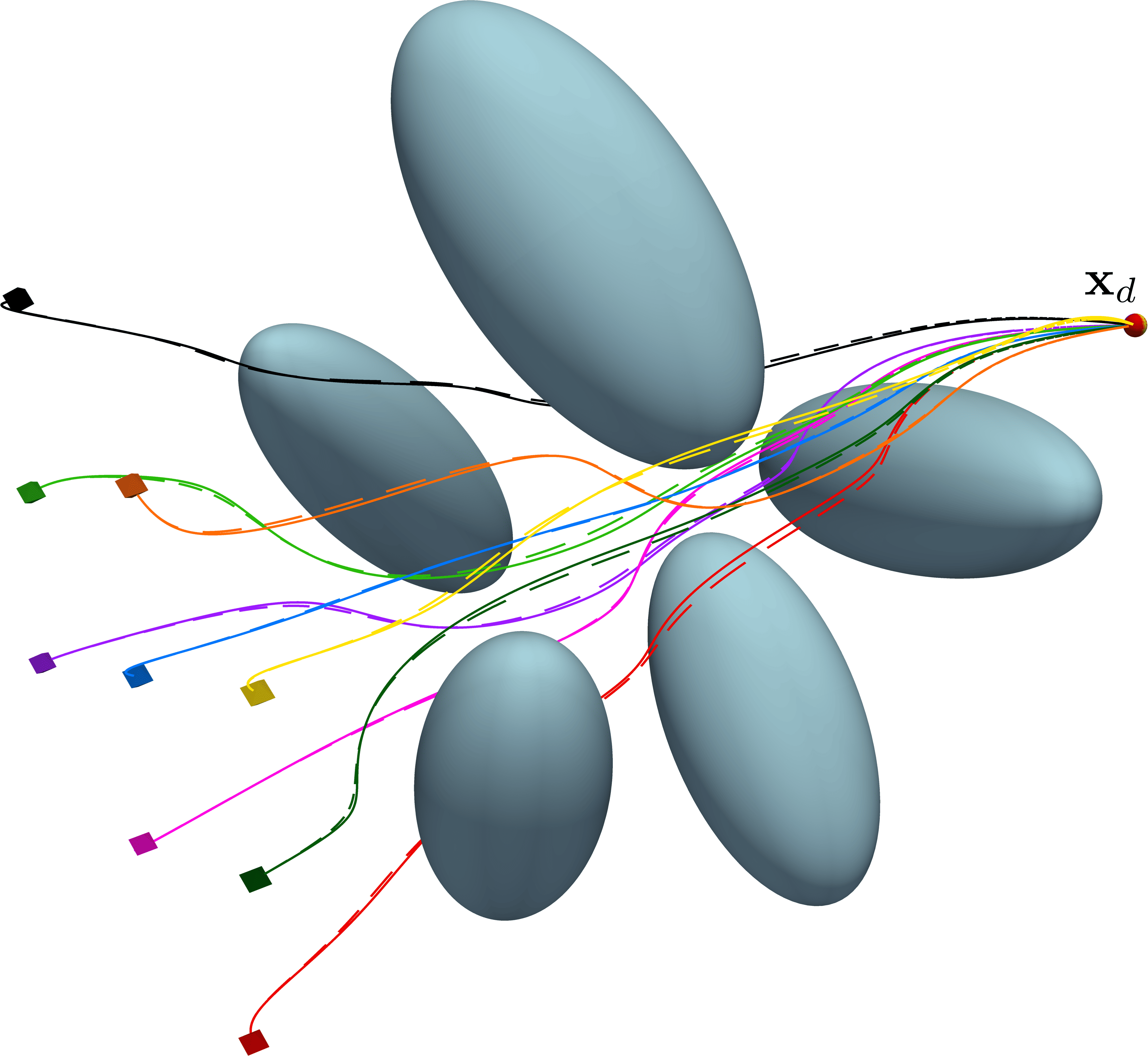}
    \caption{$\mathbf{x}$-trajectories under the DDF (solid curves) and VTF (dashed curves) control laws with $\mathbf{v}_d$ defined in \eqref{example:NF_based_vector_field}.}
    \label{fig:planner1:results-2}
\end{figure}

The first-order motion planner \eqref{second_first_order_planner_example} applies to point-robot navigation \textit{i.e.}, $r = 0$.
Following the argument of \cite[Theorem 1]{kumar2022navigation} for the closed-loop system \eqref{kinematic_closed_loop_system}-\eqref{second_first_order_planner_example}, one can show that if the target position $\mathbf{x}_d$ does not lie on any maximum margin separating hyperplane between any pair of obstacles $\mathcal{O}_i, i\in\mathbb{I}$, as required in \cite[Assumption 2]{kumar2022navigation}, then there exists $K > 0$ such that, for every $\kappa > K$, the free space $\mathcal{X}_r$ defined by \eqref{equation_free_space} is forward invariant and $\mathbf{x}_d$ is almost globally asymptotically stable on $\mathcal{X}_r$.

\begin{figure*}[ht]
    \centering
    \begin{minipage}[t]{0.3425\linewidth}
    \vspace{0pt}
    \begin{subfigure}{\linewidth}
        \centering
        \includegraphics[width=\linewidth]
        {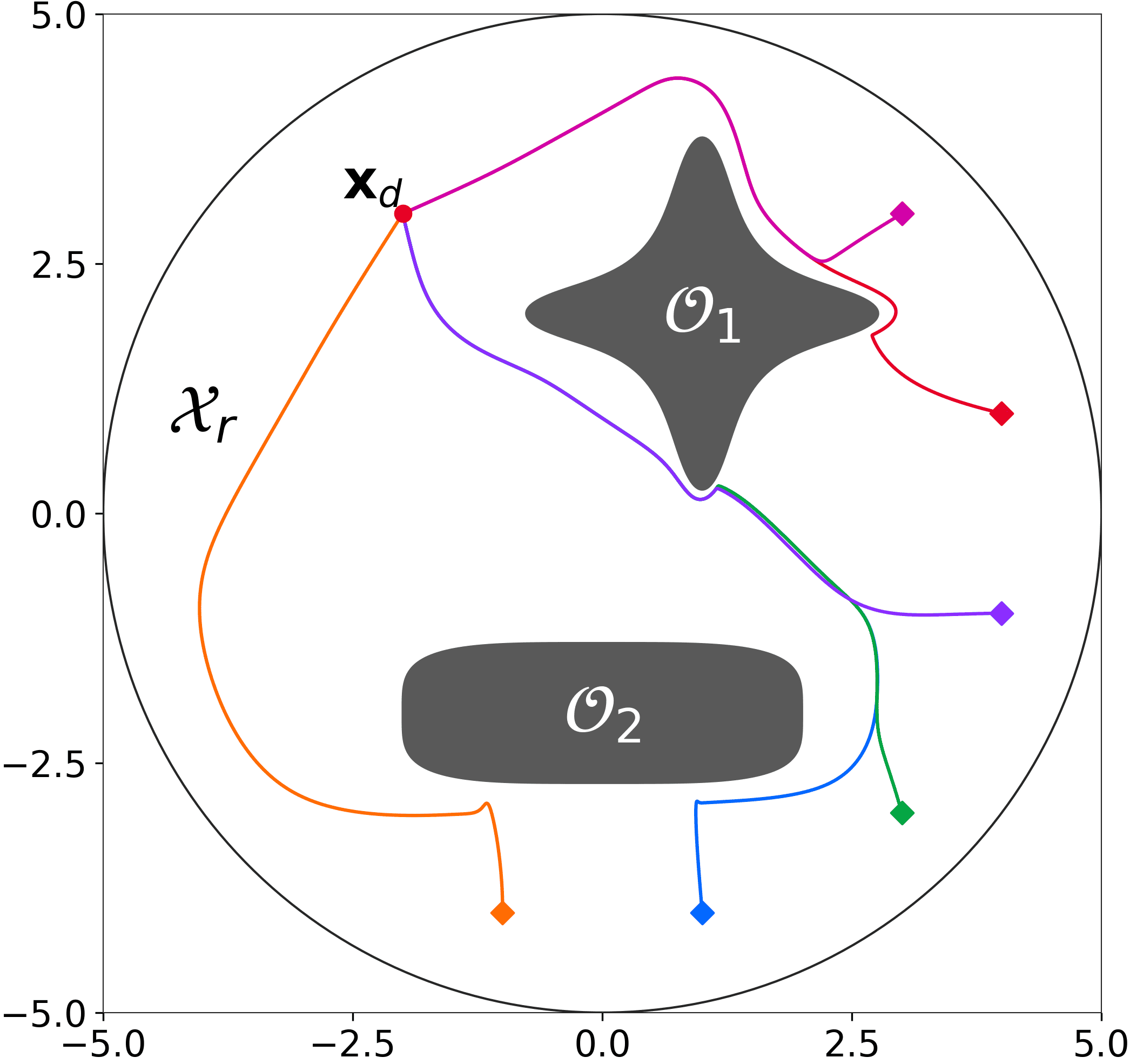}
        \caption{}
        \label{fig:star-world-trajectories}
        \end{subfigure}
    \end{minipage}
    \hfill
    \begin{minipage}[t]{0.3425\linewidth}
    \vspace{0pt}
    \begin{subfigure}{\linewidth}
        \centering
        \includegraphics[width=\linewidth]
        {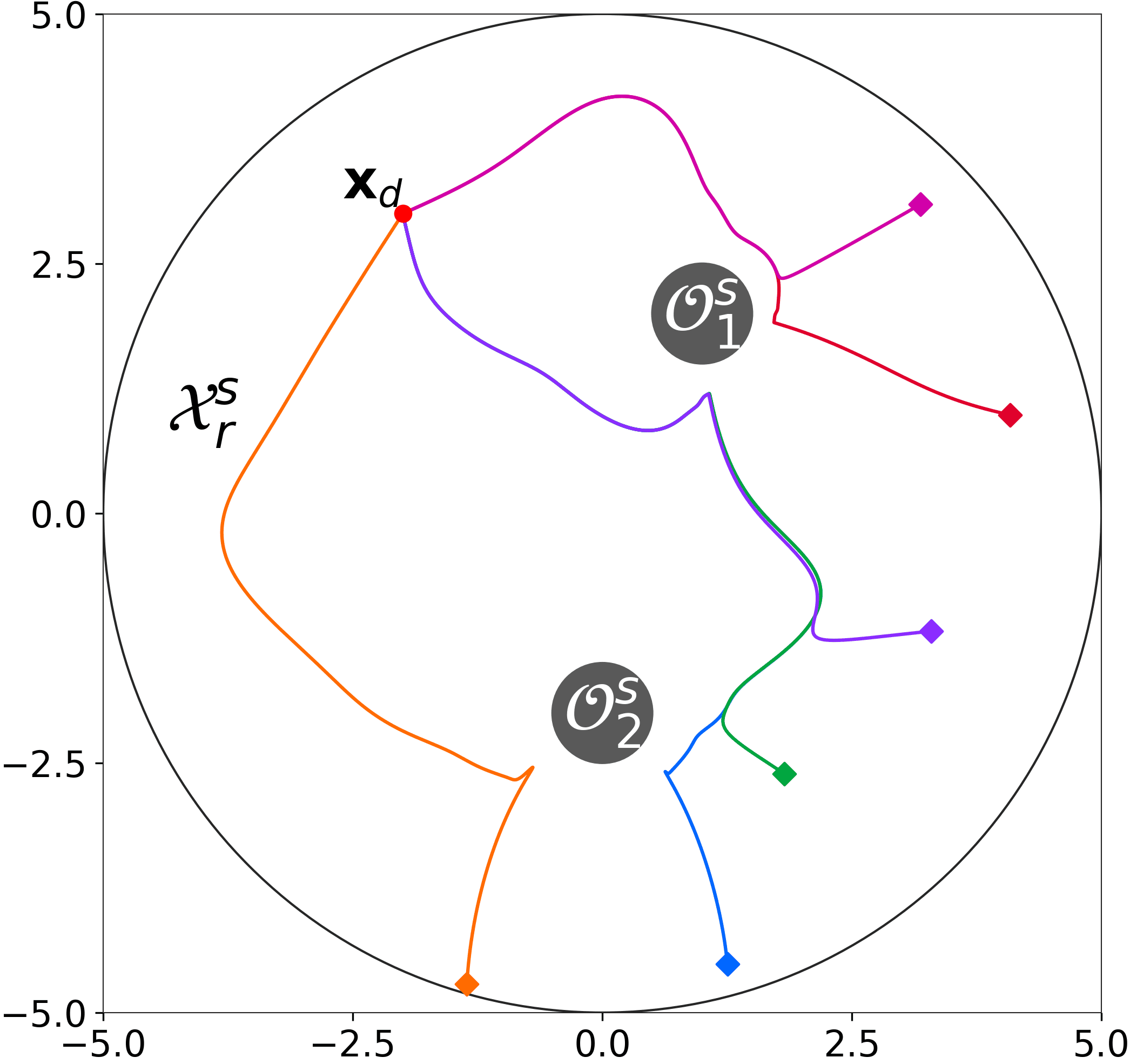}
        \caption{}
        \label{fig:sphere-world-trajectories}
        \end{subfigure}
    \end{minipage}
    \hfill
    \begin{minipage}[t]
    {0.295\linewidth}
    \vspace{0pt}
    \begin{subfigure}{\linewidth}
        \centering
        \includegraphics[width=\linewidth]
        {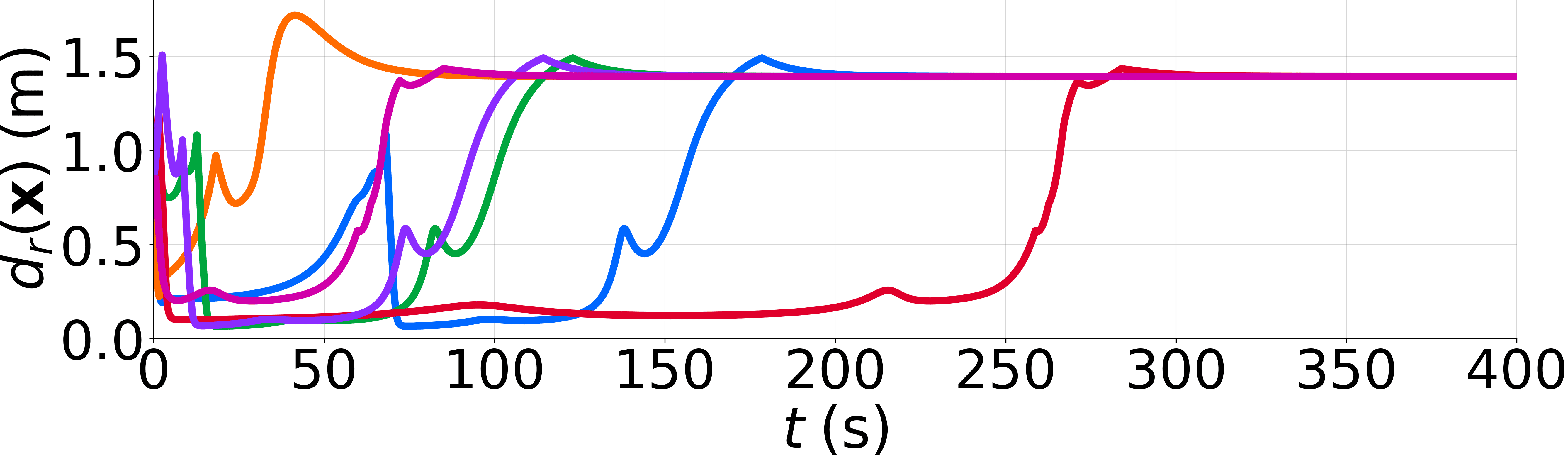}
        \caption{}
        \label{fig:star-distance-trajectories} 
    \end{subfigure}
    \begin{subfigure}{\linewidth}
        \centering
        \includegraphics[width=\linewidth]
        {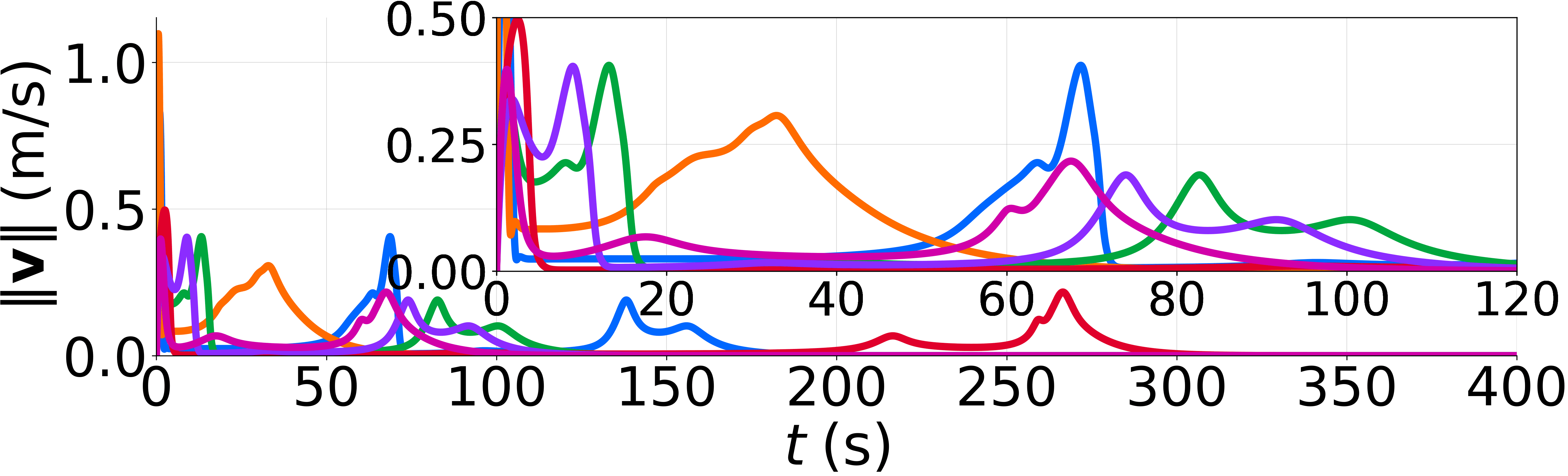}
        \caption{}
        \label{fig:star-velocity-norm-trajectories} 
    \end{subfigure}
    \begin{subfigure}{\linewidth}
        \centering
        \includegraphics[width=\linewidth]
        {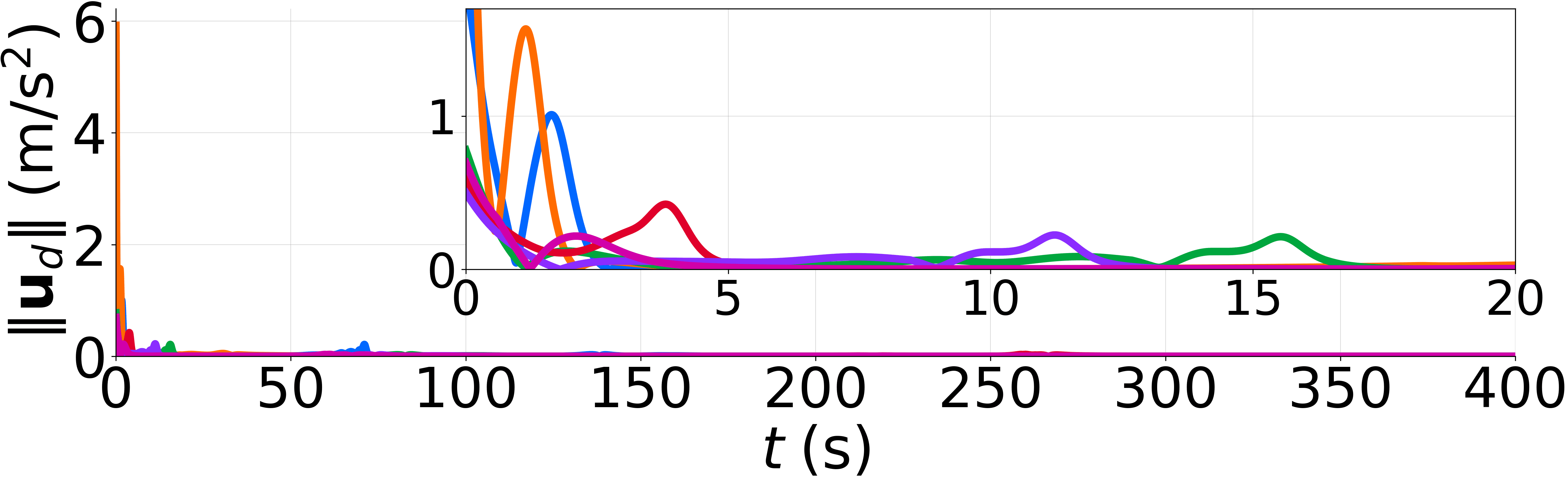}
        \caption{}
        \label{fig:star-acceleration-norm-trajectories} 
    \end{subfigure}
    \end{minipage}
    \caption{Plots obtained using the DDF control law with $\mathbf{v}_d$ defined in \eqref{star_world_motion_planner}. (a) Point robot trajectories in the star world. (b) Point robot trajectories in the transformed sphere world. (c)-(e) Plots of $d_r(\mathbf{x})$, $\|\mathbf{v}\|$ and $\|\mathbf{u}_d\|$ versus time.}
    \label{fig:star-world-results}
\end{figure*}

\begin{remark}\label{remark:modifications}
    In comparison with the first-order motion planner proposed in \cite[Section 3.1]{kumar2022navigation} (hereafter referred to as the \textit{original} planner), the first-order motion planner defined in \eqref{second_first_order_planner_example} (hereafter referred to as the \textit{modified} planner) introduces the following modifications:
    \begin{enumerate}
        \item In the \textit{original} planner, the term $\alpha(\mathbf{x})$ in \eqref{second_first_order_planner_example} is replaced by $e(\mathbf{x}) = \prod_{i\in\mathbb{I}}e_i(\mathbf{x})$, where for each $i\in\mathbb{I}$, $e_i$ is defined in \eqref{ellipsoid_function_definition}.
        As a result, when the number of obstacles increases and the robot is far from the obstacles, $e(\mathbf{x})$ can become very large, which increases the magnitude of the \textit{original} planner.
        This can make the \textit{original} planner difficult to implement with sufficiently large step sizes.
        Therefore, the normalized version of the \textit{original} planner, as implemented in \cite[Section 5]{kumar2022navigation}, is often preferred.
        In contrast, for the \textit{modified} planner, \eqref{alpha_product_definition} and \eqref{individual_alpha_definition} imply that $\alpha(\mathbf{x})$ is upper bounded by $1$ for all $\mathbf{x}\in\mathcal{X}_r$.
        
        \item In the \textit{original} planner, for each $i\in\mathbb{I}$, the term $\bar{\alpha}_i(\mathbf{x})$ in \eqref{second_first_order_planner_example} is replaced by $\bar{e}_i(\mathbf{x}) = \prod_{j\in\mathbb{I}\setminus\{i\}}e_j(\mathbf{x})$.
        Therefore, even when the robot is far from the obstacles, its trajectory is influenced by the repulsive vector field components of the \textit{original} planner.
        Furthermore, as the number of obstacles increases, larger values of the tuning parameter $\kappa$ must be chosen to mitigate the effects of these combined repulsive vector field components.
        In contrast, for the \textit{modified} planner, \eqref{individual_alpha_definition} implies that $\alpha_i(\mathbf{x})\in[0, 1]$ for all $\mathbf{x}\in\mathcal{X}_r$ and for every $i\in\mathbb{I}$. Therefore, the use of $\bar{\alpha}_i(\mathbf{x})$ in \eqref{second_first_order_planner_example} ensures that when $\mathbf{x}$ is more than $\nu$ units away from obstacle $\mathcal{O}_i$, the repulsive vector field component associated with $\mathcal{O}_i$ vanishes.
        Consequently, when the robot is more than $\nu$ units away from all obstacles, the repulsive vector field component of the \textit{modified} planner vanishes.     
    \end{enumerate}
\end{remark}

\section{Simulation results}
\label{section:simulations}

\begin{figure*}[ht]
    \centering
    \begin{minipage}[t]{0.48\textwidth}
        \vspace{0pt}
        \begin{subfigure}{\linewidth}
            \centering
            \includegraphics[width=\linewidth]{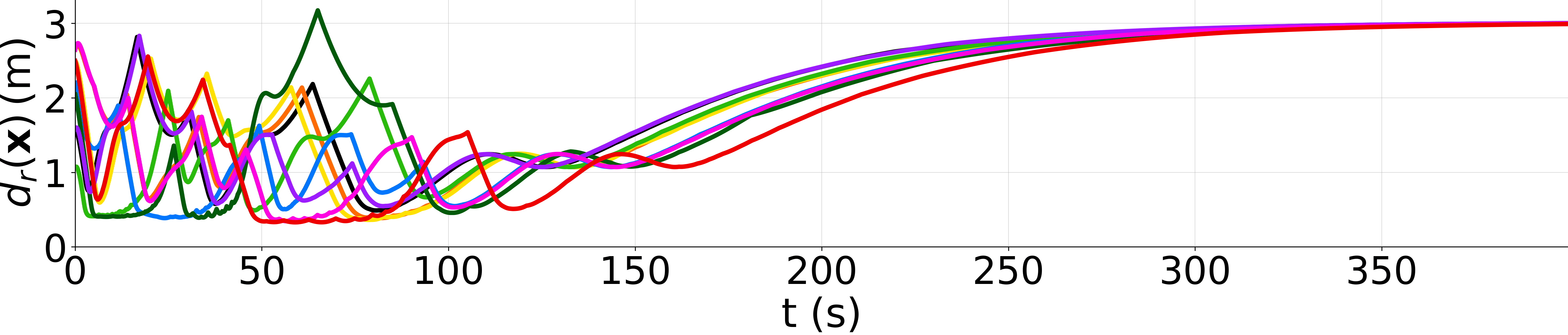}
            \caption{}
            \label{fig:planner2:VTF:distance-to-obstacles}
        \end{subfigure}

        \begin{subfigure}{\linewidth}
            \centering
            \includegraphics[width=\linewidth]{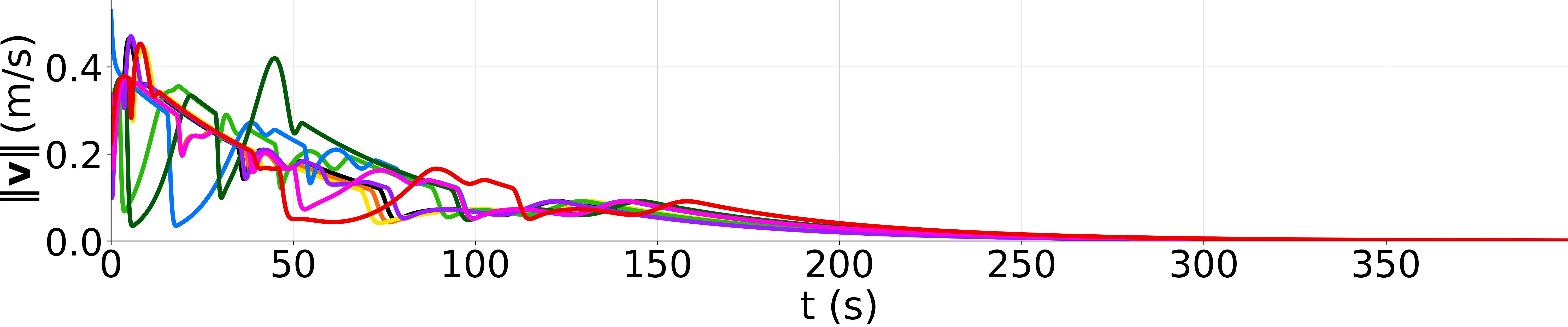}
            \caption{}
            \label{fig:planner2:VTF:velocity-norm}
        \end{subfigure}
        
    \end{minipage}\hfill
    \begin{minipage}[t]{0.48\textwidth}
        \vspace{0pt}
        \begin{subfigure}{\linewidth}
            \centering
            \includegraphics[width=\linewidth]{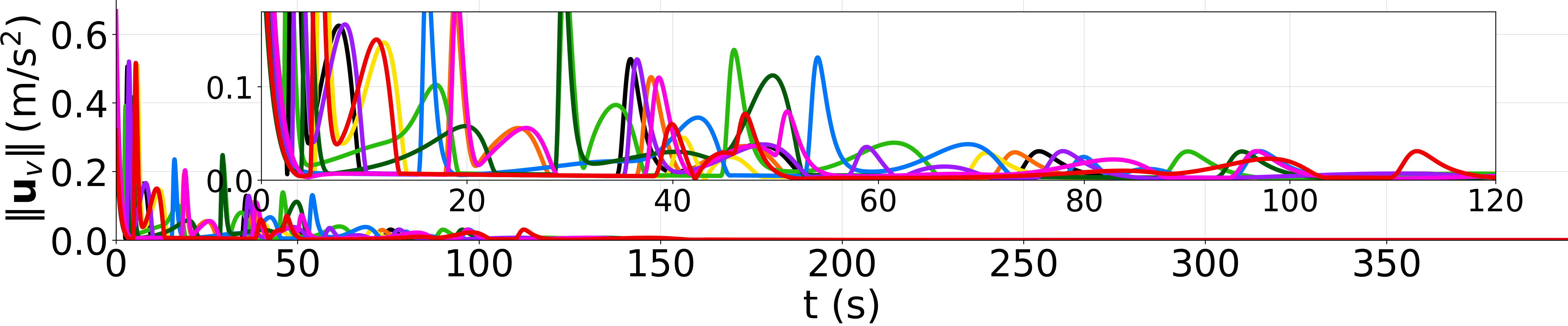}
            \caption{}
            \label{fig:planner2:VTF:acceleration-norm}
        \end{subfigure}

        \begin{subfigure}{\linewidth}
            \centering
            \includegraphics[width=\linewidth]{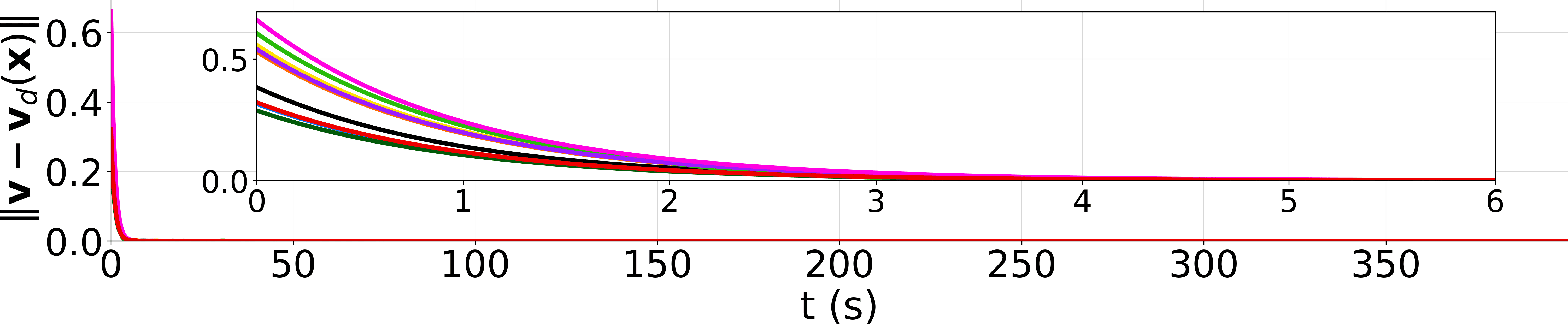}
            \caption{}
            \label{fig:planner2:VTF:velocity-error-norm}
        \end{subfigure}
    \end{minipage}

    \caption{Plots obtained using the VTF control law with $\mathbf{v}_d$ defined in \eqref{second_first_order_planner_example}.
    (a)-(d) Plots of $d_r(\mathbf{x})$, $\|\mathbf{v}\|$, $\|\mathbf{u}_v\|$ and $\|\mathbf{v} - \mathbf{v}_d(\mathbf{x})\|$ versus time.}
    \label{fig:planner2:results-1}
\end{figure*}

\begin{figure}[ht]
    \centering
    \includegraphics[width=0.9\linewidth]{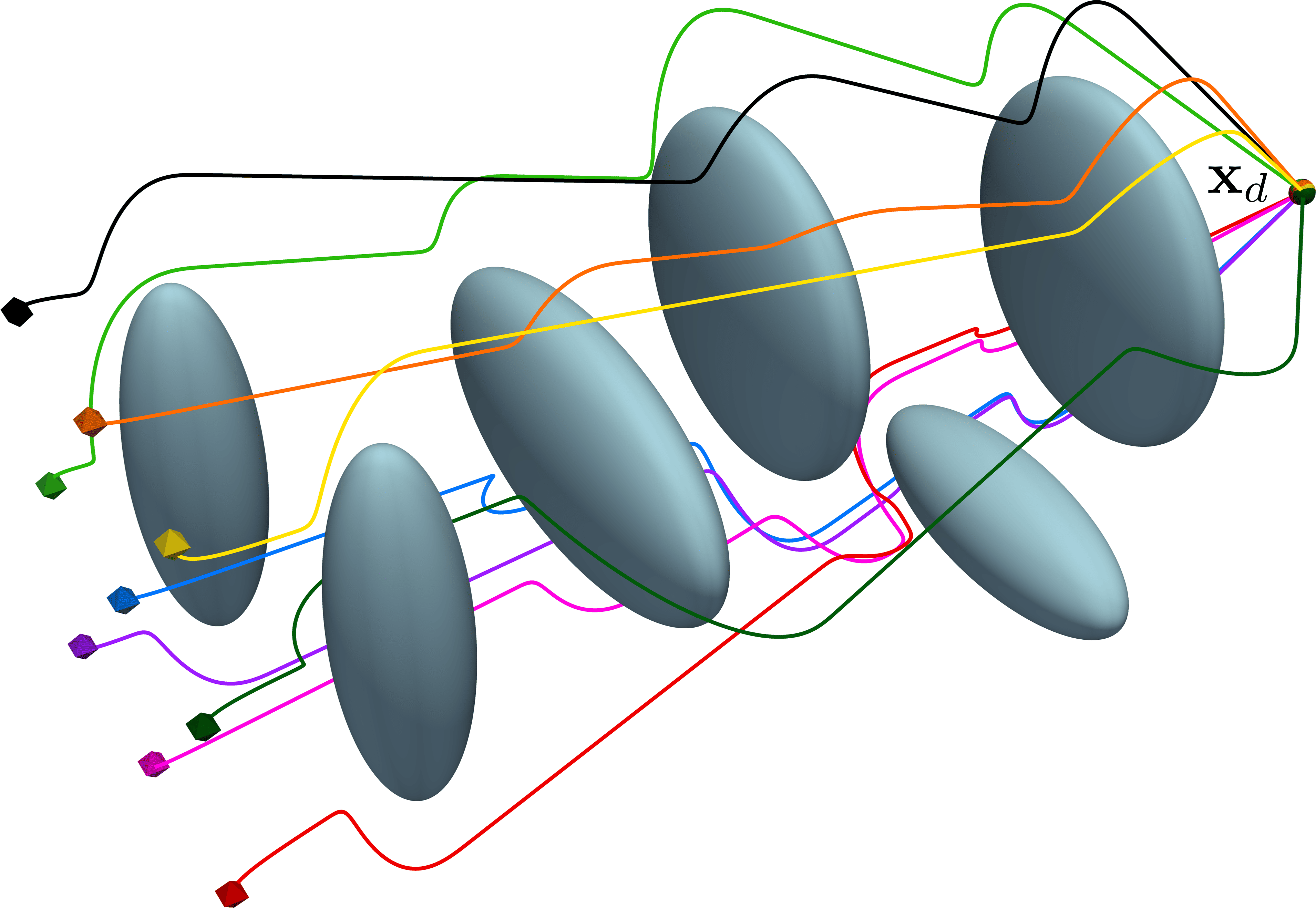}
    \caption{$\mathbf{x}$-trajectories under the VTF control law with $\mathbf{v}_d$ defined in \eqref{second_first_order_planner_example}.}
    \label{fig:planner2:results-2}
\end{figure}

For the first simulation, we consider the DDF control law in \eqref{controller:DDF} and the VTF control law in \eqref{controller:VTF} control law, together with the navigation function-based first-order motion planner $\mathbf{v}_d$ defined in \eqref{example:NF_based_vector_field}.
The workspace $\mathcal{W}$ is a three-dimensional sphere of radius $10\,\mathrm{m}$ centered at $[-5, 0, 0]^\top\,\mathrm{m}$ and contains five ellipsoidal obstacles.
For visualization purposes, the workspace boundary is omitted in Fig. \ref{fig:planner1:results-2} and only the ellipsoidal obstacles are shown.
For each $i\in\mathbb{I}$, the scalar function $h_i$ associated with obstacle $\mathcal{O}_i$ and used in \eqref{individual_obstacle_definition_using_hi}
is defined similarly to $e_i$ in \eqref{ellipsoid_function_definition}.

The point robot is initialized at nine different locations 
$[-9, 2, 2]^\top$,
$[-9, 0, 2]^\top$,
$[-9, -2, 2]^\top$,
$[-9, 2, 0]^\top$,
$[-9, 0, 0]^\top$,
$[-9, -2, 0]^\top$,
$[-9, 2, -2]^\top$,
$[-9, 0, -2]^\top$, and
$[-9, -2, -2]^\top$,
that are marked by diamond symbols in Fig. \ref{fig:planner1:results-2}.
For each initial location, the components of the initial velocity are drawn independently from a Gaussian distribution with mean $0\,\mathrm{m/s}$ and standard deviation $0.2\,\mathrm{m/s}$.
The target location is $\mathbf{x}_d = [0, 0, 0]^\top\,\mathrm{m}$.
The parameters $\epsilon_1$ and $\epsilon_2$ in \eqref{beta_function_definition} are set to $0.5\,\mathrm{m}$ and $1.5\,\mathrm{m}$, respectively.
The parameter $\delta_1$ in \eqref{attractive_potential_function} is set to $0.01$, while $\kappa$ in \eqref{example:NF_based_vector_field} is set to $20$.
The control gains in \eqref{controller:DDF} are chosen as $k_1 = 5$ and $k_d = 1$.

The $\mathbf{x}$-trajectories in Fig. \ref{fig:planner1:results-2}, and the plots of $d_r(\mathbf{x})$ vs time in Figs. \ref{fig:planner1:DDF:distance-to-obstacles} and \ref{fig:planner1:VTF:distance-to-obstacles} show that, under both the DDF and VTF control laws with $\mathbf{v}_d$ defined in \eqref{example:NF_based_vector_field}, the robot remains in the free space $\mathcal{X}_r$ and converges to the target location $\mathbf{x}_d$.

Under the DDF control law \eqref{controller:DDF}, the damping coefficient $k_d\beta(d_r(\mathbf{x}))$ increases as the robot approaches an obstacle boundary.
Consequently, the damping term $-k_d\beta(d_r(\mathbf{x}))\mathbf{v}$ increasingly opposes the robot's motion, thereby reducing its speed near the obstacle.
This allows the control input $\mathbf{u}_d$ to eventually steer $\mathbf{v}$ in the direction of the desired velocity generated by the first-order motion planner $\mathbf{v}_d$.
In contrast, under the VTF control law \eqref{controller:VTF}, the increase in $\beta(d_r(\mathbf{x}))$ increases the rate at which the robot's velocity tracks the desired velocity $\mathbf{v}_d(\mathbf{x})$.
Consequently, for the initial conditions considered in this simulation, the VTF-controlled trajectories converge toward the target faster than the corresponding DDF-controlled trajectories, as shown in Fig. \ref{fig:planner1:results-1}.

In the second simulation, we apply the DDF control law \eqref{controller:DDF}, with $\mathbf{v}_d$ defined in \eqref{star_world_motion_planner}, for a point robot navigation in a star world, introduced earlier in Section \ref{section:star-world-theory}.
The circular workspace $\mathcal{W}$ contains two star-shaped obstacles, as illustrated in Fig. \ref{fig:star-world-trajectories}.
In the transformed sphere world shown in Fig. \ref{fig:sphere-world-trajectories}, each obstacle $\mathcal{O}_i^s, i\in\{1, 2\}$, has the center as its corresponding star-shaped obstacle $\mathcal{O}_i$ and a radius of $0.5\;\mathrm{m}$.
The parameters $\epsilon_1$ and $\epsilon_2$ in \eqref{beta_function_definition} are set to $0.5\;\mathrm{m}$ and $1.5\;\mathrm{m}$, respectively.
The parameters $\delta_1$ in \eqref{attractive_potential_function} is set to $0.01$, while $\kappa$ in \eqref{example:NF_based_vector_field} is set to $10$.
The gains are selected as $k_1 = 10$ and $k_d = 1$.
The $\mathbf{x}$-trajectories in Fig. \ref{fig:star-world-trajectories}, together with the evolution of $d_r(\mathbf{x})$ in Fig. \ref{fig:star-distance-trajectories}, demonstrate that the robot
reaches the target $\mathbf{x}_d$ while maintaining collision-free
motion for all simulated initial conditions.

For the next simulation, we consider the VTF control law in \eqref{controller:VTF} together with the first-order motion planner $\mathbf{v}_d$ defined in \eqref{second_first_order_planner_example}.
The workspace $\mathcal{W}$ is a three-dimensional sphere of radius $15\,\mathrm{m}$ centered at $[-10, 0, 0]^\top\,\mathrm{m}$, and contains six ellipsoidal obstacles.
For visualization purposes, the workspace boundary is omitted in Fig. \ref{fig:planner2:results-2} and only ellipsoidal obstacles are shown.

The point robot is initialized at nine different locations
$[-20, 3, 3]^\top$,
$[-20, 0, 3]^\top$,
$[-20, -3, 3]^\top$,
$[-20, 3, 0]^\top$,
$[-20, 0, 0]^\top$,
$[-20, -3, 0]^\top$,
$[-20, 3, -3]^\top$,
$[-20, 0, -3]^\top$, and
$[-20, -3, -3]^\top$,
that are marked by diamond symbols in \ref{fig:planner2:results-2}.
For each initial location, the components of the initial velocity are drawn independently from a Gaussian distribution with mean $0\,\mathrm{m/s}$ and standard deviation $0.2\,\mathrm{m/s}$.
The target location is $\mathbf{x}_d = [0, 0, 0]^\top\,\mathrm{m}$.
The parameters in \eqref{beta_function_definition} are set to $\epsilon_1 =0.5\,\mathrm{m}$ and $\epsilon_2=1.5\,\mathrm{m}$.
The value of $\nu$ in \eqref{individual_alpha_definition} is chosen as $1\,\mathrm{m}$.
The parameter $\delta_1$ in \eqref{attractive_potential_function} is set to $1$, while $\kappa$ in \eqref{second_first_order_planner_example} is set to $10$. 
The gain $k_d$ in \eqref{controller:VTF} is set to $1$, and the gain $k_1$ in \eqref{second_first_order_planner_example} is set to $0.02$.

The $\mathbf{x}$-trajectories in Fig. \ref{fig:planner2:results-2} and the evolution of $d_r(\mathbf{x})$ in Fig. \ref{fig:planner2:VTF:distance-to-obstacles} show that, under the VTF control law with $\mathbf{v}_d$ defined in \eqref{second_first_order_planner_example}, the robot remains in the free space $\mathcal{X}_0$ and converges to the target location $\mathbf{x}_d$.
Furthermore, Fig. \ref{fig:planner2:VTF:velocity-error-norm} illustrates the monotonic decrease of the norm $\|\mathbf{v} - \mathbf{v}_d(\mathbf{x})\|$ over time, consistent with Claim \ref{theorem:VF_decrease} of Theorem \ref{theorem:VF}.

\section{Experimental validation}\label{section:experimental_validation}
We implemented the DDF control law \eqref{controller:DDF} and the VTF control law \eqref{controller:VTF} on a Quanser QDrone 2 in an \textit{a priori} known environment.
The objective was to stabilize the quadrotor at a desired target position $\mathbf{x}_d$ while avoiding known cylindrical obstacles. 

The rigid-body dynamics of the quadrotor are given by
\begin{equation}
\begin{aligned}
    \dot{\mathbf{x}} &= \mathbf{v},\\
    \dot{\mathbf{v}} &= -g\mathbf{e}_3 + \frac{T}{m}\mathbf{R}\mathbf{e}_3,\\
    \dot{\mathbf{R}} &= \mathbf{R}\boldsymbol{\omega}^\times, \\
    \mathbf{J}\dot{\boldsymbol{\omega}} &= -\boldsymbol{\omega}\times \mathbf{J}\boldsymbol{\omega} + \boldsymbol{\tau},
    \end{aligned}
\end{equation}
where $\mathbf{x}\in\mathbb{R}^3$ and $\mathbf{v}\in\mathbb{R}^3$ denote the position of the quadrotor's center of mass and its linear velocity, respectively, both expressed in the inertial frame.
The rotation matrix $\mathbf{R}\in\mathrm{SO}(3)$ maps vectors from the body-fixed frame to the inertial frame, and $\boldsymbol{\omega}\in\mathbb{R}^3$ denotes the angular velocity expressed in the body-fixed frame.
The skew symmetric matrix $\boldsymbol{\omega}^\times\in\mathbb{R}^{3\times 3}$ is defined such that $\boldsymbol{\omega}^\times\mathbf{y} = \boldsymbol{\omega} \times \mathbf{y}$ for every $\mathbf{y}\in\mathbb{R}^3$. 
Furthermore, $\mathbf{J}\in\mathbb{R}^{3\times 3}$ is the constant symmetric positive definite inertia matrix, $m$ is the mass of the quadrotor, $g = 9.81\,\mathrm{m/s^2}$, and $\mathbf{e}_3=[0, 0, 1]^\top$.
The control inputs are the thrust $T\in\mathbb{R}_{\geq 0}$ and the torque $\boldsymbol{\tau}\in\mathbb{R}^3$ expressed in the body-fixed frame.

\begin{figure*}
    \centering
    \includegraphics[width = 1\linewidth]{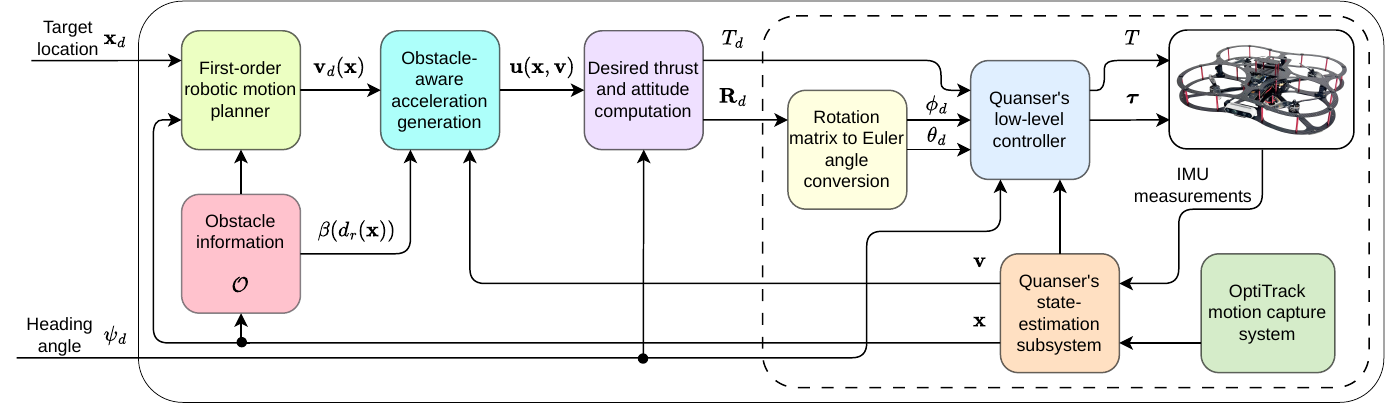}
    \caption{Implementation of the proposed control strategies on the Quanser QDrone 2.}
    \label{fig:quanser_block_diagram}
\end{figure*}

The control architecture consists of an outer loop and an inner loop, as shown in Fig. \ref{fig:quanser_block_diagram}.
In the outer loop, the proposed DDF and VTF controllers generate the desired acceleration $\mathbf{u}$ for the second-order system \eqref{equation:second_order_system}.
Motivated by \cite[Section III]{lee2010geometric}, this acceleration is then converted into a desired thrust $T_d\in\mathbb{R}_{\geq 0}$ and a desired attitude $\mathbf{R}_d\in\mathrm{SO}(3)$. 
In particular,
\begin{equation}
    T_d = m\|\mathbf{u} + g\mathbf{e}_3\|,
\end{equation}
and $\mathbf{R}_d = [\mathbf{b}_{1d}, \mathbf{b}_{2d}, \mathbf{b}_{3d}]$, where 
\begin{equation}
    \mathbf{b}_{3d} = \frac{\mathbf{u} + g\mathbf{e}_3}{\|\mathbf{u} + g\mathbf{e}_3\|}, \mathbf{b}_{2d} = \frac{\mathbf{b}_{3d}\times\mathbf{b}_{\psi_d}}{\|\mathbf{b}_{3d}\times\mathbf{b}_{\psi_d}\|}, \mathbf{b}_{1d} = \mathbf{b}_{2d}\times\mathbf{b}_{3d},
\end{equation}
with $\mathbf{b}_{\psi_d} = [\cos(\psi_d), \sin(\psi_d), 0]^\top$, where $\psi_d\in(-\pi, \pi]$ is a prescribed constant heading angle measured from $\mathbf{e}_1$ to $\mathbf{e}_2$ about $\mathbf{e}_3$.

The desired thrust $T_d$ and attitude $\mathbf{R}_d$ are then passed to the low-level inner-loop controller.
Since Quanser's controller accepts attitude references in terms of Euler angles, $\mathbf{R}_d$ is converted into the corresponding desired roll and pitch angles, $\phi_d\text{ and } \theta_d$.
These angles, together with $T_d$ and the constant heading angle $\psi_d$, are supplied to Quanser's existing PIV control architecture, which generates the individual motor commands required to track the desired attitude and realize the commanded thrust.
The constant heading angle $\psi_d$ is tracked using Quanser's existing yaw control scheme.

\begin{figure}
    \centering
    \includegraphics[width=0.8\linewidth]{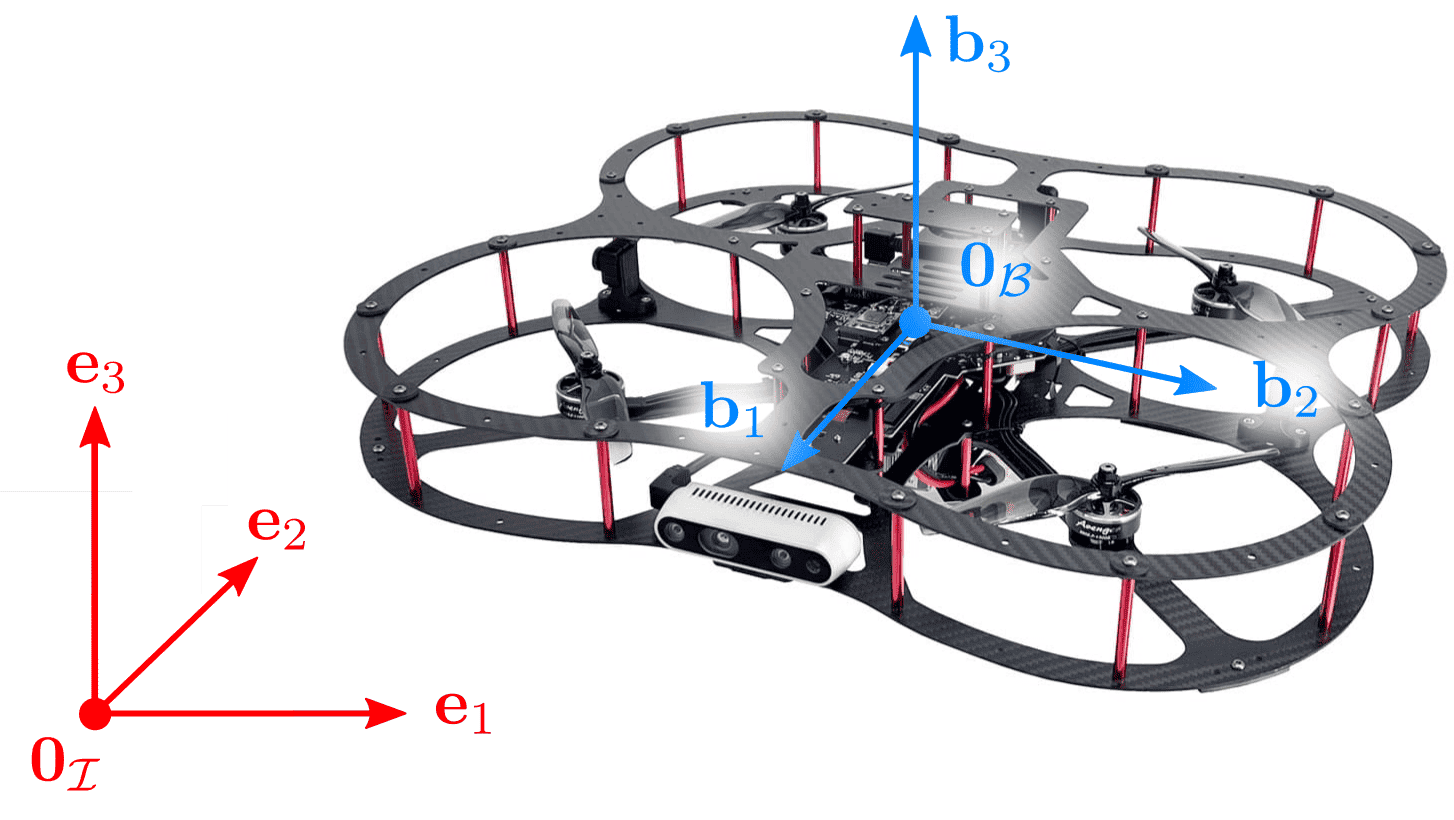}
    \caption{Quanser QDrone 2 used for experimental validation.}
    \label{fig:QDrone2}
\end{figure}

\subsection{Experimental setup}\label{section:experimental_setup}
We used the Quanser QDrone2 quadrotor, shown in Fig. \ref{fig:QDrone2}, for the experimental validation.
The position and attitude of the quadrotor were measured using the OptiTrack motion capture system. The pose measurements were combined with measurements from the onboard inertial measurement unit (IMU) within Quanser's state-estimation subsystem.
The resulting estimates of the position $\mathbf{x}$ and the velocity $\mathbf{v}$ provided the state feedback required by the proposed outer-loop controller.

The workspace was a rectangular cuboid of $6.6\,\mathrm{m}$ length, $3.6\,\mathrm{m}$ width and $2.0\,\mathrm{m}$ height, with three vertical cylindrical obstacles whose axes were parallel to $\mathbf{e}_3$.
The horizontal coordinates of their central axes were $[0.6, 0.9]^\top\,\mathrm{m}$, $[0.3, -0.8]^\top\,\mathrm{m}$ and $[-1.5, 0.2]^\top\,\mathrm{m}$.
Each obstacle had a radius of $0.2\,\mathrm{m}$.
To account for the quadrotor's size, each cylindrical obstacle was inflated by $r = 0.3\,\mathrm{m}$, resulting in an effective obstacle radius of $0.5\,\mathrm{m}$, as shown in the top view of the workspace in Fig. \ref{fig:top_view_experiment}.

We considered a planar navigation scenario at a constant altitude. 
Initially, the quadrotor was commanded to take off to a height of $0.5\,\mathrm{m}$, which was maintained throughout the experiment.
Its initial horizontal position was set to $[1.5, 0.9]^\top\,\mathrm{m}$ and the target horizontal position was set to $[-3.0, -0.6]^\top\,\mathrm{m}$.

For the implementation of the proposed control strategies, we used the navigation function-based first-order motion planner defined in \eqref{example:NF_based_vector_field}.
It was established in \cite{koditschek1990robot} that, for a sufficiently large value of the tuning parameter $\kappa$, this planner renders the target location $\mathbf{x}_d$ almost globally asymptotically stable for the single integrator system in a sphere world. 
However, for large values of $\kappa$, the term $\frac{k_1}{(f(\mathbf{x})^{\kappa} + h(\mathbf{x}))^{\frac{1}{\kappa} + 1}}$ can become very small at locations far from the target.
Therefore, for the experiment, we used the following normalized version of the first-order motion planner defined in \eqref{example:NF_based_vector_field}:
\begin{equation}\label{normalized_first_order_motion_planner}
    \mathbf{v}_d(\mathbf{x}) = \frac{-k_1\nabla_{\mathbf{x}}\varphi(\mathbf{x})}{\sqrt{\|\nabla_{\mathbf{x}}\varphi(\mathbf{x})\|^2 + \delta_2^2}},
\end{equation}
where $\delta_2 > 0$ and $\varphi$ is defined in \eqref{paternain_navigation_function_definition}.

The parameters $\epsilon_1$ and $\epsilon_2$, used in \eqref{beta_function_definition}, were set to $0.1\,\mathrm{m}$ and $0.5\,\mathrm{m}$, respectively.
The value of $\delta_1$ in \eqref{attractive_potential_function} was chosen as $0.0125$ and the value of $\kappa$ in \eqref{example:NF_based_vector_field} was set to $11$.
The parameters in \eqref{normalized_first_order_motion_planner} were chosen as $k_1 = 0.6$ and $\delta_2 = 0.005$.

Since the control efforts required to maintain the altitude are considerably higher than those required for steering it in the horizontal $x$-$y$ plane, we introduced a $3\times3$ positive definite diagonal matrix $\mathbf{K}_{\mathbf{x}}$ and replaced $k_d$ in \eqref{controller:DDF} and \eqref{controller:VTF} with another $3\times3$ positive definite diagonal matrix $\mathbf{K}_{\mathbf{v}}$.
In particular, under the DDF control law, the desired translational acceleration was obtained as follows:
\begin{equation}
    \mathbf{u}_{d}(\mathbf{x}, \mathbf{v}) = \mathbf{K}_{\mathbf{x}}\mathbf{v}_d(\mathbf{x}) - \mathbf{B}(\mathbf{x})\mathbf{K}_{\mathbf{v}}\mathbf{v},
\end{equation}
where $\mathbf{K}_{\mathbf{x}} = \mathrm{diag}(1.4, 1.4, 23)$, $\mathbf{K}_{\mathbf{v}} = \mathrm{diag}(3.14, 3.14,28)$, and $\mathbf{B}(\mathbf{x}) = \mathrm{diag}(\beta(d_r(\mathbf{x})), \beta(d_r(\mathbf{x})), 1)$.
For VTF controller implementation, the desired translational acceleration was computed as
\begin{equation}
    \mathbf{u}_{v}(\mathbf{x}, \mathbf{v}) = -\mathbf{B}(\mathbf{x})\mathbf{K}_{\mathbf{v}}(\mathbf{v} - \mathbf{v}_d(\mathbf{x})) + \mathbf{M}_d(\mathbf{x})\mathbf{v}.
\end{equation}

\subsection{Experimental results}

Under both the DDF and VTF control laws, the quadrotor reached the target region, shown in red in Fig. \ref{fig:top_view_experiment}, while avoiding the obstacles.
The video of the experiment is available online\footnote{[Online]. Available: \url{https://youtu.be/BhwDtWQv4gk}.}.
Figures \ref{fig:experiment_velocity_norm} and \ref{fig:experiment_acceleration_norm} show the norms of the horizontal components of the quadrotor velocity $\|\mathbf{P}(\mathbf{e}_3)\mathbf{v}\|$ and the control inputs $\|\mathbf{P}(\mathbf{e}_3)\mathbf{u}\|$ over time, respectively, where $\mathbf{P}(\mathbf{e}_3) = \mathbf{I}_3-\mathbf{e}_3\mathbf{e}_3^\top$.
The magenta curves were obtained during the DDF controller implementation, whereas the blue curves were obtained during the VTF controller implementation.
The dashed red vertical lines in
Figs. \ref{fig:experiment_velocity_norm} and \ref{fig:experiment_acceleration_norm} indicate the time at which Quanser's landing controller was activated during the experiment with the VTF controller.

As observed in the simulations, the quadrotor reached the target region sooner under the VTF control law than under the DDF control law.
Under the DDF controller, an increase in $\beta(d_r(\mathbf{x}))$ increases the magnitude of the damping term, thereby reducing the quadrotor's velocity near obstacles.
Whereas, under the VTF control law, an increase in $\beta(d_r(\mathbf{x}))$ increases the rate at which the quadrotor's velocity approaches the desired velocity prescribed by the first-order motion planner $\mathbf{v}_d$.

\begin{figure}
    \centering
    \includegraphics[width=1\linewidth]{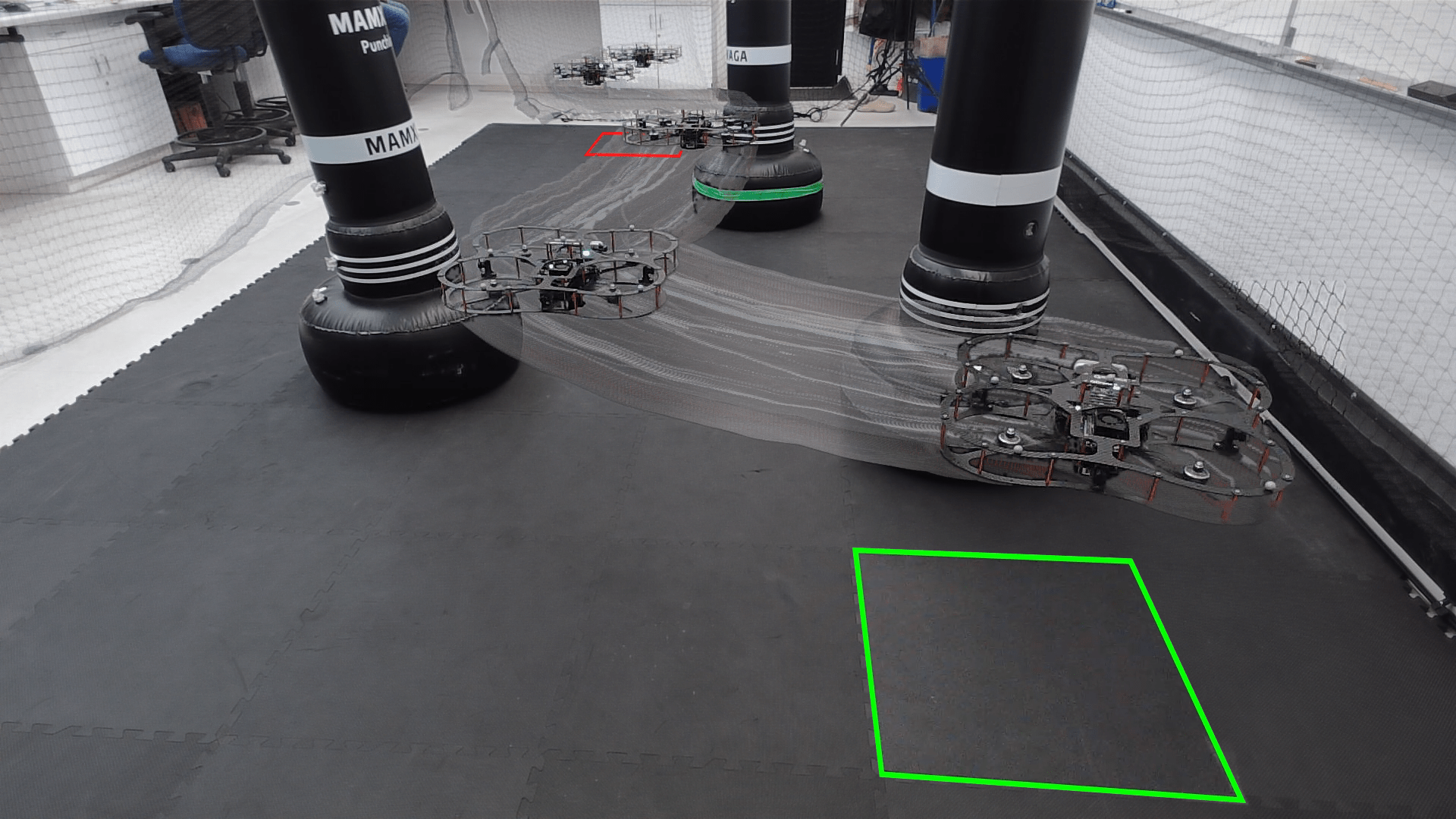}
    \caption{QDrone2 trajectory under the DDF control law.}
    \label{fig:DDF_flight}
\end{figure}

\begin{figure}
    \centering
    \includegraphics[width=1\linewidth]{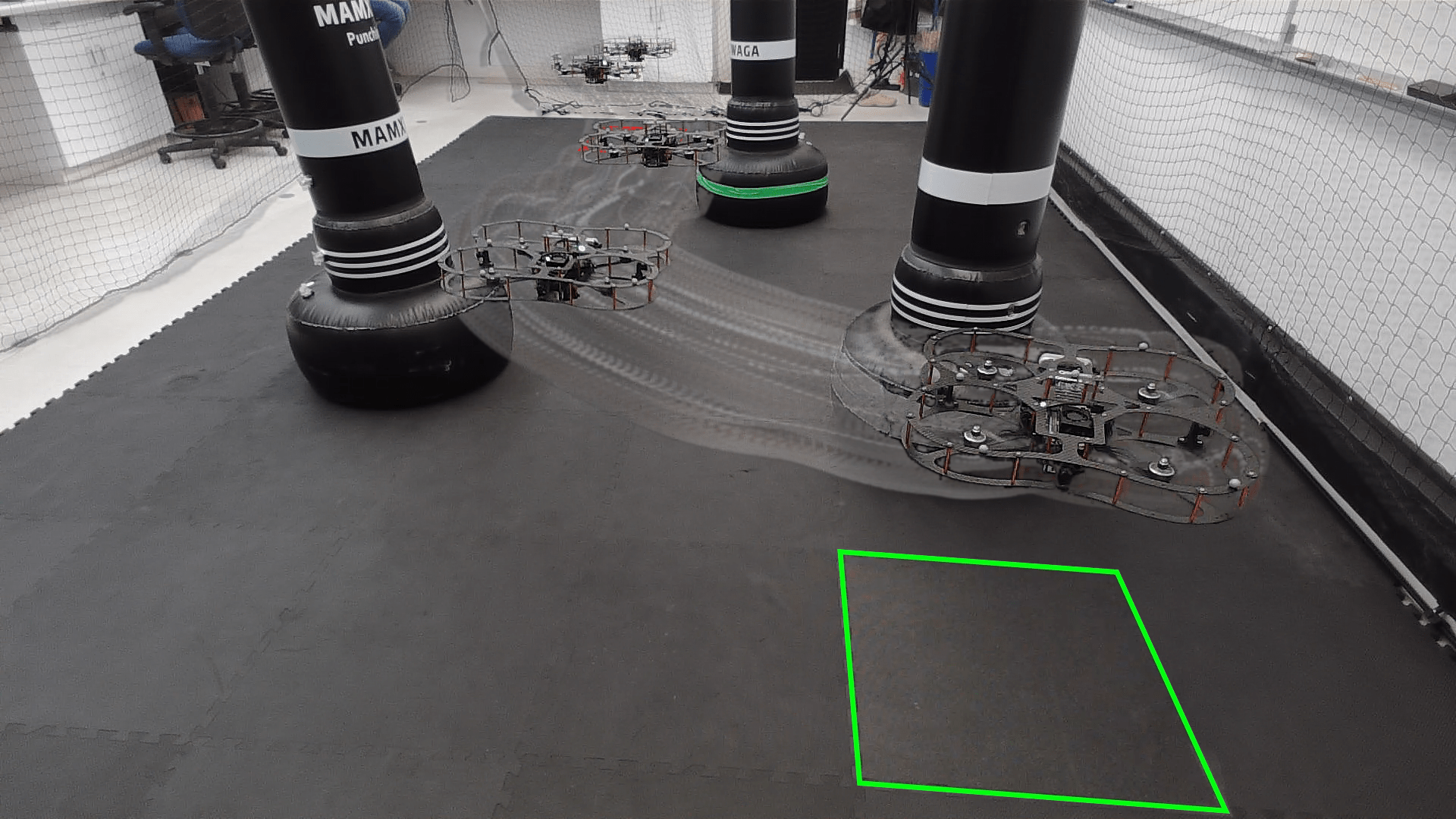}
    \caption{QDrone2 trajectory under the VTF control law.}
    \label{fig:VTF_flight}
\end{figure}

\begin{figure}
    \centering
    \includegraphics[width=1\linewidth]{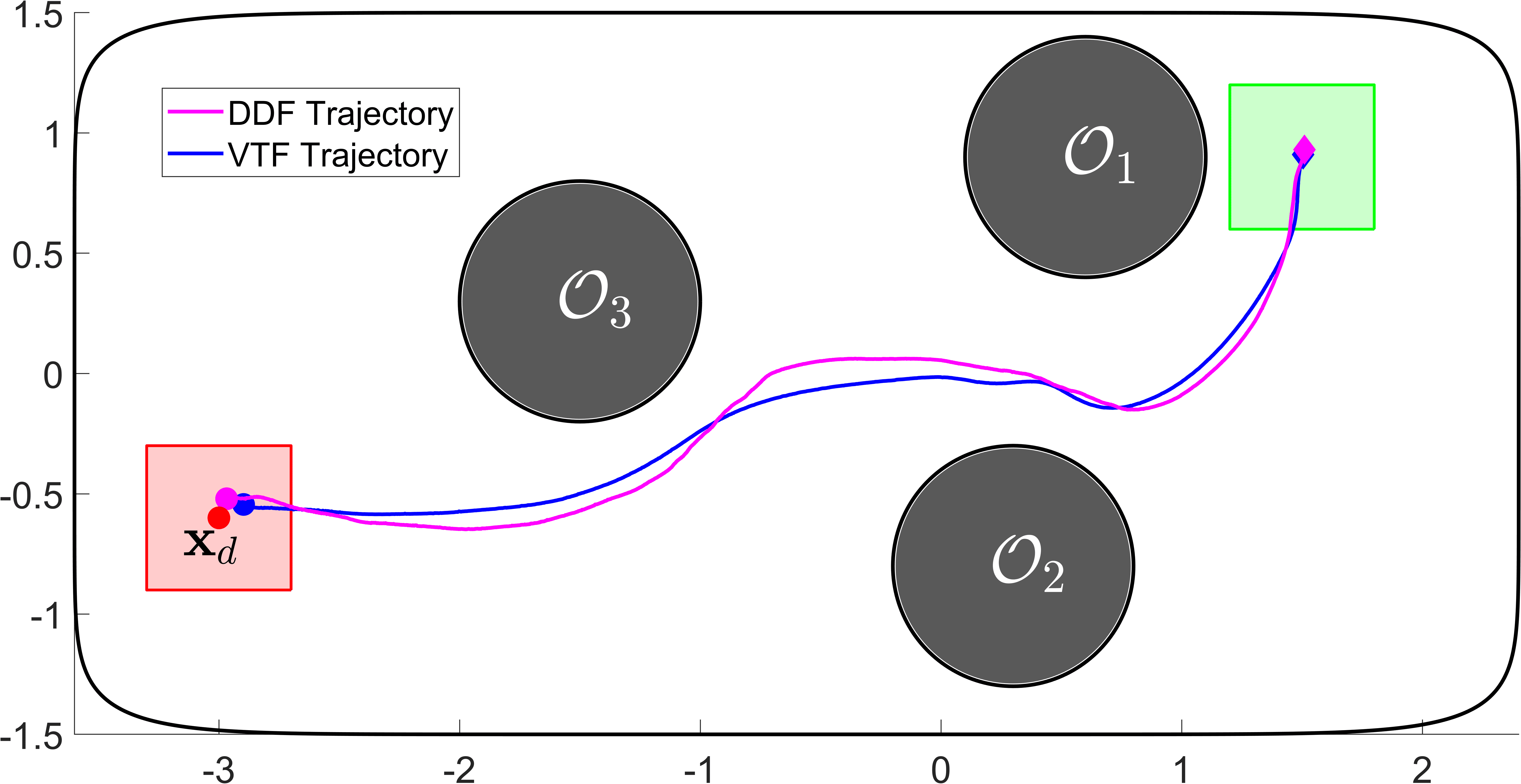}
    \caption{Top view of the quadrotor trajectories under the proposed control strategies.}
    \label{fig:top_view_experiment}
\end{figure}

\begin{figure}[ht]
    \centering

    \begin{subfigure}[b]{\linewidth}
        \centering
        \includegraphics[width=\linewidth]{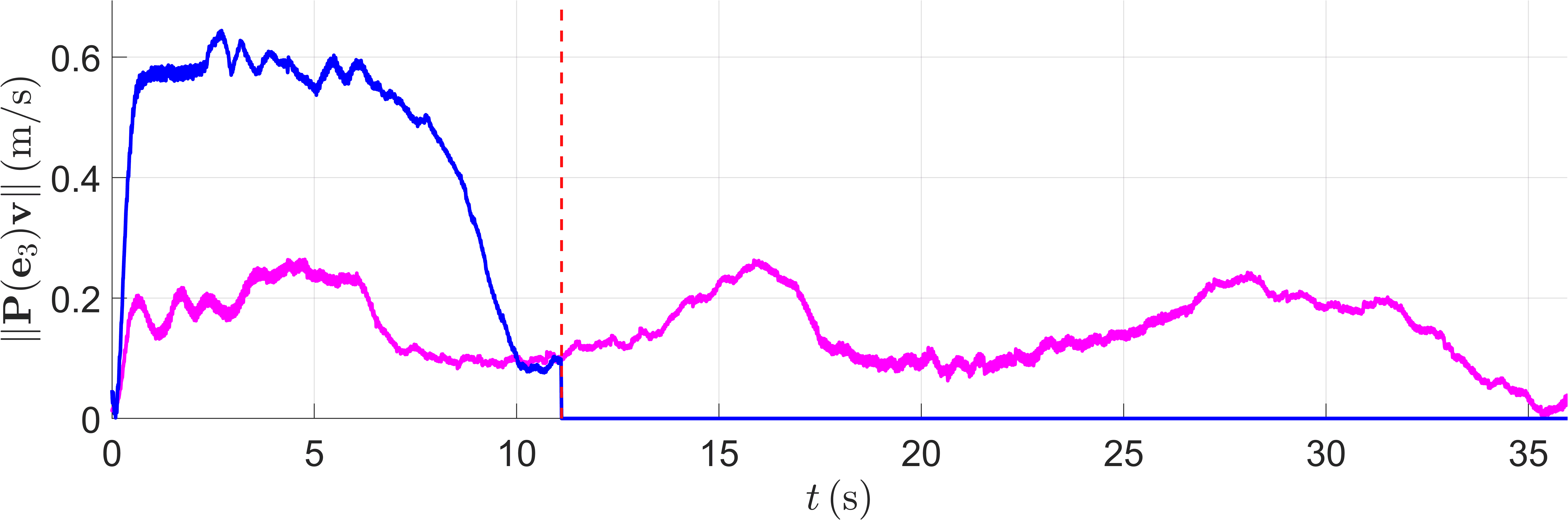}
        \caption{}
        \label{fig:experiment_velocity_norm}
    \end{subfigure}

    \medskip

    \begin{subfigure}[b]{\linewidth}
        \centering
        \includegraphics[width=\linewidth]{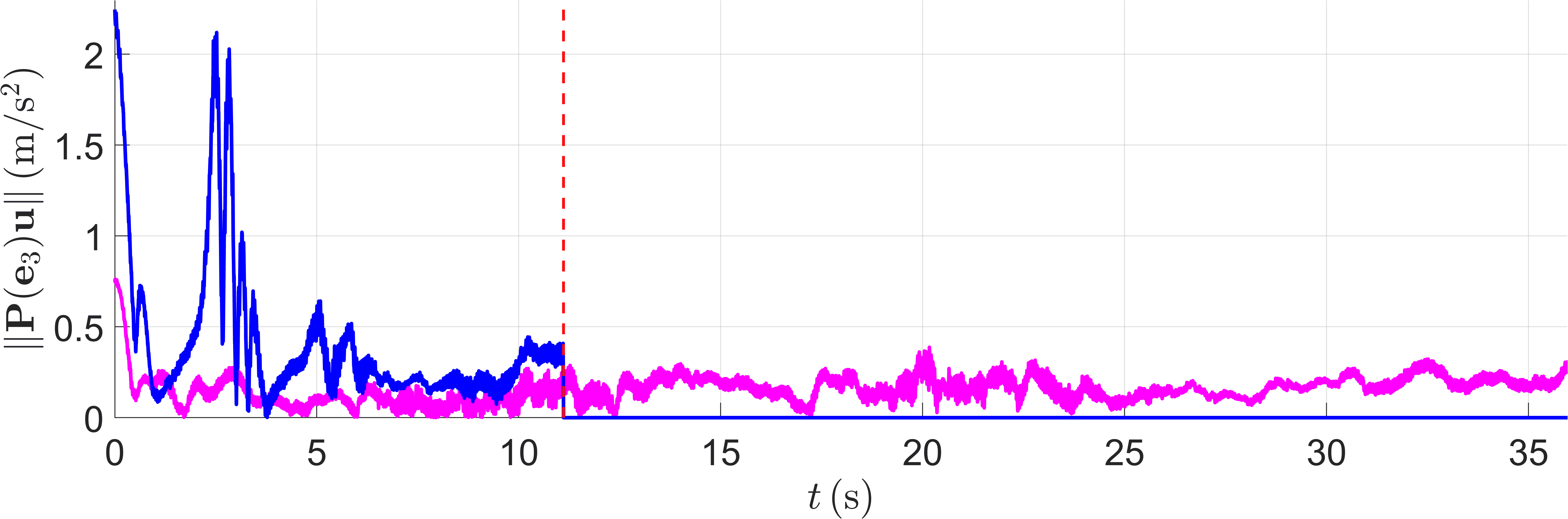}
        \caption{}
        \label{fig:experiment_acceleration_norm}
    \end{subfigure}

    \caption{Comparison of the proposed control strategies, with magenta curves obtained under the DDF control law and blue curves obtained under the VTF control law. (a) Norm of the horizontal component of the quadrotor velocity versus time. (b) Norm of the horizontal component of the control input vector versus time.}
    \label{fig:experiment_velocity_and_acceleration_norm}
\end{figure}

\section{Conclusion}\label{section:conclusion}

The paper proposes two control strategies for adapting existing feedback-based kinematic motion planners that satisfy Properties \ref{property:move_away}-\ref{property:distance_differentiable_at_E} to systems governed by second-order dynamics, while retaining their safety and convergence properties.

For gradient-based kinematic motion planners, the DDF control strategy introduces a damping term whose gain increases as the robot approaches an obstacle.
The resulting increase in damping reduces the robot's speed near the obstacle boundary, allowing the control input to eventually drive the robot's velocity $\mathbf{v}$ toward the desired velocity prescribed by the kinematic motion planner $\mathbf{v}_d$.
For feedback-based kinematic motion planners that are not derived from a navigation function, the VTF control strategy guarantees a monotonic decrease in the norm of the error between the robot's velocity and the desired velocity, as illustrated in Fig. \ref{fig:planner2:VTF:velocity-error-norm}.

In the DDF control strategy, an increase in $\beta(d_r(\mathbf{x}))$ increases the magnitude of the damping term, thereby reducing the robot's velocity near obstacles.
Whereas in the VTF control strategy, the increase in $\beta(d_r(\mathbf{x}))$ increases the rate at which the robot's velocity tracks the desired velocity.
Consequently, as observed in Fig. \ref{fig:planner1:results-1}, the VTF-controlled trajectories generally converge to the target faster than the corresponding DDF-controlled trajectories.

Although the proposed control schemes provide rigorous safety and convergence guarantees under the stated assumptions, the present analysis does not account for control input constraints.
Incorporating control saturation while preserving these guarantees is an important direction for future work.
Furthermore, the analysis requires the obstacle boundaries to be sufficiently smooth so that the distance function $d_r$ is at least twice continuously differentiable in a neighborhood of each obstacle boundary.
Extending the proposed framework to kinematic motion planners operating in environments with non-smooth obstacle boundaries is another important direction for future research.

\begin{appendix}

    \subsection{Chain recurrent set}\label{appendix:chain_recurrent_set_concept} Let $\Phi(s, \mathbf{p})$ denote the point at time $s\geq0$ along the unique solution of $\dot{\mathbf{x}} = \mathbf{v}_d(\mathbf{x})$ initialized at $\mathbf{p}\in\mathcal{X}_r^{\circ}$.
    By forward invariance of $\mathcal{X}_r^{\circ}$, as established in Remark \ref{remark:explanation-of-properties}, $\Phi(s, \mathbf{p})$ is defined and belongs to $\mathcal{X}_r^{\circ}$ for any $\mathbf{p}\in\mathcal{X}_r^{\circ}$ and for every $s\geq 0$.
    We recall the following definitions of $(\varepsilon, T)$-chains and the chain recurrent set from \cite[Definition 1.3]{mischaikow1995asymptotically}.

    \begin{definition}[$(\varepsilon, T)$-chain between two points]\label{definition:e-T_chain}
    Given $\varepsilon > 0$ and  $T > 0$, an $(\varepsilon, T)$-chain from $\mathbf{x}$ to $\mathbf{y}$ in $\mathcal{X}_r^{\circ}$ is a sequence $\{\mathbf{x}_1, \mathbf{x}_2, \ldots, \mathbf{x}_{N+1}; t_1, t_2, \ldots, t_N\}$ of points $\mathbf{x}_i\in\mathcal{X}_r^{\circ}$ and times $t_i\geq T$, where $N\in\mathbb{N}$, such that
    \begin{enumerate}
        \item $\mathbf{x}_1 = \mathbf{x}$ and $\mathbf{x}_{N+1} = \mathbf{y}$.
        \item $\|\Phi(t_i, \mathbf{x}_i) - \mathbf{x}_{i+1}\| < \varepsilon,\, i = 1, 2, \ldots, N$.\end{enumerate}\end{definition}

    \begin{definition}[Chain recurrent set]\label{definition:chain_recurrent_set}
        A point $\mathbf{p}$ in $\mathcal{X}_r^{\circ}$ is chain recurrent for $\dot{\mathbf{x}} = \mathbf{v}_d(\mathbf{x})$ if, for every $\varepsilon > 0$ and $T > 0$, there exists an $(\varepsilon, T)$-chain from $\mathbf{p}$ to itself.
        The set of all such points is the chain recurrent set associated with $\dot{\mathbf{x}} = \mathbf{v}_d(\mathbf{x})$.
    \end{definition}

\subsection{Proof of Lemma \ref{lemma:DDF}}\label{proof:lemma:DDF}
{
\subsubsection{Proof of Claim \ref{lemma:NF_claim1}}
The proof is inspired by the proof of Claim 1 of \cite[Lemma 1]{tang2023constructive}.
We proceed by contradiction.
Assume that there exists a finite time $T>0$ such that $d_r(\mathbf{x}(T)) = 0$ and $d_r(\mathbf{x}(t)) > 0, \forall t\in [0, T)$.
Since $\mathbf{x}(t)$ is continuous $\forall t\in[0, T]$, by \eqref{distance_function_d_x}, $d_r(\mathbf{x}(t))$ is continuous $\forall t\in[0, T]$.
Consequently, there exists $t_1\in[0, T)$ such that
\begin{equation}\label{DDF:bounded_distance_over_t1_to_T}
    d_r(\mathbf{x}(t))\in(0, \rho], \;\forall t\in[t_1, T),
\end{equation}
where $\rho = \min\{\delta_d, \epsilon_1\}$, the existence of $\delta_d>0$ is guaranteed by Property \ref{property:move_away}, and $\epsilon_1 > 0$ is defined in \eqref{beta_function_definition}.

By Property \ref{property:move_away}, $\delta_d \in(0, \delta_u]$, where the existence of $\delta_u > 0$ is assumed in Assumption \ref{assumption:twice_differentiable_boundary}.
Therefore, $\rho = \min\{\delta_d, \epsilon_1\}$ implies $0<\rho \leq \delta_u$.
Consequently, since $d_r(\mathbf{x}(t))\in(0, \rho],\;\forall t\in[t_1, T)$, $d_r(\mathbf{x}(t))$ is twice continuously differentiable over the time interval $[t_1, T)$, as discussed earlier in Remark \ref{remark:twice_differentiability}, and one has
\begin{equation}\label{DDF:distance_derivative}
    \dot{d}_r(\mathbf{x}(t)) = \boldsymbol{\eta}(\mathbf{x}(t))^\top\mathbf{v}(t), \;\forall t\in[t_1, T),
\end{equation}
where $\boldsymbol{\eta}(\mathbf{x})$ is defined in \eqref{Eq:eta-is-gradient-of-dr}.
Taking the time derivative of \eqref{DDF:distance_derivative} and using \eqref{controller:DDF}, one obtains
\begin{equation}\label{DDF:double_differentiation_distance}
\begin{aligned}
    \ddot{d}_r(\mathbf{x}(t)) = &- k_d\beta(d_r(\mathbf{x}(t)))\dot{d}_r(\mathbf{x}(t)) - k_1\boldsymbol{\eta}(\mathbf{x}(t))^\top\nabla_{\mathbf{x}}\varphi(\mathbf{x}(t))\\&+\mathbf{v}(t)^\top\mathbf{H}(\mathbf{x}(t))\mathbf{v}(t),
    \end{aligned}
\end{equation}
where $\mathbf{H}(\mathbf{x}) = \frac{\partial^2d_r(\mathbf{x})}{\partial\mathbf{x}^2}$.
Since $\rho = \min\{\delta_d, \epsilon_1\}$, by \eqref{beta_function_definition} and \eqref{DDF:bounded_distance_over_t1_to_T}, $\beta(d_r(\mathbf{x}(t))) = d_r(\mathbf{x}(t))^{-1}, \;\forall t\in[t_1, T)$.
Substituting this into \eqref{DDF:double_differentiation_distance} and rearranging the terms, one gets
\begin{equation}\label{DDF:other_double_differentiation_equation}
\begin{aligned}
    k_d\frac{\dot{d}_r(\mathbf{x}(t))}{d_r(\mathbf{x}(t))} = &-\ddot{d}_r(\mathbf{x}(t))-k_1\boldsymbol{\eta}(\mathbf{x}(t))^\top\nabla_{\mathbf{x}}\varphi(\mathbf{x}(t))\\  &+ \mathbf{v}(t)^\top\mathbf{H}(\mathbf{x}(t))\mathbf{v}(t).
    \end{aligned}
\end{equation}

Integrating \eqref{DDF:other_double_differentiation_equation} with respect to time from $t_1$ to $t\in[t_1, T)$, one obtains
\begin{equation}\label{equation_after_integration_NF}
    \begin{aligned}k_d\ln\left(\frac{d_r(\mathbf{x}(t))}{d_r(\mathbf{x}(t_1))}\right) &=  \dot{d}_r(\mathbf{x}(t_1)) - \dot{d}_r(\mathbf{x}(t))\\&-k_1\int_{t_{1}}^t\boldsymbol{\eta}(\mathbf{x}(\tau))^\top\nabla_{\mathbf{x}}\varphi(\mathbf{x}(\tau)) d\tau \\&+ \int_{t_{1}}^t\mathbf{v}(\tau)^\top\mathbf{H}(\mathbf{x}(\tau))\mathbf{v}(\tau)d\tau.
    \end{aligned}
\end{equation}
As $t\to T$, $d_r(\mathbf{x}(t))\to 0$ and the left-hand side of $\eqref{equation_after_integration_NF}$ approaches $-\infty$.
We proceed to analyze the right-hand side of $\eqref{equation_after_integration_NF}$ as $t\to T$.

First, we show that $\mathbf{v}(t)$ is bounded $\forall t\in[t_1, T)$.
This will allow us to show that the term $\dot{d}_r(\mathbf{x}(t))$ in \eqref{equation_after_integration_NF} remains bounded over the time interval $[t_1, T)$.
Define
\begin{equation}\label{DDF:definitionVd}
    V_d(\boldsymbol{\xi}) = k_1\varphi(\mathbf{x}) + \frac{1}{2}\|\mathbf{v}\|^2,
\end{equation}
where $\varphi(\mathbf{x})$ is a navigation function that is positive definite with respect to $\mathbf{x}_d$ over $\mathcal{X}_r^{\circ}$, as defined in Problem \ref{problem_with_existence_function}.
Differentiating \eqref{DDF:definitionVd} with respect to time and using \eqref{controller:DDF}, one obtains
\begin{equation}\label{DDF:Vddot_expression}
    \dot{V}_d(\boldsymbol{\xi}) = -k_d\beta(d_r(\mathbf{x}))\|\mathbf{v}\|^2.
\end{equation}
Since, according to \eqref{beta_function_definition}, there exists $\underline{\beta} > 0$ such that $\beta(d_r(\mathbf{x})) \geq \underline{\beta}$ for all $\mathbf{x}\in\mathcal{X}_r^{\circ}$, and $\beta(d_r(\mathbf{x}))$ is undefined only if $d_r(\mathbf{x}) = 0$, by \eqref{DDF:bounded_distance_over_t1_to_T}, it is true that 
\begin{equation}\label{DDF:Vddot_upper_bound}
    \dot{V}_d(\boldsymbol{\xi}(t))\leq -k_d\underline{\beta}\|\mathbf{v}(t)\|^2,
\end{equation}over the time interval $[t_1, T)$.
Therefore, $V_d(\boldsymbol{\xi}(t))$ remains bounded $\forall t\in[t_1, T)$, and hence $\mathbf{v}(t)$ is also bounded $\forall t\in[t_1, T)$.

Since $\rho = \min\{\delta_d, \epsilon_1\}$, and $\delta_d \in (0, \delta_u]$, one has $0 < \rho \leq \delta_d \leq \delta_u$, where $\delta_u$ is defined in Assumption \ref{assumption:twice_differentiable_boundary}, $\delta_d$ is defined in Property \ref{property:move_away}, and $\epsilon_1$ is defined in \eqref{beta_function_definition}.
Therefore, by \eqref{DDF:bounded_distance_over_t1_to_T}, $\mathbf{x}(t)\in\mathcal{N}_{\delta_d}(\mathcal{X}_r^c), \;\forall t\in[t_1, T)$, and it holds that $\boldsymbol{\eta}(\mathbf{x}(t))$ defined in \eqref{Eq:eta-is-gradient-of-dr} is bounded over the time interval $[t_1, T)$.
Consequently, since $\mathbf{v}(t)$ is bounded over $[t_1, T)$, it follows from \eqref{DDF:distance_derivative} that $\dot{d}_r(\mathbf{x}(t))$ is bounded $\forall t\in[t_1, T)$.
Moreover, since $\mathbf{x}(t)\in\mathcal{N}_{\delta_d}(\mathcal{X}_r^c), \;\forall t\in[t_1, T)$, it follows from Property \ref{property:move_away}, with $\mathbf{v}_d(\mathbf{x})=-k_1\nabla_{\mathbf{x}}\varphi(\mathbf{x})$, that \begin{equation}\label{DDF:bounded_move_away_condition}
\Lim_{t\to T}-k_1\int_{t_1}^t\boldsymbol{\eta}(\mathbf{x}(\tau))^\top\nabla_{\mathbf{x}}\varphi(\mathbf{x}(\tau)) d\tau \geq \mu(T- t_1).
\end{equation}

Finally, we show that $\Lim_{t\to T}\int_{t_{1}}^t\mathbf{v}^\top\mathbf{H}(\mathbf{x})\mathbf{v}d\tau\ne-\infty$.
Since $\rho\in(0, \delta_u]$, Assumption \ref{assumption:twice_differentiable_boundary} ensures that the distance function $d_r$, defined in \eqref{distance_function_d_x}, is twice continuously differentiable on an open neighborhood of $\overline{\mathcal{N}_{\rho}(\mathcal{X}_r^c)}$.
Therefore, by \cite[Theorem 9.41]{rudin1976principles}, its Hessian $\mathbf{H}(\mathbf{x}) = \frac{\partial^2d_r(\mathbf{x})}{\partial\mathbf{x}^2}$ is continuous and symmetric on $\overline{\mathcal{N}_{\rho}(\mathcal{X}_r^c)}$.
Furthermore, as mentioned earlier $\mathbf{x}(t)\in\mathcal{N}_{\rho}(\mathcal{X}_r^c), \;\forall t\in[t_1, T)$, where $\rho\in(0, \delta_u]$.
Therefore, the compactness of the set $\overline{\mathcal{N}_{\rho}(\mathcal{X}_r^c)}$ implies that $\mathbf{H}(\mathbf{x}(t))$ is bounded $\forall t\in[t_1, T)$.
Consequently,
\begin{equation}\label{DDF:Hessian_eigen_bound}
    \mathbf{v}^\top\mathbf{H}(\mathbf{x})\mathbf{v}\geq -|\lambda_{\mathbf{H}}^{\min}|\|\mathbf{v}\|^2,
\end{equation}
where $\lambda_{\mathbf{H}}^{\min}\in\mathbb{R}$ is smallest eigenvalue of $\mathbf{H}$ over $\overline{\mathcal{N}_{\rho}(\mathcal{X}_r^c)}$.
Multiplying and dividing the right-hand side by $k_d\underline{\beta}$, integrating with respect to time from $t_1$ to $t\in[t_1, T)$, and using \eqref{DDF:Vddot_upper_bound}, one gets
\begin{equation}\label{DDF:bounded_Hessian_Vd_term}
\begin{aligned}
    \int_{t_1}^t\mathbf{v}(\tau)^\top\mathbf{H}(\mathbf{x}(\tau))\mathbf{v}(\tau)d\tau &\geq \frac{|\lambda_{\mathbf{H}}^{\min}|}{k_d\underline{\beta}}\int_{t_1}^t-k_d\underline{\beta}\|\mathbf{v}(\tau)\|^2d\tau,\\
    &\geq \frac{|\lambda_{\mathbf{H}}^{\min}|}{k_d\underline{\beta}}(V_d(\boldsymbol{\xi}(t)) - V_d(\boldsymbol{\xi}(t_1))),
    \end{aligned}
\end{equation}
where $V_d$ defined in \eqref{DDF:definitionVd} is a positive definite function with respect to $(\mathbf{x}_d, \mathbf{0})$ over $\mathcal{X}_r^{\circ}\times\mathbb{R}^n$.
Note that since $d_r(\mathbf{x}(t))\in(0, \rho], \forall t\in[t_1, T)$,  $\dot{V}_d(\boldsymbol{\xi}(t))$ in \eqref{DDF:Vddot_expression} is well-defined $\forall t\in[t_1, T)$.
Since $V_d$ is positive definite with respect to $(\mathbf{x}_d, \mathbf{0})$ over $\mathcal{X}_r^{\circ}\times\mathbb{R}^n$ and $\mathbf{x}(t)\in\mathcal{X}_r^{\circ}, \forall t\in[t_1 , T)$, one has
\begin{equation}
    V_d(\boldsymbol{\xi}(t)) - V_d(\boldsymbol{\xi}(t_1)) \geq - V_d(\boldsymbol{\xi}(t_1)), \forall t\in[t_1, T).
\end{equation}
Substituting this into \eqref{DDF:bounded_Hessian_Vd_term} allows one to conclude that $\Lim_{t\to T}\int_{t_1}^t\mathbf{v}(\tau)^\top\mathbf{H}(\mathbf{x}(\tau))\mathbf{v}(\tau)d\tau\ne-\infty$.
As a result, the right-hand side of \eqref{equation_after_integration_NF} is bounded from below as $t\to T$, thereby leading to a contradiction.}

{
\subsubsection{Proof of Claim \ref{lemma:NF_claim2}}

We want to show that there exists $\sigma(\boldsymbol{\xi}(0)) >0 $ and $t_{\sigma}(\boldsymbol{\xi}(0))\geq 0$ such that $d_r(\mathbf{x}(t))\geq \sigma(\boldsymbol{\xi}(0)),\;\forall t\geq t_{\sigma}(\boldsymbol{\xi}(0))$.
In other words, we wish to show that
\begin{equation}
    \underset{t\to\infty}{\liminf}~d_r(\mathbf{x}(t)) > 0.
\end{equation}
We proceed by contradiction.
Since, by Claim \ref{lemma:NF_claim1} of Lemma \ref{lemma:DDF}, $d_r(\mathbf{x}(t))> 0, \;\forall t\geq 0$, we assume that
\begin{equation}\label{sequence_converges_to_0}
    \underset{t\to\infty}{\liminf}~d_r(\mathbf{x}(t)) = 0.
\end{equation}
Consequently, there exists a sequence of time $\{t_k\}_{k= 1}^{\infty}$ with $t_k\to\infty$ as $k\to\infty$ such that $d_r(\mathbf{x}(t_k))\in(0, \rho)$ and
\begin{equation}\label{sequence_tk_goes_to_0}
\Lim_{k\to\infty}d_r(\mathbf{x}(t_k)) = 0,
\end{equation}
where
$\rho = \min\{\delta_d, \epsilon_1\}$, $\delta_d$ is defined in Property \ref{property:move_away} and $\epsilon_1$ is defined in \eqref{beta_function_definition}.


For each $k\in\mathbb{N}$, let us define 
\begin{equation}\label{DDF:how_to_identify_sk}
\begin{aligned}
    \Delta_k := \underset{\Delta\in[0, t_k]}{\arg \min}&\; (d_r(\mathbf{x}(t_k-\Delta)) - \rho)^2,\\
    \text{such that }&d_r(\mathbf{x}(t))\in(0, \rho), \forall t\in(t_k-\Delta, t_k].
    \end{aligned}
\end{equation}
If, for some $k\in\mathbb{N}$, multiple minimizers exist, $\Delta_k$ is chosen as the largest minimizer.
Note that $ \Delta_k$ is such that at $t = t_k - \Delta_k$, the value $d_r(\mathbf{x}(t))$ is closest to $\rho$, and
$d_r(\mathbf{x}(t))\in(0, \rho)$ for all $t\in(t_k - \Delta_k, t_k]$.

By construction, $d_r(\mathbf{x}(t))\in(0, \rho]$ for all $t\in[t_k-\Delta_k, t_k]$ and for every $k\in\mathbb{N}$.
Additionally, as mentioned earlier in Remark \ref{remark:twice_differentiability}, by Assumption \ref{assumption:twice_differentiable_boundary}, $d_r$ is twice continuously differentiable over $\mathcal{N}_{\delta_u}(\mathcal{X}_r^c)$, where $\delta_u \geq \delta_d$ as stated in Property \ref{property:move_away}.
Therefore, since $\rho \leq \delta_d\leq \delta_u$, it follows that $d_r(\mathbf{x}(t))$ is twice continuously differentiable $\forall t\in\bigcup_{k\in\mathbb{N}}[t_k-\Delta_k, t_k]$.

Using steps analogous to those in the proof of Claim \ref{lemma:NF_claim1} of Lemma \ref{lemma:DDF}, similar to \eqref{equation_after_integration_NF}, one obtains
\begin{equation}\label{DDF:case1_logarithmic_equation}
    \begin{aligned}k_d\ln\bigg(\frac{d_r(\mathbf{x}(t_k))}{d_r(\mathbf{x}(t_k-\Delta_k))}&\bigg) =  \dot{d}_r(\mathbf{x}(t_k-\Delta_k)) - \dot{d}_r(\mathbf{x}(t_k))\\&-k_1\int_{t_k-\Delta_k}^{t_k}\boldsymbol{\eta}(\mathbf{x}(\tau))^\top\nabla_{\mathbf{x}}\varphi(\mathbf{x}(\tau)) d\tau \\&+ \int_{t_k-\Delta_k}^{t_k}\mathbf{v}(\tau)^\top\mathbf{H}(\mathbf{x}(\tau))\mathbf{v}(\tau)d\tau,
    \end{aligned}
\end{equation}
for every $k\in\mathbb{N}$.
By Property \ref{property:move_away},
\begin{equation}\label{DDF:sktk_mu_bound}
    -k_1\boldsymbol{\eta}(\mathbf{x}(t))^\top\nabla_{\mathbf{x}}\varphi(\mathbf{x}(t))\geq \mu, \forall t\in\bigcup_{k\in\mathbb{N}}[t_k -\Delta_k, t_k],
\end{equation}
where $\mu > 0$.
Moreover, similar to \eqref{DDF:bounded_Hessian_Vd_term}, one can write
\begin{equation}\label{DDF:sktk_H_bound}
\begin{aligned}
    \int_{t_k-\Delta_k}^{t_k}\mathbf{v}(\tau)^\top\mathbf{H}(\mathbf{x}(\tau))\mathbf{v}(\tau)d\tau \geq&\frac{|\lambda_{\mathbf{H}}^{\min}|}{k_d\underline{\beta}}(V_d(\boldsymbol{\xi}(t_k))\\ &- V_d(\boldsymbol{\xi}(t_k -\Delta_k)),
    \end{aligned}
\end{equation}
where $k_d > 0$, $\lambda_{\mathbf{H}}^{\min}$ is defined in \eqref{DDF:Hessian_eigen_bound}, and the existence of $\underline{\beta} > 0$ such that $\beta(d_r(\mathbf{x}))\geq\underline{\beta}$ for all $\mathbf{x}\in\mathcal{X}_r^{\circ}$ is implied by the design of $\beta(\cdot)$ in \eqref{beta_function_definition}.
Substituting \eqref{DDF:sktk_mu_bound} and \eqref{DDF:sktk_H_bound} into \eqref{DDF:case1_logarithmic_equation}, one obtains
\begin{equation}\label{DDF:logarithmic_inequality_case1}
    \begin{aligned}k_d&\ln\left(\frac{d_r(\mathbf{x}(t_k))}{d_r(\mathbf{x}(t_k - \Delta_k))}\right) \geq  \dot{d}_r(\mathbf{x}(t_k - \Delta_k)) - \dot{d}_r(\mathbf{x}(t_k))\\& + \mu\Delta_k + \frac{|\lambda_{\mathbf{H}}^{\min}|}{k_d\underline{\beta}}(V_d(\boldsymbol{\xi}(t_k)) - V_d(\boldsymbol{\xi}(t_k - \Delta_k))),
    \end{aligned}
\end{equation}
for every $k\in\mathbb{N}$.

By \eqref{DDF:how_to_identify_sk}, $d_r(\mathbf{x}(t_k-\Delta_k)) \geq d_r(\mathbf{x}(t_1-\Delta_1)), \forall k\in\mathbb{N}$.
Therefore, by \eqref{sequence_tk_goes_to_0}, it follows that the left-hand side of \eqref{DDF:logarithmic_inequality_case1} approaches $-\infty$ at $k$ tends to $\infty$.
We proceed to analyze the right-hand side of \eqref{DDF:logarithmic_inequality_case1} as $k\to\infty$.

By Claim \ref{lemma:NF_claim1} of Lemma \ref{lemma:DDF}, $d_r(\mathbf{x}(t))> 0, \forall t\geq 0$. 
Therefore, $\beta(d_r(\mathbf{x}(t)))$ is well-defined $\forall t\geq 0$.
This ensures that $\dot{V}_d(\boldsymbol{\xi}(t))$ obtained in \eqref{DDF:Vddot_expression} is well-defined and satisfies $\dot{V}_d(\boldsymbol{\xi}(t))\leq 0, \forall t\geq 0$.
Consequently, $V_d$ defined in \eqref{DDF:definitionVd} is non-increasing along the system trajectories for all $t\geq 0$.
It follows that $\mathbf{v}(t)$ is bounded for all $t\geq 0$.
Additionally, $\dot{d}_r(\mathbf{x}(t)) = \boldsymbol{\eta}(\mathbf{x}(t))^\top\mathbf{v}(t),\forall t\in\bigcup_{k\in\mathbb{N}}[t_k -\Delta_k, t_k]$, where $\boldsymbol{\eta}(\mathbf{x})$ is defined in \eqref{definition:eta}.
Since both $\mathbf{v}(t)$ and $\boldsymbol{\eta}(\mathbf{x}(t))$ are bounded $\forall t\in\bigcup_{k\in\mathbb{N}}[t_k -\Delta_k, t_k]$, it follows that the term $\dot{d}_r(\mathbf{x}(t_k-\Delta_k)) - \dot{d}_r(\mathbf{x}(t_k))$ in the right-hand side of \eqref{DDF:logarithmic_inequality_case1} is bounded for all $k\in\mathbb{N}$.

Moreover, for every $k\in\mathbb{N}$, $\Delta_k \geq 0$.
Furthermore, since $V_d$ defined in \eqref{DDF:definitionVd} is positive definite with respect to $(\mathbf{x}_d, \mathbf{0})$ over $\mathcal{X}_r^{\circ}\times\mathbb{R}^n$ and $\dot{V}_d(\boldsymbol{\xi}(t))\leq 0, \forall t\geq 0$, it follows that
\begin{equation}
    V_d(\boldsymbol{\xi}(t_k-\Delta_k)) \geq V_d(\boldsymbol{\xi}(t_k)) \geq 0, \forall k\in\mathbb{N}.
\end{equation}
As a result, the right-hand side of \eqref{DDF:logarithmic_inequality_case1} does not approach $-\infty$ as $k\to\infty$, thereby leading to a contradiction.
}

\subsubsection{Proof of Claim \ref{lemma:NF_claim3}}
By construction, $\beta(d_r(\mathbf{x})) > 0$ for all $\mathbf{x}\in\mathcal{X}_r^{\circ}$.
Therefore, by Claim \ref{lemma:NF_claim1} of Lemma \ref{lemma:DDF}, it follows that \begin{equation}\label{DDF:beta_lower_bound}\beta(d_r(\mathbf{x}(t)))> 0, \forall t\geq 0.\end{equation}
Therefore, taking norms on both sides of \eqref{controller:DDF}, one gets
\begin{equation}\label{DDF:first_bound}
    \|\mathbf{u}_d(\boldsymbol{\xi}(t))\| \leq k_1\|\nabla_{\mathbf{x}}\varphi(\mathbf{x}(t))\| + k_d\beta(d_r(\mathbf{x}(t)))\|\mathbf{v}(t)\|,
\end{equation}
for all $t\geq 0$, where we used the fact that for any $\|\mathbf{a} + \mathbf{b}\| \leq\|\mathbf{a}\| + \|\mathbf{b}\|, \forall \mathbf{a}, \mathbf{b}\in\mathbb{R}^n$.
As mentioned earlier in the proof of Claim \ref{lemma:NF_claim2} of Lemma \ref{lemma:DDF}, $\mathbf{v}(t)$ is bounded $\forall t\geq 0$.
Therefore, there exists $D_v(\boldsymbol{\xi}(0)) > 0$ such that $\|\mathbf{v}(t)\|\leq D_v(\boldsymbol{\xi}(0)), \forall t\geq 0$.
Substituting this bound in \eqref{DDF:first_bound}, one gets
\begin{equation}\label{DDF:second_bound}
    \|\mathbf{u}_d(\boldsymbol{\xi}(t))\| \leq k_1\|\nabla_{\mathbf{x}}\varphi(\mathbf{x}(t))\| + k_d\beta(d_r(\mathbf{x}(t)))D_v(\boldsymbol{\xi}(0)), 
\end{equation}

To proceed with the proof, we require the following fact:
\begin{fact}\label{Fact:precompact-DDF}
    For the closed-loop system \eqref{equation:second_order_system}-\eqref{controller:DDF}, under Assumption \ref{assumption:twice_differentiable_boundary}, for every initial condition $\boldsymbol{\xi}(0)\in\mathcal{X}_r^{\circ}\times\mathbb{R}^n$, the corresponding position trajectory $\mathbf{x}(t)$ is precompact in $\mathcal{X}_r^{\circ}$. 
\end{fact}
\proof{See Appendix \ref{proof:fact:precompact-DDF}.}

Since $\nabla_{\mathbf{x}}\varphi$ is continuous over $\mathcal{X}_r^{\circ}$, the precompactness of $\mathbf{x}(t)$ in $\mathcal{X}_r^{\circ}$, established in Fact \ref{Fact:precompact-DDF}, ensures the existence of $D_d(\boldsymbol{\xi}(0)) > 0$ such that
\[k_1\|\nabla_{\mathbf{x}}\varphi(\mathbf{x}(t))\| \leq D_d(\boldsymbol{\xi}(0)), \forall t\geq 0.\]

Furthermore, since, by \eqref{beta_function_definition}, $\beta\circ d_r$ is continuous for all $\mathbf{x}\in\mathcal{X}_r^{\circ}$, the precompactness of $\mathbf{x}(t)$ also ensures the existence of $\bar{\beta}(\boldsymbol{\xi}(0)) > 0$ such that
\[\beta(d_r(\mathbf{x}(t)))\leq \bar{\beta}(\boldsymbol{\xi}(0)), \forall t\geq 0.\]
Substituting these bounds in \eqref{DDF:second_bound}, one obtains $\|\mathbf{u}_d(\boldsymbol{\xi}(t))\| \leq B_d(\boldsymbol{\xi}(0))$, where
\begin{equation}
    B_d(\boldsymbol{\xi}(0)) = D_d(\boldsymbol{\xi}(0)) + k_d\bar{\beta}(\boldsymbol{\xi}(0))D_v(\boldsymbol{\xi}(0)).
\end{equation}
This completes the proof.

\subsection{Proof of Fact \ref{Fact:precompact-DDF}}\label{proof:fact:precompact-DDF}
Since $d_r(\mathbf{x}(0)) > 0$, it follows from Claims \ref{lemma:NF_claim1} and \ref{lemma:NF_claim2} of Lemma \ref{lemma:DDF} that there exists $\underline{d}_r(\boldsymbol{\xi}(0)) >0$ such that $d_r(\mathbf{x}(t)) \geq \underline{d}_r(\boldsymbol{\xi}(0)), \forall t \geq 0$. 
Consequently, for every $\boldsymbol{\xi}(0)\in\mathcal{X}_r^{\circ}\times\mathbb{R}^n$, $\mathbf{x}(t)\in\mathcal{K}(\boldsymbol{\xi}(0)), \forall t\geq 0$, where
\begin{equation*}
    \mathcal{K}(\boldsymbol{\xi}(0)) := \{\mathbf{x}\in\mathcal{X}_r^{\circ}\mid d_r(\mathbf{x})\geq \underline{d}_r(\boldsymbol{\xi}(0))\}\subset\mathcal{X}_r^{\circ}
\end{equation*}
is a closed subset of $\mathbb{R}^n$.
By Claim \ref{lemma:NF_claim1} of Lemma \ref{lemma:DDF}, if $\boldsymbol{\xi}(0)\in\mathcal{X}_r^{\circ}\times\mathbb{R}^n$, then \eqref{DDF:Vddot_upper_bound} holds for all $t\geq 0$.
Consequently,
\[
\varphi(\mathbf{x}(t))\leq \frac{1}{k_1}V_d(\boldsymbol{\xi}(t))\leq \frac{1}{k_1}V_d(\boldsymbol{\xi}(0)), \forall t\geq 0,
\]
where $V_d$ is defined in \eqref{DDF:definitionVd}.
Thus, $\mathbf{x}(t)\in\mathcal{K}_d(\boldsymbol{\xi}(0)), \forall t\geq 0$, where
\begin{equation*}
    \mathcal{K}_d(\boldsymbol{\xi}(0)):= \left\{\mathbf{x}\in\mathcal{K}(\boldsymbol{\xi}(0))\,\middle|\, \varphi(\mathbf{x})\leq \frac{1}{k_1}V_d(\boldsymbol{\xi}(0))\right\}.
\end{equation*}
By continuity of $\varphi$ on $\mathcal{X}_r^{\circ}$, closedness of $\mathcal{K}(\boldsymbol{\xi}(0))$ in $\mathbb{R}^n$, and boundedness of the sublevel sets of $\varphi$ as stated in Problem \ref{problem_with_existence_function}, the set $\mathcal{K}_d(\boldsymbol{\xi}(0))$ is compact in $\mathbb{R}^n$ and contained in $\mathcal{X}_r^\circ$.
Consequently, $\mathbf{x}(t)$ is precompact in $\mathcal{X}_r^{\circ}$.
This completes the proof.

\subsection{Proof of Theorem \ref{theorem:DDF}}\label{proof:theorem:DDF}
In the light of Lemma \ref{lemma:DDF}, the forward invariance of $\mathcal{X}_r^{\circ}\times\mathbb{R}^n$ for the closed-loop system \eqref{equation:second_order_system}-\eqref{controller:DDF} is straightforward to establish.

\subsubsection{Proof of Claim \ref{theorem:NF_claim2}}

For the closed-loop system \eqref{equation:second_order_system}-\eqref{controller:DDF}, by setting $\mathbf{v} = \mathbf{0}$ and $\dot{\mathbf{v}} = \mathbf{0}$, and using Property \ref{property:equilibrium_set}, one can verify that the set of equilibrium points is $\mathcal{S}\cup\{(\mathbf{x}_d, \mathbf{0})\}$, where $\mathcal{S}$ is defined in \eqref{definition:undesired_target_set}.

\subsubsection{Proof of Claim \ref{theorem:NF_claim3}}
The proof is separated into two parts as follows:

\noindent{\bf Part 1:} We show that the set $\mathcal{S}\cup\{(\mathbf{x}_d, \mathbf{0})\}$ is globally attractive for the closed-loop system \eqref{equation:second_order_system}-\eqref{controller:DDF} over $\mathcal{X}_r^{\circ}\times\mathbb{R}^n$.
Specifically, we show that, for every initial condition $\boldsymbol{\xi}(\mathbf{0})\in\mathcal{X}_r^{\circ}\times\mathbb{R}^n$, the corresponding state trajectory satisfies $\Lim_{t\to\infty}\boldsymbol{\xi}(t) = (\mathbf{x}^*, \mathbf{0})$ for some $(\mathbf{x}^*, \mathbf{0})\in\mathcal{S}\cup\{(\mathbf{x}_d, \mathbf{0})\}$.


Assume that $\Lim_{t\to\infty}\mathbf{v}(t) = \mathbf{0}$ and $\Lim_{t\to\infty}\dot{\mathbf{v}}(t) = \mathbf{0}$.
According to Claims \ref{lemma:NF_claim1} and \ref{lemma:NF_claim2} of Lemma \ref{lemma:DDF}, $\beta(d_r(\mathbf{x}(t)))$ is bounded for all $t\geq 0$, as stated earlier in the proof of Claim \ref{lemma:NF_claim3} of Lemma \ref{lemma:DDF}.
Therefore, since $\Lim_{t\to\infty}\mathbf{v}(t) = \mathbf{0}$, one has $\Lim_{t\to\infty}\beta(d_r(\mathbf{x}(t)))\mathbf{v}(t) = \mathbf{0}$.
Consequently, by \eqref{equation:second_order_system} and \eqref{controller:DDF}, $\Lim_{t\to\infty}\dot{\mathbf{v}}(t) = \mathbf{0}$ implies that
$\Lim_{t\to\infty}\nabla_{\mathbf{x}}\varphi(\mathbf{x}(t)) = \mathbf{0}$.
Since, 
by Fact \ref{Fact:precompact-DDF}, for \eqref{equation:second_order_system}-\eqref{controller:DDF},
$\mathbf{x}(t)$ is precompact in $\mathcal{X}_r^{\circ}$, and $\nabla_{\mathbf{x}}\varphi$ vanishes exactly on $\mathcal{E}\cup\{\mathbf{x}_d\}$ in $\mathcal{X}_r^{\circ}$, this implies that
\begin{equation}\label{global_attractivity_DDF}
    \Lim_{t\to\infty}d(\mathbf{x}(t), \mathcal{E}\cup\{\mathbf{x}_d\}) = 0.
\end{equation}
Since, by Property \ref{property:equilibrium_set}, $\mathcal{E}\cup\{\mathbf{x}_d\}$ is a finite set of isolated points and $\mathbf{x}(t)$ is continuous $\forall t\geq 0$, \eqref{global_attractivity_DDF} implies that $\Lim_{t\to\infty}\mathbf{x}(t) = \mathbf{x}^*$ for some $\mathbf{x}^*\in\mathcal{E}\cup\{\mathbf{x}_d\}$.
Consequently, the definition of $\mathcal{S}$ in \eqref{definition:undesired_target_set} and the fact that $\Lim_{t\to\infty}\mathbf{v}(t) = \mathbf{0}$ imply that $\Lim_{t\to\infty}\boldsymbol{\xi}(t) = (\mathbf{x}^*, \mathbf{0})$ for some $(\mathbf{x}^*, \mathbf{0})\in\mathcal{S}\cup\{(\mathbf{x}_d, \mathbf{0})\}$.

We proceed to prove that $\Lim_{t\to\infty}\mathbf{v}(t) = \mathbf{0}$ and $\Lim_{t\to\infty}\dot{\mathbf{v}}(t) = \mathbf{0}$.
By Claim \ref{lemma:NF_claim1} of Lemma \ref{lemma:DDF}, $d_r(\mathbf{x}(t)) > 0, \forall t\geq 0$, and it follows that $\beta(d_r(\mathbf{x}(t)))$ is well-defined $\forall t\geq 0$.
Therefore, $\dot{V}_d(\boldsymbol{\xi})$ obtained in \eqref{DDF:Vddot_expression} is well-defined $\forall t\geq 0$, where $V_d$ defined in \eqref{DDF:definitionVd} is a positive definite function with respect to $(\mathbf{x}_d, \mathbf{0})$ over $\mathcal{X}_r^{\circ}\times\mathbb{R}^n$.
From \eqref{DDF:Vddot_upper_bound}, \begin{equation}\label{DDF:Vddot_expression_again}\dot{V}_d(\boldsymbol{\xi}(t)) \leq -k_d\underline{\beta}\|\mathbf{v}(t)\|^2\leq 0, \forall t\geq 0.\end{equation} 
Since $V_d(\boldsymbol{\xi}(t))$ is non-increasing and bounded from below by zero, it converges as $t\to\infty$.
Integrating \eqref{DDF:Vddot_expression_again} with respect to time from $0$ to $t\in[0, \infty)$, and rearranging the terms, yields
\begin{equation}
    \int_0^t\|\mathbf{v}(\tau)\|^2d\tau \leq \frac{V_d(\boldsymbol{\xi}(0)) - V_d(\boldsymbol{\xi}(t))}{k_d\underline{\beta}}.
\end{equation}
Consequently, $\Lim_{t\to\infty}\int_0^t\|\mathbf{v}(\tau)\|^2d\tau$ exists and is finite.
Furthermore, by Claim \ref{lemma:NF_claim3} of Lemma \ref{lemma:DDF}, $\dot{\mathbf{v}}(t)$ is bounded for all $t\geq 0$, which implies that $\mathbf{v}(t)$ is uniformly continuous on $[0, \infty)$. 
Therefore, by Barbalat's lemma, $\Lim_{t\to\infty}\mathbf{v}(t) = \mathbf{0}$.

Next, to show that $\Lim_{t\to\infty}\dot{\mathbf{v}}(t)=\mathbf{0}$, we make use of the extension of Barbalat's lemma \cite[Lemma 1]{micaelli1993trajectory}, which is restated as follows:
\begin{lemma}\label{extension_of_Barbalat_lemma}
    Let $f(t)$ and $g(t)$ be two function from $\mathbb{R}_{\geq0}$ to $\mathbb{R}$ such that $f$ is differentiable and $g$ is uniformly continuous on $\mathbb{R}_{\geq0}$.
    If $\underset{t\to\infty}{\lim}f(t) = c$ and $\underset{t\to\infty}{\lim}\left(\dot{f}(t) - g(t)\right) = 0$, then $\underset{t\to\infty}{\lim}\dot{f}(t) = \underset{t\to\infty}{\lim}g(t) = 0$, where $c$ is a constant.
\end{lemma}

Note that Lemma \ref{extension_of_Barbalat_lemma}, which is applicable to scalar-valued functions, is being applied elementwise to the vector-valued functions $\mathbf{v}(t)$ and $-k_1\nabla_{\mathbf{x}}\varphi(\mathbf{x}(t))$.
By Fact \ref{Fact:precompact-DDF}, for any initial condition $\boldsymbol{\xi}(0)\in\mathcal{X}_r^{\circ}\times\mathbb{R}^n$, the corresponding state trajectory $\mathbf{x}(t)$ is precompact in $\mathcal{X}_r^{\circ}$ for the closed-loop system \eqref{equation:second_order_system}-\eqref{controller:DDF}.
Therefore, since, by assumption in Problem \ref{problem_with_existence_function}, $\varphi$ is twice continuously differentiable over $\mathcal{X}_r^{\circ}$, there exists $M_d(\boldsymbol{\xi}(0)) > 0$ such that
\begin{equation}k_1\left\|\frac{\partial^2\varphi(\mathbf{x}(t))}{\partial\mathbf{x}^2}\right\|_F \leq M_d(\boldsymbol{\xi}(0)), \forall t\geq 0.\end{equation}
Consequently, $-k_1\nabla_{\mathbf{x}}\varphi(\mathbf{x}(t))$ is uniformly continuous on $[0, \infty)$.
Moreover, according to Claims \ref{lemma:NF_claim1} and \ref{lemma:NF_claim2} of Lemma \ref{lemma:DDF}, $\beta(d_r(\mathbf{x}(t)))$ is bounded for all $t\geq 0$, as mentioned earlier in the proof of Claim \ref{lemma:NF_claim3} of Lemma \ref{lemma:DDF}.
Consequently, $\Lim_{t\to\infty}\mathbf{v}(t)=\mathbf{0}$ implies that $\Lim_{t\to\infty}\beta(d_r(\mathbf{x}(t)))\mathbf{v}(t) = \mathbf{0}$.
Note that by \eqref{controller:DDF}, \begin{equation}\dot{\mathbf{v}}(t) -(- k_1\nabla_{\mathbf{x}}\varphi(\mathbf{x}(t))) = -k_d\beta(d_r(\mathbf{x}(t)))\mathbf{v}(t).\end{equation}
Therefore, according to Lemma \ref{extension_of_Barbalat_lemma}, $\Lim_{t\to\infty}\dot{\mathbf{v}}(t)=\mathbf{0}$.
Since $\Lim_{t\to\infty}\mathbf{v}(t) = \Lim_{t\to\infty}\dot{\mathbf{v}}(t) = \mathbf{0}$, it follows by Claim \ref{theorem:NF_claim2} of Theorem \ref{theorem:DDF} that the set $\mathcal{S}\cup\{(\mathbf{x}_d, \mathbf{0})\}$ is globally attractive for the closed-loop system \eqref{equation:second_order_system}-\eqref{controller:DDF} over $\mathcal{X}_r^{\circ}\times\mathbb{R}^n$.

\noindent{\bf Part 2:} 
We show that for the closed-loop system \eqref{equation:second_order_system}-\eqref{controller:DDF}, the desired equilibrium $(\mathbf{x}_d, \mathbf{0})$ is asymptotically stable, and every equilibrium point in $\mathcal{S}$ has a stable manifold of zero Lebesgue measure in $\mathbb{R}^n\times\mathbb{R}^n$.

By Properties \ref{property:vd_continuously_differentiable} and \ref{property:distance_differentiable_at_E}, $\mathbf{u}_d$ in \eqref{controller:DDF} is continuously differentiable in an open neighborhood of every equilibrium point in $\mathcal{S}\cup\{(\mathbf{x}_d, \mathbf{0})\}$.
Therefore, to analyze the properties of the equilibrium points in $\mathcal{S}\cup\{(\mathbf{x}_d, \mathbf{0})\}$, we examine the eigenvalues of the Jacobian matrices of the closed-loop system \eqref{equation:second_order_system}-\eqref{controller:DDF} at these points.
For $(\mathbf{x}^*, \mathbf{0})\in\mathcal{S}\cup\{(\mathbf{x}_d, \mathbf{0})\}$, the Jacobian matrix $\mathbf{J}_d(\mathbf{x}^*, \mathbf{0})$ is given by
\begin{equation*}
    \mathbf{J}_d(\mathbf{x}^*, \mathbf{0}) = \begin{bmatrix}\mathbf{0}_n &\mathbf{I}_n\\\mathbf{M}_d(\mathbf{x}^*) &-k_d\beta(d_r(\mathbf{x}^*))\mathbf{I}_n\end{bmatrix},
\end{equation*}
where $d_r(\mathbf{x}^*) = d_s(\mathbf{x}^*, \mathcal{O}) - r$ defined as in \eqref{distance_function_d_x} and $\mathbf{M}_d(\mathbf{x}^*) = -k_1\frac{\partial^2\varphi(\mathbf{x}^*)}{\partial\mathbf{x}^2}$ satisfies Property \ref{property:eigenvalue_analysis}.

Let $\lambda$ denote the eigenvalues of $\mathbf{J}_d(\mathbf{x}^*, \mathbf{0})$. 
The matrix $\mathbf{J}_{d\lambda} = \mathbf{J}_d -\lambda\mathbf{I}_{n}$ is given by
\begin{equation*}
    \mathbf{J}_{d\lambda} = \begin{bmatrix}-\lambda\mathbf{I}_n &\mathbf{I}_n\\\mathbf{M}_d(\mathbf{x}^*) &-(\lambda + k_d\beta(d_r(\mathbf{x}^*)))\mathbf{I}_n\end{bmatrix}.
\end{equation*}
To proceed with the proof, we use \cite[Fact 2.14.13]{bernstein2009matrix}, which is restated as a fact below:
\begin{fact}\label{fact_block_matrix_determinant}
    If $\mathbf{M} = \begin{bmatrix}\mathbf{A} &\mathbf{B}\\\mathbf{C} &\mathbf{D}\end{bmatrix}$, where $\mathbf{A}, \mathbf{B}, \mathbf{C}, \mathbf{D}\in\mathbb{R}^{n\times n}$, and $\mathbf{A}\mathbf{C} = \mathbf{C}\mathbf{A}$, then
    \begin{equation*}
        \det(\mathbf{M}) = \det(\mathbf{A}\mathbf{D} - \mathbf{C}\mathbf{B}).
    \end{equation*}
\end{fact}


Using Fact \ref{fact_block_matrix_determinant}, one can verify that
\begin{equation*}
\begin{aligned}
    \det(\mathbf{J}_{d\lambda}) & = (-1)^n\det\left(\mathbf{M}_d(\mathbf{x}^*) -\lambda(\lambda + k_d\beta(d_r(\mathbf{x}^*)))\mathbf{I}_n \right).
    \end{aligned}
\end{equation*}
Therefore, $\det(\mathbf{J}_{d\lambda}) = 0$ if and only if $\lambda(\lambda + k_d\beta(d_r(\mathbf{x}^*))$ is an eigenvalue of $\mathbf{M}_d(\mathbf{x}^*)$.
Let $\theta_i(\mathbf{x}^*), i\in\{1, \ldots, n\}$ denote $n$ eigenvalues of $\mathbf{M}_d(\mathbf{x}^*)$.
Consequently, every eigenvalue $\lambda$ of $\mathbf{J}_{d}(\mathbf{x}^*, \mathbf{0})$ satisfies
\begin{equation}\label{quadratic_equation}
    \lambda^2 + k_d\beta(d_r(\mathbf{x}^*))\lambda - \theta_i(\mathbf{x}^*) = 0,
\end{equation}
for some $i\in\{1, \ldots, n\}$.
The expression for $\lambda$ is given by
\begin{equation}\label{quadratic_solution}
    \lambda = \frac{-k_d\beta(d_r(\mathbf{x}^*))\pm\sqrt{k_d^2\beta(d_r(\mathbf{x}^*))^2 + 4\theta_i(\mathbf{x}^*)}}{2}.
\end{equation}

Since $\varphi$ is twice continuously differentiable on $\mathcal{X}_r^{\circ}$, as stated in Problem \ref{problem_with_existence_function}, it follows from \cite[Theorem 9.41]{rudin1976principles} that $\mathbf{M}_d(\mathbf{x}) = -k_1\frac{\partial^2\varphi(\mathbf{x})}{\partial\mathbf{x}^2}$ is continuous and symmetric on $\mathcal{X}_r^{\circ}$.
Therefore, for every $\mathbf{x}^*\in\mathcal{E}\cup\{\mathbf{x}_d\}$, all eigenvalues of $\mathbf{M}_d(\mathbf{x}^*)$ are real.
Consequently, by Property \ref{property:stable_equilibria}, all eigenvalues of $\mathbf{M}_d(\mathbf{x}_d)$ are negative, \textit{i.e.,} $\theta_i(\mathbf{x}_d) < 0, \forall i\in\{1, \ldots, n\}$.
As a result, since $k_d > 0$ and $\beta(d_r(\mathbf{x}_d)) > 0$, using \eqref{quadratic_solution}, one can verify that all eigenvalues of $\mathbf{J}_d(\mathbf{x}_d, \mathbf{0})$ have negative real parts.
Consequently, $(\mathbf{x}_d, \mathbf{0})$ is an asymptotically stable equilibrium point for the closed-loop system \eqref{equation:second_order_system}-\eqref{controller:DDF}.

On the other hand, by Property \ref{property:unstable_equilibria} for every $\mathbf{x}^*\in\mathcal{E}$,  $\mathbf{M}_d(\mathbf{x}^*)$ has at least one positive eigenvalue.
Therefore, using \eqref{quadratic_solution}, one can verify that for every $(\mathbf{x}^*, \mathbf{0})\in\mathcal{S}$, $\mathbf{J}_d(\mathbf{x}^*, \mathbf{0})$ has at least one positive eigenvalue.
Therefore, by the center manifold theorem \cite[Section 2.7, Pg. 116]{perko2013differential},
the stable manifold associated with every $(\mathbf{x}^*, \mathbf{0})\in\mathcal{S}$ has zero Lebesgue measure in $\mathcal{X}_r^{\circ}\times\mathbb{R}^n$. 

Since, by Property \ref{property:equilibrium_set}, $\mathcal{E}$ consists only of isolated points, every element $(\mathbf{x}^*, \mathbf{0})$ in $\mathcal{S}$ is an isolated equilibrium point.
Therefore, since the stable manifold of each equilibrium point in $\mathcal{S}$ has zero Lebesgue measure, the union of these stable manifolds also has zero Lebesgue measure in $\mathcal{X}_r^{\circ}\times\mathbb{R}^n$.

As established earlier, $\mathcal{S}\cup\{(\mathbf{x}_d, \mathbf{0})\}$ is globally attractive for the closed-loop system \eqref{equation:second_order_system}-\eqref{controller:DDF} over $\mathcal{X}_r^{\circ}\times\mathbb{R}^n$.
Furthermore, $(\mathbf{x}_d, \mathbf{0})$ is an asymptotically stable equilibrium point and the union of the stable manifolds of the equilibria in $\mathcal{S}$ has zero Lebesgue measure in $\mathcal{X}_r^{\circ}\times\mathbb{R}^n$.
Therefore, $(\mathbf{x}_d, \mathbf{0})$ is almost globally asymptotically stable for the closed-loop system \eqref{equation:second_order_system}-\eqref{controller:DDF} over $\mathcal{X}_r^{\circ}\times\mathbb{R}^n$, and the proof of Claim \ref{theorem:NF_claim3} of Theorem \ref{theorem:DDF} is complete.

\subsection{Proof of Lemma \ref{lemma:VF}}\label{proof:lemma:VF}

\subsubsection{Proof of Claim \ref{lemma:VF_claim1}}
We proceed by contradiction.
Assume that there exists a finite time $T > 0$ such that $d_r(\mathbf{x}(T)) = 0$ and $d_r(\mathbf{x}(t))>0, \forall t\in[0, T)$.
Since $\mathbf{x}(t)$ is continuous $\forall t\in[0, T]$, by \eqref{distance_function_d_x}, $d_r(\mathbf{x}(t))$ is continuous $\forall t\in[0, T]$.
Consequently, there exists $t_1\in[0, T)$ such that
\begin{equation}\label{VTF:Bounded_distance}
    d_r(\mathbf{x}(t))\in(0, \rho], \forall t\in[t_1, T),
\end{equation}
where $\rho = \min\{\delta_d, \epsilon_1\}$, the existence of $\delta_d>0$ is guaranteed by Property \ref{property:move_away}, and $\epsilon_1 >0$ is defined in \eqref{beta_function_definition}.

By Property \ref{property:move_away}, $\delta_d\in(0, \delta_u]$, where the existence of $\delta_u > 0$ is assumed in Assumption \ref{assumption:twice_differentiable_boundary}.
Therefore, $\rho = \min\{\delta_d, \epsilon_1\}$ implies that $0<\rho\leq \delta_u$.
Consequently, since $d_r(\mathbf{x}(t))\in(0, \rho], \forall t\in[t_1, T)$, $d_r(\mathbf{x}(t))$ is twice continuously differentiable over the time interval $[t_1, T)$, as discussed earlier in Remark \ref{remark:twice_differentiability}, and one has
\begin{equation}\label{VTF:distance_derivative}
    \dot{d}_r(\mathbf{x}(t)) = \boldsymbol{\eta}(\mathbf{x}(t))^\top\mathbf{v}(t), \forall t\in[t_1, T),
\end{equation}
where $\boldsymbol{\eta}(\mathbf{x})$ is defined in \eqref{Eq:eta-is-gradient-of-dr}.

Defining $\mathbf{z} = \mathbf{v} - \mathbf{v}_d(\mathbf{x})$ and substituting $\mathbf{v} = \mathbf{v}_d(\mathbf{x}) + \mathbf{z}$ in \eqref{VTF:distance_derivative}, yields
\begin{equation}
    \dot{d}_r(\mathbf{x}(t)) = \boldsymbol{\eta}(\mathbf{x}(t))^\top\mathbf{v}_d(\mathbf{x}(t)) + \boldsymbol{\eta}(\mathbf{x}(t))^\top\mathbf{z}(t), \forall t\in[t_1, T).
\end{equation}
Since $\rho = \min\{\delta_d, \epsilon_1\}$, using Property \ref{property:move_away} and \eqref{VTF:Bounded_distance}, one obtains
\begin{equation}
    \dot{d}_r(\mathbf{x}(t))\geq \mu + \boldsymbol{\eta}(\mathbf{x}(t))^\top\mathbf{z}(t), \forall t\in[t_1, T),
\end{equation}
where $\mu > 0$.
Furthermore, since $\|\boldsymbol{\eta}(\mathbf{x}(t))\|=1, \forall t\in[t_1, T)$, it follows that
\begin{equation}\label{VTF:distance_mu_bound_claim1}
    \dot{d}_r(\mathbf{x}(t))\geq \mu - \|\mathbf{z}(t)\|, \forall t\in[t_1, T),
\end{equation}

Suppose $\Lim_{t\to T}\|\mathbf{z}(t)\| = 0$.
Therefore, there exists a time $t_s\in[t_1, T)$ such that $\forall t\in[t_s, T)$, $\|\mathbf{z}(t)\| < \mu$.
Consequently, $\dot{d}_r(\mathbf{x}(t)) > 0, \forall t\in[t_s, T)$.
In other words, there exists a time $t_s\in[t_1, T)$ such that $d_r(\mathbf{x}(t_s)) > 0$ and $d_r(\mathbf{x}(t))$ is strictly increasing over the time interval $[t_s, T)$.
As a result, since $d_r(\mathbf{x}(t))$ is continuous over $[t_s, T]$, $d_r(\mathbf{x}(T)) \ne 0$, thereby leading to a contradiction.

We proceed to show that $\Lim_{t\to T}\|\mathbf{z}(t)\| = 0$.
Define
\begin{equation}\label{VTF:Velocity_difference_scalar_function}
    V_v(\boldsymbol{\xi}) = \frac{1}{2}\|\mathbf{z}\|^2, \mathbf{z} = \mathbf{v} - \mathbf{v}_d(\mathbf{x}).
\end{equation}
Taking the time derivative and using \eqref{controller:VTF}, one obtains
\begin{equation}\label{VTF:Scalar_function_derivative}
    \dot{V}_v(\boldsymbol{\xi}) =  -k_d\beta(d_r(\mathbf{x}(t)))\|\mathbf{z}\|^2.
\end{equation}
Since $d_r(\mathbf{x}(t))\in(0, \rho], \forall t\in[t_1, T)$ and $\rho = \min\{\delta_d, \epsilon_1\}$, by \eqref{beta_function_definition}, one has $\beta(d_r(\mathbf{x}(t))) = d_r(\mathbf{x}(t))^{-1}, \forall t\in[t_1, T)$, and it follows from \eqref{VTF:Scalar_function_derivative} that
\begin{equation}\label{VTF:simplified_scalar_function_derivative}
    \dot{V}_v(\boldsymbol{\xi}(t)) = \frac{-2k_d}{d_r(\mathbf{x}(t))}V_v(\boldsymbol{\xi}(t)), \forall t\in [t_1, T).
\end{equation}

To show $\Lim_{t\to T}\|\mathbf{z}(t)\| = 0$, we proceed in the following manner.
First, we show that $\mathbf{v}(t)$ is bounded for all $t\in[t_1, T)$. 
This will imply that there exists $D_v > 0$ such that $\|\mathbf{v}(t)\|\leq D_v, \;\forall t\in[t_1, T)$.
Then, using \eqref{VTF:distance_derivative}, we will show that $d_{r}(\mathbf{x}(t)) \leq D_v(T - t)$ for all $t\in[t_1, T)$.
After substituting this bound into \eqref{VTF:simplified_scalar_function_derivative} and integrating with respect to time from $t_1$ to $t\in[t_1, T)$, we will be able to conclude that $\Lim_{t\to T}V_v(\boldsymbol{\xi}(t)) = 0$, and the result will follow.

We proceed to show that there exists $D_v > 0$ such that $\|\mathbf{v}(t)\|\leq D_v, \forall t\in[t_1, T)$.
Since $d_r(\mathbf{x}(t))>0, \forall t\in[t_1, T)$, $\dot{V}_v(\boldsymbol{\xi}(t))$ in \eqref{VTF:Scalar_function_derivative} is well-defined $\forall t\in[t_1, T)$, and $V_v(\boldsymbol{\xi}(t))$ is non-increasing on $[t_1, T)$.
Therefore, $\mathbf{z}(t)$ is bounded $\forall t\in[t_1, T)$ and satisfies $\|\mathbf{z}(t)\|\leq \|\mathbf{z}(t_1)\|, \forall t\in[t_1, T)$.
Additionally, as assumed in Problem \ref{problem_with_VF}, $\|\mathbf{v}_d(\mathbf{x})\|\leq \bar{D}_d, \forall \mathbf{x}\in\mathcal{X}_r^{\circ}$.
Therefore, since $\mathbf{v}(t) = \mathbf{z}(t) + \mathbf{v}_d(\mathbf{x}(t))$, one has
\begin{equation}
    \|\mathbf{v}(t)\| \leq D_v, \forall t\in[t_1, T),
\end{equation}
where $D_v = \|\mathbf{z}(t_1)\| + \bar{D}_d$.

Now, we show that $d_r(\mathbf{x}(t))\leq D_v(T - t), \forall t\in[t_1, T)$.
Since $d_r(\mathbf{x}(t))\in(0, \rho], \forall t\in[t_1, T)$ and $\rho \leq \delta_u$, it follows from \eqref{definition:eta} that $\|\boldsymbol{\eta}(\mathbf{x}(t))\| = 1, \forall t\in[t_1, T)$.
Therefore, by \eqref{VTF:distance_derivative}, one has
\begin{equation}
    \dot{d}_r(\mathbf{x}(t)) \geq -D_v, \forall t\in[t_1, T).
\end{equation}
Integrating with respect to time from $t\in[t_1, T)$ to $T$, one obtains
\begin{equation}
    d_r(\mathbf{x}(T)) - d_r(\mathbf{x}(t)) \geq -D_v(T-t).
\end{equation}
Substituting $d_r(\mathbf{x}(T)) = 0$, and rearranging the terms, one gets
\begin{equation}\label{VTF:transient_bound_on_distance}
    d_r(\mathbf{x}(t))\leq D_v(T-t), \forall t\in[t_1, T).
\end{equation}

Using \eqref{VTF:simplified_scalar_function_derivative} and \eqref{VTF:transient_bound_on_distance}, one has
\begin{equation}
    \dot{V}_v(\boldsymbol{\xi}(t)) \leq \frac{-2k_d}{D_v(T-t)}V_v(\boldsymbol{\xi}(t)), \forall t\in[t_1, T),
\end{equation}
where $k_d, D_v > 0$.
Integrating with respect to time from $t_1$ to $t\in[t_1, T)$, one obtains
\begin{equation}\label{VTF:scalar_function_is_upper_bounded}
    {V}_v(\boldsymbol{\xi}(t))\leq V_v(\boldsymbol{\xi}(t_1))\left(\frac{T-t}{T-t_1}\right)^{\frac{2k_d}{D_v}}.
\end{equation}
It follows from \eqref{VTF:scalar_function_is_upper_bounded} that $\Lim_{t\to T}V_v(\boldsymbol{\xi}(t)) = 0$, and hence $\Lim_{t\to T}\|\mathbf{z}(t)\| = 0$.
This completes the proof of Claim \ref{lemma:VF_claim1} of Lemma \ref{lemma:VF}.

\subsubsection{Proof of Claim \ref{lemma:VF_claim2}}
We will show that for every $\boldsymbol{\xi}(0)\in\mathcal{X}_r^{\circ}\times\mathbb{R}^n$, there exists $t_s(\boldsymbol{\xi}(0))\geq0$ and $\zeta >0$ such that $\dot{d}_r(\mathbf{x}(t)) \geq \zeta, \forall t\geq t_s(\boldsymbol{\xi}(0))$ whenever $d_r(\mathbf{x}(t))\in(0, \delta_d]$.
This will imply that there exists $t_v(\boldsymbol{\xi}(0))\geq t_s(\boldsymbol{\xi}(0))$ such that $d_r(\mathbf{x}(t))\geq \delta_d, \forall t\geq t_v(\boldsymbol{\xi}(0))$.

By Property \ref{property:move_away}, $\delta_d\in(0, \delta_u]$.
Therefore, following arguments used to derive \eqref{VTF:distance_mu_bound_claim1} in the proof of Claim \ref{lemma:VF_claim1} of Lemma \ref{lemma:VF}, one can conclude that the inequality
\begin{equation}\label{VTF:distance_mu_bound_claim2}
    \dot{d}_r(\mathbf{x}) \geq \mu - \|\mathbf{z}\|,
\end{equation}
holds for all $\mathbf{x}\in\mathcal{N}_{\delta_d}(\mathcal{X}_r^c)$, where the existence of $\mu >0$ is assumed in Property \ref{property:move_away}, $\mathbf{z} = \mathbf{v} - \mathbf{v}_d(\mathbf{x})$, and the $\delta_d$-neighborhood $\mathcal{N}_{\delta_d}(\mathcal{X}_r^c)$ of the unsafe region $\mathcal{X}_r^c$ is defined as in Section \ref{section:notations}.
To proceed with the proof, we require the following lemma:
\begin{lemma}\label{Lemma:velocity_difference_goes_to_0}
    Consider the closed-loop system \eqref{equation:second_order_system}-\eqref{controller:VTF}, under Assumption \ref{assumption:twice_differentiable_boundary}.
    Then, for all $t\geq0$, there exists $\underline{\beta}>0$ such that
    \[
    \|\mathbf{z}(t)\|\leq \|\mathbf{z}(0)\|e^{-k_d\underline{\beta}t},
    \]where  $\mathbf{z}(t) = \mathbf{v}(t)-\mathbf{v}_d(\mathbf{x}(t))$. 
\end{lemma}
\proof{Define $V_v(\boldsymbol{\xi}) = \frac{1}{2}\|\mathbf{z}\|^2$.
Taking the time derivative of $V_v$ and using \eqref{controller:VTF}, one obtains
\begin{equation*}
    \dot{V}_v = -k_d\beta(d_r(\mathbf{x}))\|\mathbf{z}\|^2.
\end{equation*}
According to \eqref{beta_function_definition}, there exists $\underline{\beta}>0$ such that $\beta(d_r(\mathbf{x})) \geq \underline{\beta}$ for all $\mathbf{x}\in\mathcal{X}_r^{\circ}$ and $\beta(d_r(\mathbf{x}))$ is undefined only if $d_r(\mathbf{x}) = 0$.
Since $\mathbf{x}(0)\in\mathcal{X}_r^{\circ}$, using Claim \ref{lemma:VF_claim1} of Lemma \ref{lemma:VF}, one can confirm that $\mathbf{x}(t) \in\mathcal{X}_r^{\circ}, \forall t\geq 0$.
Therefore, $\beta(d_r(\mathbf{x}(t)))$ is well-defined $\forall t\geq 0$. 
It follows that
\begin{equation}\label{FO_exponential_decrease}
    \dot{V}_v(\boldsymbol{\xi}(t))  \leq -2k_d\underline{\beta}V_v(\boldsymbol{\xi}(t))\leq 0, \forall t\geq 0.
\end{equation}
Consequently, one has
\begin{equation*}\label{exponential_decrease_1}
   0\leq V_v(\boldsymbol{\xi}(t))\leq V_v(\boldsymbol{\xi}(0))e^{-2k_d\underline{\beta}t}, \forall t\geq 0.
\end{equation*}
Substituting $V_v(\boldsymbol{\xi}) = \frac{1}{2}\|\mathbf{z}\|^2$ into the above inequality and taking the square roots yields the claim, thus completing the proof.}

According to Lemma \ref{Lemma:velocity_difference_goes_to_0}, $\Lim_{t\to\infty}\|\mathbf{z}(t)\| = 0$.
Therefore, for any $s\in\left(0, \mu\right)$, there exists $t_s(\boldsymbol{\xi}(0))\geq 0$ such that
\begin{equation}
    \|\mathbf{z}(t)\|\leq s, \forall t\geq t_s(\boldsymbol{\xi}(0)).
\end{equation}
Consequently, it follows from \eqref{VTF:distance_mu_bound_claim2} that for all $t\geq t_s(\boldsymbol{\xi}(0))$, the inequality $\dot{d}_r(\mathbf{x}(t))\geq \zeta > 0$ holds whenever $d_r(\mathbf{x}(t))\in(0, \delta_d]$ with $\zeta = \mu - s$.
As a result, there exists a time $t_v(\boldsymbol{\xi}(0))\geq t_s(\boldsymbol{\xi}(0))$ such that $d_r(\mathbf{x}(t))\geq \delta_d, \forall t\geq t_v(\boldsymbol{\xi}(0))$.
This completes the proof of Claim \ref{lemma:VF_claim2} of Lemma \ref{lemma:VF}.

\subsubsection{Proof of Claim \ref{lemma:VF_claim3}}
By construction, $\beta(d_r(\mathbf{x})) > 0$ for all $\mathbf{x}\in\mathcal{X}_r^{\circ}$.
Therefore, by Claim \ref{lemma:VF_claim1} of Lemma \ref{lemma:VF}, it follows that
\begin{equation}
    \beta(d_r(\mathbf{x}(t))) > 0, \forall t\geq 0.
\end{equation}
Therefore, taking norms of both sides of \eqref{controller:VTF}, one obtains
\begin{equation}\label{VTF:Control_bound_1}
    \|\mathbf{u}_v(\boldsymbol{\xi})\| \leq k_d\beta(d_r(\mathbf{x}))\|\mathbf{v} - \mathbf{v}_d(\mathbf{x})\| + \|\mathbf{M}_d(\mathbf{x})\|_F\|\mathbf{v}\|,
\end{equation}
where $\boldsymbol{\xi} = (\mathbf{x}, \mathbf{v})\in\mathcal{X}_r^{\circ}\times\mathbb{R}^n$ and we used the property $\|\mathbf{a} + \mathbf{b}\|\leq \|\mathbf{a}\| + \|\mathbf{b}\|$ for any $\mathbf{a}, \mathbf{b}\in\mathbb{R}^n$.

According to Lemma \ref{Lemma:velocity_difference_goes_to_0}, $\|\mathbf{v}(t) - \mathbf{v}_d(\mathbf{x}(t))\|\leq \bar{z}(\boldsymbol{\xi}(0)), \forall t\geq 0$, where $\bar{z}(\boldsymbol{\xi}(0)) = \|\mathbf{v}(0) - \mathbf{v}_d(\mathbf{x}(0))\|$.
Lemma \ref{Lemma:velocity_difference_goes_to_0} also implies that $\|\mathbf{v}(t)\|\leq \bar{D}_d + \bar{z}(\boldsymbol{\xi}(0)), \forall t\geq 0$, where the constant $\bar{D}_d>0$ is assumed in Problem \ref{problem_with_VF}.
Consequently, by \eqref{VTF:Control_bound_1}, for any $\boldsymbol{\xi}(0)\in\mathcal{X}_r^{\circ}\times\mathbb{R}^n$,
\begin{equation}\label{VTF:intermediate_bound_expression}
    \begin{aligned}\|\mathbf{u}_v(\boldsymbol{\xi}(t))\| &\leq k_d\beta(d_r(\mathbf{x}(t)))\bar{z}(\boldsymbol{\xi}(0))\\
    &+(\bar{D}_d + \bar{z}(\boldsymbol{\xi}(0)))\|\mathbf{M}_d(\mathbf{x}(t))\|_F,\end{aligned}
\end{equation}
for all $t\geq 0$.

To proceed with the proof, we require the following fact:
\begin{fact}\label{Fact:precompact-VTF}
    For the closed-loop system \eqref{equation:second_order_system}-\eqref{controller:VTF}, under Assumption \ref{assumption:twice_differentiable_boundary}, for every initial condition $\boldsymbol{\xi}(0)\in\mathcal{X}_r^{\circ}\times\mathbb{R}^n$, the corresponding position trajectory $\mathbf{x}(t)$ is precompact in $\mathcal{X}_r^{\circ}$. 
\end{fact}
\proof{See Appendix \ref{proof:fact:precompact-VTF}.}

Since $\beta\circ d_r$ is continuous over $\mathcal{X}_r^{\circ}$, the precompactness of $\mathbf{x}(t)$ in $\mathcal{X}_r^{\circ}$, established in Fact \ref{Fact:precompact-VTF}, guarantees the existence of $\bar{\beta}(\boldsymbol{\xi}(0))>0$ such that $\beta(d_r(\mathbf{x}(t))) \leq \bar{\beta}(\boldsymbol{\xi}(0)), \forall t\geq 0$.
Similarly, since $\mathbf{v}_d$ is continuously differentiable over $\mathcal{X}_r^{\circ}$ by Property \ref{property:vd_continuously_differentiable}, the precompactness of $\mathbf{x}(t)$ over $\mathcal{X}_r^{\circ}$ ensures that there exists $M_d(\boldsymbol{\xi}(0))>0$ such that $\|\mathbf{M}_d(\mathbf{x}(t))\|_F \leq M_d(\boldsymbol{\xi}(0)), \forall t\geq 0$.
Substituting these bounds in \eqref{VTF:intermediate_bound_expression}, one obtains $\|\mathbf{u}_v(\boldsymbol{\xi}(t))\|\leq B_v(\boldsymbol{\xi}(0)), \forall t\geq 0$, where
\begin{equation*}
    B_v(\boldsymbol{\xi}(0)) = k_d\bar{\beta}(\boldsymbol{\xi}(0))\bar{z}(\boldsymbol{\xi}(0)) + (\bar{D}_d + \bar{z}(\boldsymbol{\xi}(0)))M_d(\boldsymbol{\xi}(0)).
\end{equation*}
This completes the proof.

\subsection{Proof of Fact \ref{Fact:precompact-VTF}}\label{proof:fact:precompact-VTF}
Consider the map $T:\mathcal{X}_r^{\circ}\times\mathbb{R}^n\to\mathcal{X}_r^{\circ}\times\mathbb{R}^n$ given by $T(\mathbf{x}, \mathbf{v}) = (\mathbf{x}, \mathbf{v}- \mathbf{v}_d(\mathbf{x}))$.
Since, by Property \ref{property:vd_continuously_differentiable}, $\mathbf{v}_d$ is continuously differentiable over $\mathcal{X}_r^{\circ}$, $T$ is continuously differentiable over its domain.
Furthermore, its unique inverse function is given by $T^{-1}(\mathbf{x}, \mathbf{z}) = (\mathbf{x}, \mathbf{z} + \mathbf{v}_d(\mathbf{x}))$ and is also continuously differentiable over its domain. 
Consequently, $T$ is a diffeomorphism on $\mathcal{X}_r^{\circ}\times\mathbb{R}^n$.
Under the transformation $T$, the closed-loop system \eqref{equation:second_order_system}-\eqref{controller:VTF} becomes
\begin{subequations}\label{transformed_closed_loop_system}
\begin{align}
    \dot{\mathbf{x}}
    &= \mathbf{v}_d(\mathbf{x})+\mathbf{z},
    \label{transformed_closed_loop_system_x}\\
    \dot{\mathbf{z}}
    &= -k_d\beta(d_r(\mathbf{x}))\mathbf{z}.
    \label{transformed_closed_loop_system_z}
\end{align}
\end{subequations}

Since systems \eqref{equation:second_order_system}-\eqref{controller:VTF} and \eqref{transformed_closed_loop_system} are equivalent under the diffeomorphism $T$, we proceed to show that, for every initial condition $\boldsymbol{\xi}_z(0)\in\mathcal{X}_r^{\circ}\times\mathbb{R}^n$, the position trajectory $\mathbf{x}(t)$ of \eqref{transformed_closed_loop_system} is precompact in $\mathcal{X}_r^{\circ}$, where $\boldsymbol{\xi}_z:=(\mathbf{x}, \mathbf{z})\in\mathcal{X}_r^{\circ}\times\mathbb{R}^n$.

It follows from Claims \ref{lemma:VF_claim1} and \ref{lemma:VF_claim2} of Lemma \ref{lemma:VF} that there exists $\underline{\delta}(\boldsymbol{\xi}_z(0)) > 0$ such that $d_r(\mathbf{x}(t)) \geq \underline{\delta}(\boldsymbol{\xi}_z(0)), \forall t\geq 0$, \textit{i.e.}, \[\mathbf{x}(t)\in\mathcal{X}_{r + \underline{\delta}(\boldsymbol{\xi}_z(0))}, \forall t\geq 0,\]
where $\mathcal{X}_{r + \underline{\delta}(\boldsymbol{\xi}_z(0))}$ is obtained by replacing $p$ with $r + \underline{\delta}(\boldsymbol{\xi}_z(0))$ in \eqref{equation_free_space}.
The set $\mathcal{X}_{r+\underline{\delta}(\boldsymbol{\xi}_z(0))}$ is closed in $\mathbb{R}^n$ and contained in $\mathcal{X}_r^\circ$.
If $\mathcal{X}_r^{\circ}$ is bounded, then $\mathcal{X}_{r + \underline{\delta}(\boldsymbol{\xi}_z(0))}$ is compact in $\mathbb{R}^n$.
Therefore, for every initial condition in $\mathcal{X}_r^{\circ}\times\mathbb{R}^n$, the corresponding state trajectory $\mathbf{x}(t)$ is precompact in $\mathcal{X}_r^{\circ}$.

We consider the case where $\mathcal{X}_r^{\circ}$ is unbounded.
By the assumption in Problem \ref{problem_with_VF}, there exists $R > 0$ such that $(\mathbf{x} - \mathbf{x}_d)^\top\mathbf{v}_d(\mathbf{x}) \leq 0$ for all $\mathbf{x}\in\mathcal{X}_r^{\circ}$ satisfying $V_x(\mathbf{x})\geq R$, where $V_x(\mathbf{x}) = \|\mathbf{x} - \mathbf{x}_d\|$.
Let $t_1\geq 0$ be arbitrary such that $V_x(\mathbf{x}(t_1)) > R$. 
If $V_x(\mathbf{x}(t))\geq R$ for all $t\in[0, t_1]$, then set $t_0 = 0$.
Otherwise, by continuity, there exists $t_0\in(0, t_1)$ such that $V_x(\mathbf{x}(t_0)) = R$ and $V_x(\mathbf{x}(t))\geq R$ for all $t\in[t_0, t_1]$.

For $t\in[t_0, t_1]$, the time derivative of $V_x$ is given by
\begin{equation*}
    \dot{V}_x(\mathbf{x}(t)) = \frac{(\mathbf{x}(t) - \mathbf{x}_d)^\top\mathbf{v}_d(\mathbf{x}(t))}{\|\mathbf{x}(t) -\mathbf{x}_d\|} + \frac{(\mathbf{x}(t) - \mathbf{x}_d)^\top\mathbf{z}(t)}{\|\mathbf{x}(t) -\mathbf{x}_d\|}.
\end{equation*}
Since $(\mathbf{x}(t) - \mathbf{x}_d)^\top\mathbf{v}_d(\mathbf{x}(t)) \leq 0, \forall t\in[t_0, t_1]$, one has
\begin{equation}
    \dot{V}_x(\mathbf{x}(t))\leq \|\mathbf{z}(t)\|,\,\forall t\in[t_0, t_1].
\end{equation}
Integrating with respect to time from $t_0$ to $t_1$, one obtains
\begin{equation}\label{equation:V_x_value_at_t1}
\begin{aligned}
    V_x(\mathbf{x}(t_1)) &\leq V_x(\mathbf{x}(t_0)) + \int_{t_0}^{t_1}\|\mathbf{z}(\tau)\|d\tau\\
    &\leq \max\{V_x(\mathbf{x}(0)), R\} + \int_{0}^{\infty}\|\mathbf{z}(\tau)\|d\tau.
\end{aligned}
\end{equation}
By Lemma \ref{Lemma:velocity_difference_goes_to_0}, one has
\begin{equation}
    \|\mathbf{z}(t)\|\leq \|\mathbf{z}(0)\|e^{-k_d\underline{\beta}t},
\end{equation}
with $k_d > 0$ and $\underline{\beta} > 0$.
Integrating with respect to time from $0$ to $\infty$, one gets
\begin{equation}
    \int_{0}^\infty\|\mathbf{z}(\tau)\|d\tau \leq \frac{\|\mathbf{z}(0)\|}{k_d\underline{\beta}}.
\end{equation}
Substituting this into \eqref{equation:V_x_value_at_t1}, yields $V_x(\mathbf{x}(t_1)) \leq C(\boldsymbol{\xi}_z(0))$, where
\begin{equation}
    C(\boldsymbol{\xi}_z(0))= \max\{V_x(\mathbf{x}(0)), R\} + \frac{\|\mathbf{z}(0)\|}{k_d\underline{\beta}}.
\end{equation}
Since the above bound holds whenever $V_x(\mathbf{x}(t)) > R$, it follows that
\begin{equation}
    V_x(\mathbf{x}(t)) \leq C(\boldsymbol{\xi}_z(0)), \;\forall t\geq 0.
\end{equation}
Consequently, for any initial condition $\boldsymbol{\xi}_z(0)\in\mathcal{X}_r^{\circ}\times\mathbb{R}^n$, $\mathbf{x}(t)\in\mathcal{V}(\boldsymbol{\xi}_z(0)), \forall t\geq 0$, where
\begin{equation}
    \mathcal{V}(\boldsymbol{\xi}_z(0)) = \{\mathbf{x}\in\mathbb{R}^n\mid V_x(\mathbf{x}) \leq C(\boldsymbol{\xi}_z(0))\}.
\end{equation}
Furthermore, as mentioned earlier, by Claims \ref{lemma:VF_claim1} and \ref{lemma:VF_claim2} of Lemma \ref{lemma:VF}, for any initial condition $\boldsymbol{\xi}_z(0)\in\mathcal{X}_r^{\circ}\times\mathbb{R}^n$, $\mathbf{x}(t)\in\mathcal{X}_{r + \underline{\delta}(\boldsymbol{\xi}_z(0))}, \forall t\geq 0$.
Therefore, 
\[\mathbf{x}(t)\in\mathcal{V}(\boldsymbol{\xi}_z(0))\cap\mathcal{X}_{r + \underline{\delta}(\boldsymbol{\xi}_z(0))}, \forall t\geq 0.\]
Consequently, since, for every $\boldsymbol{\xi}_z(0)\in\mathcal{X}_r^{\circ}\times\mathbb{R}^n$, $\mathcal{V}(\boldsymbol{\xi}_z(0))\cap\mathcal{X}_{r + \underline{\delta}(\boldsymbol{\xi}_z(0))}$ is compact in $\mathbb{R}^n$, and is contained in $\mathcal{X}_r^{\circ}$, it follows that the corresponding state trajectory $\mathbf{x}(t)$ is precompact in $\mathcal{X}_r^{\circ}$.
This completes the proof.

\subsection{Proof of Theorem \ref{theorem:VF}}\label{proof:theorem_first_order_convergence}
In light of Lemma \ref{lemma:VF}, the forward invariance of $\mathcal{X}_r^{\circ}\times\mathbb{R}^n$ for the closed-loop system \eqref{equation:second_order_system}-\eqref{controller:VTF} is straightforward to establish.
Claim \ref{theorem:VF_decrease} of Theorem \ref{theorem:VF} follows from the inequality \eqref{FO_exponential_decrease} established in the proof of Lemma \ref{Lemma:velocity_difference_goes_to_0}.

\subsubsection{Proof of Claim \ref{theorem:VF_claim2}}

For the closed-loop system \eqref{equation:second_order_system}-\eqref{controller:VTF}, by setting $\mathbf{v} = \mathbf{0}$ and $\dot{\mathbf{v}} = \mathbf{0}$, and using Property \ref{property:equilibrium_set}, one can verify that the set of equilibrium points is $\mathcal{S}\cup\{(\mathbf{x}_d, \mathbf{0})\}$, where $\mathcal{S}$ is defined in \eqref{definition:undesired_target_set}.

\subsubsection{Proof of Claim \ref{theorem:VF_claim3}}

Under the diffeomorphism $T$ defined in the proof of 
Fact \ref{Fact:precompact-VTF},
the closed-loop systems \eqref{equation:second_order_system}-\eqref{controller:VTF} and \eqref{transformed_closed_loop_system} are equivalent.
Moreover, since $T(\mathbf{x}^*, \mathbf{v}^*) = (\mathbf{x}^*, \mathbf{v}^* - \mathbf{0}) = (\mathbf{x}^*, \mathbf{v}^*)$ for every $(\mathbf{x}^*, \mathbf{v}^*)\in\mathcal{S}\cup\{(\mathbf{x}_d, \mathbf{0})\}$, 
these equilibria are fixed points of $T$.
Furthermore, because $T$ is a diffeomorphism, the stability properties of equilibrium points are preserved under the state transformation \cite[Ex. 4.26]{khalil2002nonlinear}. 
Therefore, we proceed to analyze the stability properties of $(\mathbf{x}^*, \mathbf{z}^*) \in \mathcal{S}\cup\{(\mathbf{x}_d, \mathbf{0})\}$ for the transformed system \eqref{transformed_closed_loop_system}.
The proof is separated into two parts as follows:

\noindent{\bf Part 1:} We show that the set $\mathcal{S}\cup\{(\mathbf{x}_d, \mathbf{0})\}$ is globally attractive for the closed-loop system \eqref{transformed_closed_loop_system} over $\mathcal{X}_r^{\circ}\times\mathbb{R}^n$.
Specifically, we show that $\Lim_{t\to\infty}(\mathbf{x}(t), \mathbf{z}(t))\in\mathcal{S}\cup\{(\mathbf{x}_d, \mathbf{0})\}$.

According to Lemma \ref{Lemma:velocity_difference_goes_to_0}, for any initial condition $(\mathbf{x}(0), \mathbf{z}(0))\in\mathcal{X}_r^{\circ}\times\mathbb{R}^n$, the corresponding state trajectory $\mathbf{z}(t)$ converges exponentially to $\mathbf{0}$.
Consequently, the $\mathbf{x}$-subsystem \eqref{transformed_closed_loop_system_x} is asymptotically autonomous in the sense of \cite{mischaikow1995asymptotically}, with limit equation \eqref{kinematic_closed_loop_system}.

Since under the diffeomorphism $T$, the closed-loop systems \eqref{equation:second_order_system}-\eqref{controller:VTF} and \eqref{transformed_closed_loop_system} are equivalent, by Fact \ref{Fact:precompact-VTF}, for the closed-loop system \eqref{transformed_closed_loop_system}, for any initial condition in $\mathcal{X}_r^{\circ}\times\mathbb{R}^n$, the corresponding state trajectory $\mathbf{x}(t)$ is precompact in $\mathcal{X}_r^{\circ}$.
Therefore, by \cite[Theorem 1.8]{mischaikow1995asymptotically}, $\mathbf{x}(t)$ approaches the chain recurrent set associated with the autonomous system $\dot{\mathbf{x}} = \mathbf{v}_d(\mathbf{x})$.
Consequently, since, by the assumption in Problem \ref{problem_with_VF}, the chain recurrent set associated with $\dot{\mathbf{x}} = \mathbf{v}_d(\mathbf{x})$ is a subset of $\mathcal{E}\cup\{\mathbf{x}_d\}$, it follows that 
\begin{equation}\label{VTF:distance_to_equilibrium_set_reduces}
\Lim_{t\to\infty}d(\mathbf{x}(t), \mathcal{E}\cup\{\mathbf{x}_d\}) = 0.    
\end{equation}
Since, by Property \ref{property:equilibrium_set}, $\mathcal{E}\cup\{\mathbf{x}_d\}$ is a finite set of isolated points and $\mathbf{x}(t)$ is continuous $\forall t\geq 0$, \eqref{VTF:distance_to_equilibrium_set_reduces} implies that $\Lim_{t\to\infty}\mathbf{x}(t)\in\mathcal{E}\cup\{\mathbf{x}_d\}$.
Consequently, the definition of $\mathcal{S}$ in \eqref{definition:undesired_target_set} and the fact that $\Lim_{t\to\infty}\mathbf{z}(t) = \mathbf{0}$ imply that
$\Lim_{t\to\infty}(\mathbf{x}(t), \mathbf{z}(t))\in\mathcal{S}\cup\{(\mathbf{x}_d, \mathbf{0})\}$.

\noindent{\bf Part 2:}
We will show that for the closed-loop system \eqref{transformed_closed_loop_system}, the desired equilibrium $(\mathbf{x}_d, \mathbf{0})$ is asymptotically stable, and every equilibrium point in $\mathcal{S}$ has a stable manifold of zero Lebesgue measure in $\mathcal{X}_r^{\circ}\times\mathbb{R}^n$.

To analyze the properties of the equilibrium points in $\mathcal{S}\cup\{(\mathbf{x}_d, \mathbf{0})\}$, we examine the eigenvalues of the Jacobian matrix of the transformed system \eqref{transformed_closed_loop_system} at these points.
Since $\mathbf{z}^* = \mathbf{0}$, for each $(\mathbf{x}^*, \mathbf{z}^*)\in\mathcal{S}\cup\{(\mathbf{x}_d, \mathbf{0})\}$, the Jacobian matrix $\mathbf{J}_z(\mathbf{x}^*, \mathbf{z}^*)$ is given by
\begin{equation}
    \mathbf{J}_z(\mathbf{x}^*, \mathbf{z}^*) = \begin{bmatrix}\mathbf{M}_d(\mathbf{x}^*) & \mathbf{I}_{n}\\
    \mathbf{0}_n & -k_d\beta(d_r(\mathbf{x}^*))\mathbf{I}_n
    \end{bmatrix}.
\end{equation} 

Since $\mathbf{J}_z(\mathbf{x}^*, \mathbf{z}^*)$ is an upper block triangular matrix, by \cite[Proposition 5.5.13]{bernstein2009matrix}, the set of eigenvalues of $\mathbf{J}_z(\mathbf{x}^*, \mathbf{z}^*)$ is equal to the union of the set of eigenvalues of $\mathbf{M}_d(\mathbf{x}^*)$ and $n$ eigenvalues of value $-k_d\beta(d_r(\mathbf{x}^*))$.
Since $\mathbf{x}^*\in\mathcal{E}\cup\{\mathbf{x}_d\}\subset\mathcal{X}_r^{\circ}$, one has $d_r(\mathbf{x}^*) > 0$.
Therefore, since $k_d > 0$, it follows from the definition of $\beta$ in \eqref{beta_function_definition} that $-k_d\beta(d_r(\mathbf{x}^*)) < 0$ for all $\mathbf{x}^*\in\mathcal{E}\cup\{\mathbf{x}_d\}$.
Consequently, to characterize the stability properties of the equilibrium points $(\mathbf{x}^*, \mathbf{z}^*)\in\mathcal{S}\cup\{(\mathbf{x}_d, \mathbf{0})\}$, it is sufficient to analyze the eigenvalues of $\mathbf{M}_d(\mathbf{x}^*)$.

By Property \ref{property:stable_equilibria}, all eigenvalues of $\mathbf{M}_d(\mathbf{x}_d)$ have negative real parts.
Therefore, all eigenvalues of $\mathbf{J}_z(\mathbf{x}_d, \mathbf{0})$ have negative real parts, and hence $(\mathbf{x}_d, \mathbf{0})$ is asymptotically stable for the closed-loop system \eqref{transformed_closed_loop_system}.
Similarly, Property \ref{property:stable_equilibria} implies that $\mathbf{J}_z(\mathbf{x}^*, \mathbf{0})$ has at least one eigenvalue with a positive real part for every $\mathbf{x}^*\in\mathcal{E}$.
Therefore, by the center manifold theorem \cite[Section 2.7, Page 116]{perko2013differential}, the stable manifold associated with every $(\mathbf{x}^*, \mathbf{0})\in\mathcal{S}$ has zero Lebesgue measure in $\mathcal{X}_r^{\circ}\times\mathbb{R}^n$.

Since, by Property \ref{property:equilibrium_set}, $\mathcal{E}$ consists only of isolated points, every element $(\mathbf{x}^*, \mathbf{0})$ of $\mathcal{S}$ is an isolated equilibrium point.
Therefore, since the stable manifold of each equilibrium point in $\mathcal{S}$ has zero Lebesgue measure, the union of these stable manifolds also has zero Lebesgue measure in $\mathcal{X}_r^{\circ}\times\mathbb{R}^n$.

As established earlier, $\mathcal{S}\cup\{(\mathbf{x}_d, \mathbf{0})\}$ is globally attractive for the closed-loop system \eqref{transformed_closed_loop_system} over $\mathcal{X}_r^{\circ}\times\mathbb{R}^n$.
Furthermore, $(\mathbf{x}_d, \mathbf{0})$ is an asymptotically stable equilibrium point and the union of the stable manifolds of the equilibria in $\mathcal{S}$ has zero Lebesgue measure in $\mathcal{X}_r^{\circ}\times\mathbb{R}^n$.
Therefore, $(\mathbf{x}_d, \mathbf{0})$ is almost globally asymptotically stable for the closed-loop system \eqref{transformed_closed_loop_system} over $\mathcal{X}_r^{\circ}\times\mathbb{R}^n$.
This completes the proof.

\end{appendix}

\bibliographystyle{IEEEtran}
\bibliography{reference}

@article{mischaikow1995asymptotically,
  title={Asymptotically autonomous semiflows: chain recurrence and Lyapunov functions},
  author={Mischaikow, Konstantin and Smith, Hal and Thieme, Horst R},
  journal={Transactions of the American Mathematical Society},
  volume={347},
  number={5},
  pages={1669--1685},
  year={1995}
}

@article{welde2023compositional,
  title={A compositional approach to certifying almost global asymptotic stability of cascade systems},
  author={Welde, Jake and Kvalheim, Matthew D and Kumar, Vijay},
  journal={IEEE Control Systems Letters},
  volume={7},
  pages={1969--1974},
  year={2023},
  publisher={IEEE}
}

@inproceedings{lee2010geometric,
  title={Geometric tracking control of a quadrotor UAV on SE (3)},
  author={Lee, Taeyoung and Leok, Melvin and McClamroch, N Harris},
  booktitle={49th IEEE conference on decision and control (CDC)},
  pages={5420--5425},
  year={2010},
  organization={IEEE}
}

@book{delfour2011shapes,
  title={Shapes and geometries: metrics, analysis, differential calculus, and optimization},
  author={Delfour, Michel C. and Zol{\'e}sio, Jean-Paul},
  series={Advances in Design and Control},
  volume={22},
  edition={2nd},
  year={2011},
  publisher={SIAM},
  address={Philadelphia, PA, USA},
  isbn={9780898719369},
}

@book{rataj2019curvature,
  title={Curvature measures of singular sets},
  author={Rataj, Jan and Z{\"a}hle, Martina},
  series={Springer Monographs in Mathematics},
  year={2019},
  publisher={Springer},
  address={Cham, Switzerland},
  isbn={9783030181826},
}

@incollection{khatib1986real,
  title={Real-time obstacle avoidance for manipulators and mobile robots},
  author={Khatib, Oussama},
  booktitle={Autonomous robot vehicles},
  pages={396--404},
  year={1986},
  publisher={Springer}
}

@article{koditschek1990robot,
  title={Robot navigation functions on manifolds with boundary},
  author={Koditschek, Daniel E and Rimon, Elon},
  journal={Advances in applied mathematics},
  volume={11},
  number={4},
  pages={412--442},
  year={1990},
  publisher={Elsevier}
}

@article{verginis2021adaptive,
  title={Adaptive robot navigation with collision avoidance subject to 2nd-order uncertain dynamics},
  author={Verginis, Christos K and Dimarogonas, Dimos V},
  journal={Automatica},
  volume={123},
  pages={109303},
  year={2021},
  publisher={Elsevier}
}

@article{koditschek1992exact,
  title={Exact robot navigation using artificial potential functions},
  author={Koditschek, D and Rimon, E},
  journal={IEEE Trans. Robot. Automat},
  volume={8},
  pages={501--518},
  year={1992}
}

@article{li2018navigation,
  title={Navigation functions with time-varying destination manifolds in star worlds},
  author={Li, Caili and Tanner, Herbert G},
  journal={IEEE Transactions on Robotics},
  volume={35},
  number={1},
  pages={35--48},
  year={2018},
  publisher={IEEE}
}

@article{paternain2017navigation,
  title={Navigation functions for convex potentials in a space with convex obstacles},
  author={Paternain, Santiago and Koditschek, Daniel E and Ribeiro, Alejandro},
  journal={IEEE Transactions on Automatic Control},
  volume={63},
  number={9},
  pages={2944--2959},
  year={2017},
}

@article{arslan2019sensor,
  title={Sensor-based reactive navigation in unknown convex sphere worlds},
  author={Arslan, Omur and Koditschek, Daniel E},
  journal={The International Journal of Robotics Research},
  volume={38},
  number={2-3},
  pages={196--223},
  year={2019},
  publisher={SAGE Publications Sage UK: London, England}
}

@inproceedings{filippidis2012navigation,
  title={Navigation functions for everywhere partially sufficiently curved worlds},
  author={Filippidis, Ioannis F and Kyriakopoulos, Kostas J},
  booktitle={IEEE International Conference on Robotics and Automation},
  pages={2115--2120},
  year={2012},
}

@article{kumar2022navigation,
  title={Navigation of a quadratic potential with ellipsoidal obstacles},
  author={Kumar, Harshat and Paternain, Santiago and Ribeiro, Alejandro},
  journal={Automatica},
  volume={146},
  pages={110643},
  year={2022},
  publisher={Elsevier}
}

@article{wang2017safety,
  title={Safety barrier certificates for collisions-free multirobot systems},
  author={Wang, Li and Ames, Aaron D and Egerstedt, Magnus},
  journal={IEEE Transactions on Robotics},
  volume={33},
  number={3},
  pages={661--674},
  year={2017},
  publisher={IEEE}
}

@article{icsleyen2022low,
  title={From low to high order motion planners: Safe robot navigation using motion prediction and reference governor},
  author={{\.I}{\c{s}}leyen, Aykut and Van De Wouw, Nathan and Arslan, {\"O}m{\"u}r},
  journal={IEEE Robotics and Automation Letters},
  volume={7},
  number={4},
  pages={9715--9722},
  year={2022},
  publisher={IEEE}
}

@inproceedings{arslan2017smooth,
  title={Smooth extensions of feedback motion planners via reference governors},
  author={Arslan, Omur and Koditschek, Daniel E},
  booktitle={2017 IEEE International Conference on Robotics and Automation (ICRA)},
  pages={4414--4421},
  year={2017},
  organization={IEEE}
}

@inproceedings{tang2023constructive,
  title={Constructive barrier feedback for collision avoidance in leader-follower formation control},
  author={Tang, Zhiqi and Cunha, Rita and Hamel, Tarek and Silvestre, Carlos},
  booktitle={2023 62nd IEEE Conference on Decision and Control (CDC)},
  pages={368--374},
  year={2023},
  organization={IEEE}
}

@article{stavridis2017dynamical,
  title={Dynamical system based robotic motion generation with obstacle avoidance},
  author={Stavridis, Sotiris and Papageorgiou, Dimitrios and Doulgeri, Zoe},
  journal={IEEE Robotics and Automation Letters},
  volume={2},
  number={2},
  pages={712--718},
  year={2017},
  publisher={IEEE}
}

@inproceedings{smaili2024dissipative,
  title={Dissipative avoidance feedback for reactive navigation under second-order dynamics},
  author={Smaili, Lyes and Tang, Zhiqi and Berkane, Soulaimane and Hamel, Tarek},
  booktitle={2025 IEEE 64th Conference on Decision and Control (CDC)},
  pages={4452--4459},
  year={2025},
  organization={IEEE}
}

@phdthesis{micaelli1993trajectory,
  title={Trajectory tracking for unicycle-type and two-steering-wheels mobile robots},
  author={Micaelli, Alain and Samson, Claude},
  year={1993},
  school={Inria}
}

@book{khalil2002nonlinear,
  title={Nonlinear Systems},
  author={Khalil, H. K.},
  edition={3rd},
  year={2002},
  publisher={Prentice Hall},
  address={Upper Saddle River, N.J., USA},
  isbn={9780130673893},
}

@book{bernstein2009matrix,
  title={Matrix mathematics: theory, facts, and formulas},
  author={Bernstein, Dennis S.},
  edition={2nd},
  year={2009},
  publisher={Princeton University Press},
  address={Princeton, N.J., USA},
  isbn={9780691140391},
}

@book{lee2009manifolds,
  title={Manifolds and differential geometry},
  author={Lee, Jeffrey M.},
  series={Graduate Studies in Mathematics},
  volume={107},
  year={2009},
  publisher={American Mathematical Society},
  address={Providence, RI, USA},
  isbn={9780821848159},
}

@book{perko2013differential,
  title={Differential equations and dynamical systems},
  author={Perko, Lawrence},
  volume={7},
  year={2013},
  publisher={Springer Science \& Business Media},
  address={New York, NY, USA}
}

@book{rudin1976principles,
  title={Principles of Mathematical Analysis},
  author={Rudin, W.},
  isbn={9780070856134},
  lccn={75179033},
  series={International series in pure and applied mathematics},
  year={1976},
  publisher={McGraw-Hill},
  edition={3},
  address={New York, NY, USA}
}
\end{document}